\documentclass{article}

\PassOptionsToPackage{numbers, compress}{natbib}

\usepackage{template}

\usepackage{amsmath,amsfonts,bm}

\def\eqref#1{equation~\ref{#1}}

\def\1{\bm{1}}

\DeclareMathAlphabet{\mathsfit}{\encodingdefault}{\sfdefault}{m}{sl}
\SetMathAlphabet{\mathsfit}{bold}{\encodingdefault}{\sfdefault}{bx}{n}

\def\gJ{{\mathcal{J}}}

\usepackage[utf8]{inputenc} %
\usepackage[T1]{fontenc}    %
\usepackage[backref=section]{hyperref}
\usepackage{url}

\usepackage{booktabs}
\usepackage{multirow}
\usepackage{colortbl}
\usepackage{multicol}
\usepackage{amsfonts}
\usepackage{nicefrac}
\usepackage{microtype}
\usepackage[most]{tcolorbox}
\usepackage[dvipsnames]{xcolor}
\usepackage{latexsym}
\usepackage{graphicx}
\usepackage{float}
\usepackage{subcaption}
\usepackage{wrapfig}
\usepackage{dirtytalk}
\usepackage{enumitem}
\usepackage{makecell}

\usepackage[toc,page,header]{appendix}

\usepackage[nohints]{minitoc}

\usepackage{bm}

\usepackage{tabularx} 
\usepackage{arydshln}
\usepackage{ragged2e} 
\newcolumntype{L}{>{\RaggedRight\hangafter=1\hangindent=0em}X}

\usepackage{amsmath}
\usepackage{amssymb}
\usepackage{mathtools}
\usepackage{amsthm}

\usepackage[linesnumbered,ruled,vlined]{algorithm2e}

\hypersetup{
    colorlinks=true,
    linkcolor=Maroon,
    citecolor=ForestGreen,
    filecolor=magenta,      
    urlcolor=magenta,
}

\usepackage[capitalize,noabbrev]{cleveref}
\crefname{section}{§}{§§}
\Crefname{section}{§}{§§}
\crefname{lemma}{lemma}{lemma}
\Crefname{lemma}{Lemma}{Lemma}

\usepackage{calligra}
\DeclareMathAlphabet{\mathcalligra}{T1}{calligra}{m}{n}

\usepackage{pifont}

\theoremstyle{plain}
\newtheorem{theorem}{Theorem}[section]
\newtheorem{proposition}[theorem]{Proposition}

\theoremstyle{definition}

\theoremstyle{remark}

\renewcommand{\paragraph}[1]{\vspace{1mm}\noindent\textbf{#1}}

\DeclareCaptionLabelFormat{cont}{#1~#2\alph{ContinuedFloat}}

\tcbset{
  promptbox/.style={
    top=10pt,
    colback=lightgray!20,
    colframe=Black,
    colbacktitle=NavyBlue,
    enhanced,
    center,
    attach boxed title to top center={yshift=-0.1in,xshift=0.0in},
    boxed title style={boxrule=0pt,colframe=white,},
  }
}
\newtcolorbox{promptbox}[2][]{promptbox, title=#2,#1}
\tcbset{
  takeawaybox/.style={
    top=10pt,
    colback=lightgray!20,
    colframe=Black,
    colbacktitle=BurntOrange,
    enhanced,
    center,
    attach boxed title to top center={yshift=-0.1in,xshift=0.0in},
    boxed title style={boxrule=0pt,colframe=white,},
  }
}
\newtcolorbox{takeawaybox}[2][]{takeawaybox, title=#2,#1}
\tcbset{
  observationbox/.style={
    top=10pt,
    colback=lightgray!20,
    colframe=Black,
    colbacktitle=YellowGreen,
    enhanced,
    center,
    attach boxed title to top center={yshift=-0.1in,xshift=0.0in},
    boxed title style={boxrule=0pt,colframe=white,},
  }
}
\newtcolorbox{observationbox}[2][]{observationbox, title=#2,#1}

\newtcbtheorem[number within=section]{prompt}{Prompt}{
    top=10pt,
    colback=lightgray!20,
    colframe=Black,
    colbacktitle=NavyBlue,
    enhanced,
    center,
    fonttitle=\bfseries,
    breakable,
    separator sign={:},
    description delimiters={}{},
    boxrule=0.8pt,
    arc=3pt,
    outer arc=3pt,
    boxsep=0pt,
    left=12pt,right=12pt,top=10pt,bottom=10pt,
    lefttitle=12pt,righttitle=12pt,toptitle=5pt,bottomtitle=5pt,
    before skip=12pt,after skip=12pt,
    pad at break*=3mm,
    fontupper=\normalsize
}{pt}

\newtcbtheorem[number within=section]{example}{Case}{
    top=10pt,
    colback=lightgray!20,
    colframe=Black,
    colbacktitle=Tan,
    enhanced,
    center,
    fonttitle=\bfseries,
    breakable,
    separator sign={:},
    description delimiters={}{},
    boxrule=0.8pt,
    arc=3pt,
    outer arc=3pt,
    boxsep=0pt,
    left=12pt,right=12pt,top=10pt,bottom=10pt,
    lefttitle=12pt,righttitle=12pt,toptitle=5pt,bottomtitle=5pt,
    before skip=12pt,after skip=12pt,
    pad at break*=3mm,
    fontupper=\normalsize
}{em}
\newtcbtheorem[number within=section]{response}{Response}{
    top=10pt,
    colback=lightgray!20,
    colframe=Black,
    colbacktitle=YellowGreen,
    enhanced,
    center,
    attach boxed title to top center={yshift=-0.1in,xshift=0.0in},
    boxed title style={boxrule=0pt,colframe=white,},
    breakable,
    separator sign={:},
    description delimiters={}{},
    boxrule=0.8pt,
    arc=3pt,
    outer arc=3pt,
    boxsep=0pt,
    left=12pt,right=12pt,top=10pt,bottom=10pt,
    lefttitle=12pt,righttitle=12pt,toptitle=5pt,bottomtitle=5pt,
    before skip=12pt,after skip=12pt,
    pad at break*=3mm,
    fontupper=\normalsize
}{rp}

\usepackage{xspace}
\newcommand{\name}[0]{\textsc{{AIPO}}\xspace}
\newcommand{\drop}[1]{\textcolor{red!80!black}{$\downarrow${\scriptsize #1}}}

\title{\name: Learning to Reason from Active Interaction}

\author{Junnan Liu, Linhao Luo, Thuy-Trang Vu \& Gholamreza Haffari  \\
Department of Data Science and AI, Faculty of Information Technology, \\
Monash University, Australia\\
}

\date{\today}
\correspondence{Junnan Liu at
  \href{mailto:junnan.liu@monash.edu}{\texttt{junnan.liu@monash.edu}}}

\begin{document}
\doparttoc %
\faketableofcontents %

\maketitle

\begin{abstract}
Recent advances in large language models (LLMs) have demonstrated strong reasoning capabilities, largely stimulated by reinforcement learning with verifiable rewards (RLVR).
However, the exploration of existing RLVR algorithms remains largely constrained by the knowledge and reasoning strategies already accessible to the policy model.
Although recent methods introduce external expert demonstrations to broaden exploration, they typically rely on complete trajectory-level guidance, which can be sample-inefficient, information-sparse, and insufficiently adaptive to intermediate reasoning bottlenecks.
Inspired by collaborative multi-agent systems, we propose \textbf{\name}, an enhanced reinforcement learning framework that improves LLM reasoning through active multi-agent interaction during exploration.
Specifically, when encountering reasoning bottlenecks, \name enables the policy model to proactively consult three functional collaborative agents, namely the \textit{verify agent}, \textit{knowledge agent}, and \textit{reasoning agent}, thereby obtaining fine-grained and state-dependent guidance during rollout.
The resulting mixed-policy trajectories expose the policy to reasoning directions that may be difficult to discover through isolated on-policy exploration.
To learn effectively from collaborator-provided tokens, we further introduce a corrected importance sampling coefficient together with a lower-bound clipping strategy to mitigate off-policy discrepancy and vanishing gradients.
After training, the policy model reasons independently without relying on collaborative agents.
Extensive experiments across mathematical, scientific, coding, and puzzle reasoning benchmarks show that \name consistently improves reasoning performance and generalizes across different policy models, collaborator backbones, and RLVR algorithms.
Additional out-of-distribution and training-dynamics analyses indicate that active interaction broadens the policy model's \emph{effective exploration boundary}, enabling it to solve additional problems independently under the same inference-time sampling budget.
\end{abstract}

\section{Introduction} \label{sec:intro}

Large language models (LLMs) have demonstrated strong reasoning abilities, enabling them to solve complex mathematical, coding, and scientific tasks \citep{openai2024o1, openai2024o3, openai2025gpt5, abs-2501-12948, abs-2501-12599, abs-2505-09388, abs-2507-06261, abs-2507-15855, xai2025grok4}. 
A key factor behind this progress is Chain-of-Thought (CoT) reasoning, where models explore and reflect over intermediate reasoning steps to construct robust reasoning processes before producing final answers \citep{abs-2503-09567}. 
These capabilities are largely stimulated by reinforcement learning with verifiable rewards (RLVR) \citep{SchulmanWDRK17, abs-2402-03300}, which encourages LLMs, as policy models, to explore diverse reasoning trajectories and learn from the successful trajectories under verifiable reward signals, thereby acquiring stronger reasoning policies that support self-improvement and scalable inference-time reasoning \citep{abs-2408-03314}.

Despite these advances, existing RLVR frameworks remain constrained by the knowledge and reasoning strategies that the policy model can access during its own rollouts, as illustrated in \Cref{fig:intro} (A)~\citep{abs-2504-13837,LiaoXCLHZLW25,abs-2509-25666,HeFW25,abs-2504-14945,abs-2506-19767}.
Although RLVR encourages exploration over diverse reasoning trajectories, the reachable trajectory space is still largely determined by the policy's pretrained knowledge and current reasoning behavior.
Consequently, standard on-policy RLVR can improve performance by exploring more effectively within this accessible region, but may still fail to expose the policy to useful reasoning directions that it is unlikely to discover independently~\citep{abs-2504-13837,abs-2504-10478,abs-2504-07912}.

To overcome this limitation, recent studies seek to enhance model performance and expand capability boundaries by leveraging guidance from stronger external expert models such as expert trajectories~\citep{abs-2501-12948} or critiques~\citep{abs-2509-26306}, primarily through supervised fine-tuning~\citep{abs-2509-06948,abs-2601-18734} or offline reinforcement learning~\citep{abs-2508-11408,abs-2504-14945} with expert demonstrations, as illustrated in \Cref{fig:intro} (B) \citep{abs-2509-04419}. 
However, these methods typically depend on complete expert trajectories, which are costly to sample, information-sparse, and often redundant for training. 
Moreover, full-trajectory learning provides only coarse-grained supervision, offering limited fine-grained guidance for identifying and resolving intermediate reasoning bottlenecks~\citep{abs-2510-25992}. 
It also confines exploration to a static expert-generated distribution, leaving the trained model susceptible to the capability limits and biases of the static external expert. 

\begin{figure}[t]
    \centering
    \includegraphics[width=1.\textwidth]{figs/introv1.pdf}
    \caption{Comparison between existing methods and the proposed \name. (A) Standard on-policy RLVR explores trajectories generated by the policy model itself and is therefore constrained by its current exploration distribution. (B) Existing expert-guided approaches introduce supervision from externally generated expert trajectories, for example through supervised, distillation, or off-policy learning objectives. (C) \name instead introduces targeted collaborator feedback \emph{during} rollout, producing mixed-policy trajectories conditioned on the policy's current reasoning bottlenecks. } \label{fig:intro} 
    \vspace{-2.0em}
\end{figure}

Motivated by self-evolving and collaborative multi-agent systems~\citep{abs-2510-08529,abs-2602-23008,abs-2510-23595,abs-2506-19767}, which extend the capabilities of monolithic models through collaboration and communication among multiple LLMs, we propose \textbf{\name} (\textbf{A}ctive and \textbf{I}nteractive \textbf{P}olicy \textbf{O}ptimization).
As illustrated in \Cref{fig:intro} (C), \name enhances RLVR exploration through active interaction with external collaborative agents.
Compared with methods based on complete expert demonstrations or trajectory-level critiques, \name provides fine-grained, state-dependent guidance at intermediate reasoning bottlenecks and does not necessarily require a stronger expert model.
To provide complementary forms of assistance, \name introduces three functional collaborators: \ding{182} \textit{Verify agent}, which verifies intermediate conclusions; \ding{183} \textit{Knowledge agent}, which provides relevant domain knowledge; and \ding{184} \textit{Reasoning agent}, which assists with intermediate sub-problems.
During exploration, the policy autonomously decides whether and which collaborator to invoke, integrates the returned feedback into its reasoning, and thereby constructs \emph{mixed-policy} trajectories that may contain reasoning directions unlikely to be discovered through isolated on-policy rollout.
Then, the policy optimizes successful mixed-policy trajectories and internalizes their useful reasoning signals.
Because these trajectories contain externally generated tokens, we introduce a corrected importance sampling coefficient and a lower-bound clipping strategy to mitigate the resulting off-policy discrepancy and vanishing-gradient problem.
After training, all collaborators are removed and the policy model reasons independently using the knowledge and reasoning strategies internalized through interaction.

We conduct extensive experiments on diverse reasoning benchmarks, including AIME24, AIME25, MATH500~\citep{HendrycksBKABTS21}, LiveMathBench~\citep{abs-2412-13147}, GPQA-Diamond~\citep{abs-2311-12022}, MBPP~\citep{abs-2108-07732}, LiveCodeBench~\citep{JainHGLYZWSSS25}, and Reasoning-Gym~\citep{abs-2505-24760}.
The results show that \name consistently improves over competitive RLVR, distillation, and expert-guided baselines across both in-domain and out-of-domain benchmarks.
The gains generalize across different policy-model and collaborator backbones from the Qwen~\citep{abs-2412-15115} and Llama~\citep{abs-2407-21783} families, and \name can be integrated with different RLVR algorithms, including GRPO~\citep{abs-2402-03300}, DAPO~\citep{abs-2503-14476}, and GSPO~\citep{abs-2507-18071}.
Further analyses on DeepMath and the training dynamics show that active interaction broadens the policy model's \emph{effective exploration boundary}: after training, the policy solves additional problems independently without requiring collaborators or additional inference-time computation.

\section{Preliminaries} \label{sec:preliminaries} 

\vspace{-0.5em}
\paragraph{RLVR for LLM Reasoning.}
Reinforcement Learning with Verifiable Rewards (RLVR) is widely adopted to improve LLM reasoning~\citep{abs-2501-12948, abs-2501-12599, abs-2505-09388, openai2024o1, openai2024o3, abs-2507-06261, abs-2503-24290}.
To avoid the separate value network required by PPO~\citep{SchulmanWDRK17, Ouyang0JAWMZASR22}, critic-free methods such as ReMax~\citep{LiXZL00L24}, RLOO~\citep{AhmadianCGFKPUH24}, GRPO~\citep{abs-2501-12948, abs-2402-03300}, and REINFORCE++~\citep{abs-2501-03262} replace the value estimate with reward-based baselines and optimize
\begin{equation} \label{eq:standard_clipped_surrogate}
    \setlength\abovedisplayskip{0.0pt}%
    \setlength\belowdisplayskip{0.0pt}
    \gJ(\theta) \!=\! \mathbb E_{q \sim P,\{\bm \tau_i\} \sim \pi_{\theta_{\text{old}}}}\!\Bigg[\!\frac1G\!\sum_{i=1}^G\!\frac1{|\bm \tau_i|}\!\sum_{t=1}^{|\bm \tau_i|}\!\Big\{ \min\!\big(\rho_{i,t}\tilde A_t^i,\,\text{clip}(\rho_{i,t},1\!-\!\epsilon,1\!+\!\epsilon)\,\tilde A_t^i\big)\!-\!\beta\,D_{\text{KL}}\!\big[\pi_\theta\|\pi_{\text{ref}}\big]\Big\} \Bigg],
\end{equation}
where $\{\bm{\tau}_1,\dots,\bm{\tau}_G\} \sim \pi_{\theta_{\text{old}}}(\cdot|q)$ is a group of $G$ trajectories rolled out for a query $q \sim P$, $\tilde A_t^i$ is the normalized advantage computed from a reward baseline, $\rho_{i,t} = \frac{\pi_\theta ( \bm{\tau}_{i,t} \mid \bm{\tau}_{i,<t} )}{\pi_{\theta_{\text{old}}} ( \bm{\tau}_{i,t} \mid \bm{\tau}_{i,<t} )}$ is the importance sampling coefficient correcting the shift between the current and rollout policies, clipping bounds the update, and $D_{\text{KL}}$ regularizes $\pi_\theta$ toward a reference policy $\pi_{\text{ref}}$.

\vspace{-0.5em}
\paragraph{Off-Policy Enhanced RLVR.}
To broaden the reasoning trajectories available to the policy beyond those produced by its own rollout distribution, recent methods~\citep{abs-2504-14945,abs-2509-04419,abs-2508-11408,abs-2509-06948} introduce demonstrations from a stronger teacher $\pi_{\text{ext}}$, combining on-policy trajectories $\{\bm{\tau}_i\}_{i=1}^N \sim \pi_\theta(\cdot|q)$ with externally generated trajectories $\{\bm{\tau}_i^\prime\}_{i=1}^M \sim \pi_{\text{ext}}(\cdot|q)$:
\begin{equation}
    \setlength\abovedisplayskip{0.0pt}%
    \setlength\belowdisplayskip{0.0pt}
    \gJ_{c}(\theta) = \gJ \left(\theta, \{\bm{\tau}_i\}_{i=1}^N \right) + \gJ_{\text{ext}} \left(\theta, \phi, \{\bm{\tau}_i^\prime\}_{i=1}^M \right),
\end{equation}
where $\phi$ denotes the teacher parameters and the auxiliary term $\gJ_{\text{ext}}$ is instantiated either as an SFT objective~\citep{abs-2508-11408,abs-2506-07527,abs-2506-19767} or as an RLVR objective corrected for the \textit{off-policy} discrepancy~\citep{abs-2504-14945}.

\section{Methodology} \label{sec:method} 

\vspace{-0.5em}
In this section, we introduce \name, which consists of two complementary components:
\ding{182} \textit{exploration}, an active and interactive rollout process in which the policy selectively queries external collaborators to broaden the set of reachable reasoning trajectories as shown in \Cref{fig:method}~(a) and \Cref{sec:interaction}; and
\ding{183} \textit{learning}, in which the policy optimizes successful mixed-policy trajectories and internalizes the useful reasoning signals they contain (\Cref{sec:learn_interaction}).
Because mixed-policy trajectories additionally contain tokens generated by external collaborators, learning from them requires addressing the resulting off-policy discrepancy and vanishing-gradient problem (\Cref{fig:method}~(b)).

\subsection{Enhanced Rollout Based on Active Interaction} \label{sec:interaction}

To broaden the policy model's effective exploration space, we design an enhanced rollout process based on active multi-agent interaction.
Instead of requiring the policy to explore solely through trajectories generated from its own current distribution, the proposed rollout allows it to request targeted assistance when encountering intermediate reasoning bottlenecks.
The resulting feedback introduces potentially useful reasoning directions into the rollout while preserving the policy's responsibility for integrating them into the final solution.

\paragraph{Overview of Active Interaction. }
Given a question $q$, the policy model $\pi_\theta$ first performs basic reasoning using its internal knowledge and capabilities, such as problem decomposition, solution planning, and elementary arithmetic operations~\citep{Wei0SBIXCLZ22}. 
If the policy model can solve the question independently, external guidance is unnecessary, which is often overlooked by existing methods~\citep{abs-2410-18982, abs-2411-16489, abs-2501-19393, abs-2506-04178}. 
Conversely, when the policy model identifies uncertainty or an unresolved sub-problem during rollout, it may actively invoke an external collaborator to obtain targeted assistance. 
Under outcome-guided RLVR optimization, the policy can learn when and how to invoke collaborators, as interactions appearing in successful trajectories receive positive learning signals.
Successful interactions can introduce useful reasoning signals into rewarded trajectories, which are subsequently reinforced through policy optimization.

\begin{figure}[t]
    \centering
    \includegraphics[width=1.\textwidth]{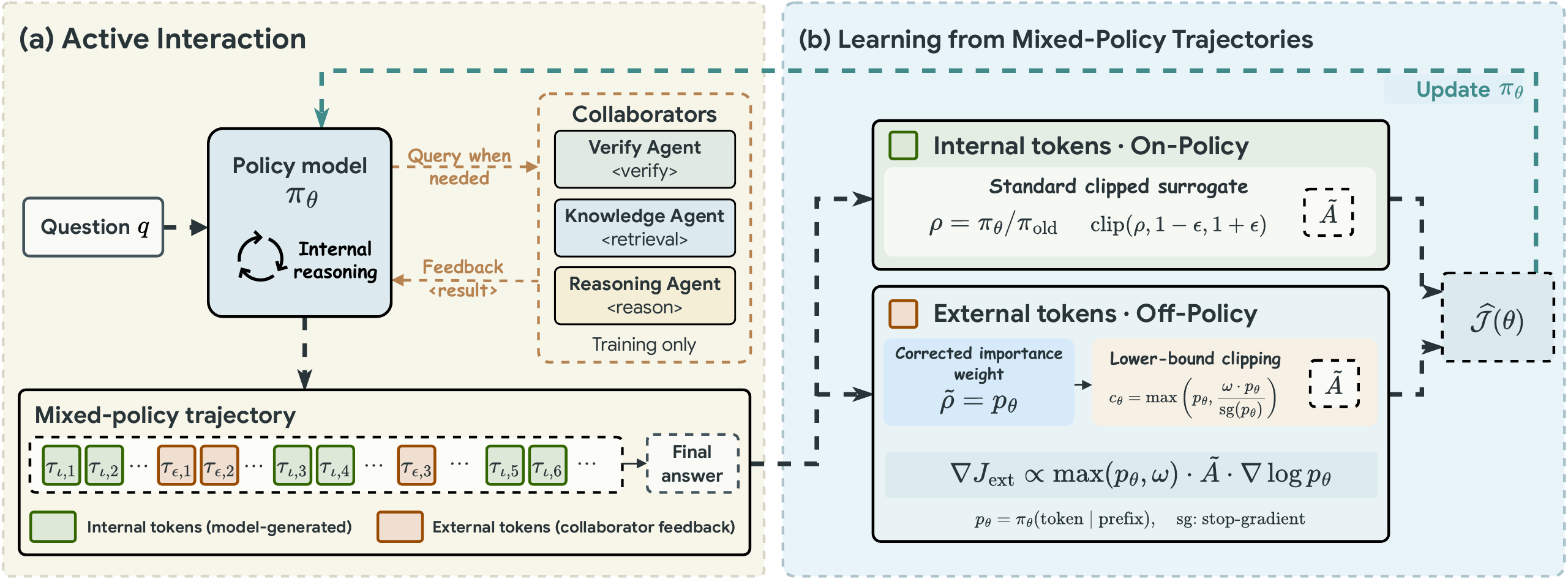}
    \caption{Illustration of \name. In the \name framework, during each rollout, the policy model engages in active interactions with collaborators. We then compute the reward and optimize the policy model using losses derived from both internal (on-policy) and external (off-policy) tokens. Additionally, we propose a corrected importance sampling coefficient and clipping strategy to mitigate off-policy issues and the vanishing gradient problem for external tokens.}
    \label{fig:method}
    \vspace{-1.5em}
\end{figure}

\paragraph{Workflow of Active Interaction. } 
Specifically, we provide the policy model with the definitions and descriptions of three functional agents, i.e., collaborators, in the system prompt (see Prompt~\ref{pt:full_prompt}). 
These agents provide complementary forms of assistance for complex reasoning, and the policy interacts with them through a structured protocol.
Detailed prompts are provided in \Cref{app:collaborator_prompt}.
\begin{itemize}[leftmargin=*]
    \vspace{-0.5em}
    \item \textit{Verify Agent:} Verifies the correctness of intermediate conclusions. The policy model invokes this agent by enclosing the conclusion to be checked within \texttt{<verify>} and \texttt{</verify>} tags, and the agent returns the verification result within \texttt{<result>} and \texttt{</result>} tags.
    \vspace{-0.2em}
    \item \textit{Knowledge Agent:} Provides knowledge required for reasoning. The policy model invokes this agent by enclosing its query within \texttt{<retrieval>} and \texttt{</retrieval>} tags, and the agent returns the retrieved information within \texttt{<result>} and \texttt{</result>} tags. Following ZeroSearch~\citep{abs-2505-04588}, the Knowledge Agent synthesizes six short documents, including one noisy document, using the reference answer from the training data as hidden supervision. The policy observes only the generated documents, and the mechanism is used only during training. Further details are provided in \Cref{app:knowledge_prompt_design}.
    \vspace{-0.2em}
    \item \textit{Reasoning Agent:} Solves intermediate sub-tasks encountered during reasoning. The policy model invokes this agent by enclosing the sub-task within \texttt{<reason>} and \texttt{</reason>} tags, and the agent returns the corresponding result within \texttt{<result>} and \texttt{</result>} tags.
    \vspace{-0.5em}
\end{itemize}

The responses from these agents are incorporated into the rollout trajectory as external tokens, producing mixed-policy trajectories that combine the policy model's own reasoning with collaborator-provided responses:
\begin{equation}
    \setlength\abovedisplayskip{0.5pt}%
    \setlength\belowdisplayskip{0.5pt}
    \bm{\tau} =
    \left\{
    \textcolor[RGB]{0,218,0}{\bm{\tau}_{\iota,1}}, \ldots,
    \textcolor[RGB]{255,51,51}{\bm{\tau}_{\epsilon,1}}, \ldots,
    \textcolor[RGB]{0,218,0}{\bm{\tau}_{\iota,t}}, \ldots,
    \textcolor[RGB]{255,51,51}{\bm{\tau}_{\epsilon,|\bm{\tau}_\epsilon|}},
    \ldots,
    \textcolor[RGB]{0,218,0}{\bm{\tau}_{\iota,|\bm{\tau}_\iota|}}
    \right\}
    \sim
    \prod_{i=1}^{|\bm{\tau}_\iota|}
    \pi_{\theta} \left( \bm{\tau}_{\iota, i} \mid \bm{\tau}_{<i} \right)
    \cdot
    \prod_{j=1}^{|\bm{\tau}_\epsilon|}
    \pi_{\epsilon} \left( \bm{\tau}_{\epsilon,j} \mid \bm{I}_j \right),
\end{equation}
where $\textcolor[RGB]{0,218,0}{\bm{\tau}_\iota}$ denotes the sequence of tokens generated by the policy model, $\textcolor[RGB]{255,51,51}{\bm{\tau}_\epsilon}$ denotes the sequence of tokens generated by external collaborators, the subscript indexes each token, and $\bm{I}_j$ denotes the interaction context used to generate the $j$-th external token. 
For simplicity, we merge the three collaborator policies into a unified external distribution $\pi_{\epsilon}$, while $\pi_{\theta}$ denotes the policy model. 
We provide additional discussions in \Cref{app:error_mitigation}.

Importantly, the policy model does not rely on external agents during inference. 
Instead, it reasons independently by leveraging the knowledge and strategies internalized during training. 
This distinguishes our method from conventional multi-agent systems~\citep{abs-2503-09501,abs-2504-15257,abs-2504-16129}: interaction is used only to enhance training, whereas the trained model performs reasoning autonomously at test time.

\subsection{Learning with Active Interaction} \label{sec:learn_interaction}

During interactive rollout, the policy model generates a trajectory $\bm{\tau}$ that consists of both internal tokens $\bm{\tau}_\iota$ generated by the policy model and external tokens $\bm{\tau}_\epsilon$ provided by collaborators. 
Unlike previous methods~\citep{abs-2503-09516,abs-2503-05592,abs-2505-19300}, which typically discard external tokens and optimize only over the internal tokens generated by the policy model, \name incorporates both internal and external tokens into the learning process.

\paragraph{Off-Policy Issues. }
The primary challenge lies in learning from tokens generated by external collaborators. 
Standard RL objectives generally assume that all tokens are sampled from the current or previous policy distribution. 
The classical clipped surrogate objective is defined in \Cref{eq:standard_clipped_surrogate}, where the importance sampling coefficient for each token $\bm{\tau}_{i,t}$ is 
$
    \rho_{i,t}
    =
    \frac{
    \pi_\theta \left( \bm{\tau}_{i,t} \mid \bm{\tau}_{i,<t} \right)
    }{
    \pi_{\theta_{\text{old}}} \left( \bm{\tau}_{i,t} \mid \bm{\tau}_{i,<t} \right)
    },
$ 
and $\pi_{\theta_{\text{old}}}$ denotes the policy from the previous iteration. 
Directly applying \Cref{eq:standard_clipped_surrogate} to external tokens $\bm{\tau}_\epsilon$ may introduce off-policy issues due to distributional mismatch between the policy model and external collaborators, potentially destabilizing training~\citep{SchulmanWDRK17,abs-2504-14945}. 
Conventional approaches often avoid this issue by excluding external tokens from the policy loss~\citep{abs-2503-05592,abs-2505-19300}. 
In contrast, \name explicitly incorporates external tokens into policy optimization, enabling the policy model to acquire useful knowledge and reasoning patterns from collaborators.

\paragraph{Correcting the Importance Sampling Coefficient for External Tokens.}
To mitigate off-policy errors, we introduce a corrected importance sampling coefficient $\tilde{\rho}$ for external tokens. 
Since external tokens are sampled from the collaborator distribution $\pi_{\epsilon}$ rather than $\pi_{\theta_{\text{old}}}$, the objective can be decomposed as
\begin{equation}
    \setlength\abovedisplayskip{0.5pt}%
    \setlength\belowdisplayskip{0.5pt}
    \mathcal{J}(\theta)
    =
    \mathbb{E}_{\pi_{\theta_{\text{old}}}}
    \left[
    \frac{1}{|\bm{\tau}_\iota|}
    \sum_{\bm{\tau}_{t} \in \bm{\tau}_\iota}
    \frac{\pi_\theta(\bm{\tau}_{t} \mid \bm{\tau}_{<t})}
    {\pi_{\theta_{\text{old}}}(\bm{\tau}_{t} \mid \bm{\tau}_{<t})}
    \tilde{A}_t
    \right]
    +
    \mathbb{E}_{\pi_{\epsilon}}
    \left[
    \frac{1}{|\bm{\tau}_\epsilon|}
    \sum_{\bm{\tau}_{t} \in \bm{\tau}_\epsilon}
    \frac{\pi_\theta(\bm{\tau}_{t} \mid \bm{\tau}_{<t})}
    {\pi_{\epsilon}(\bm{\tau}_{t} \mid \bm{I}_t)}
    \tilde{A}_t
    \right],
\end{equation} 
where clipping and KL-penalty terms are omitted for brevity. 
Although directly using the collaborator distribution in the denominator is theoretically natural, it poses two practical challenges: 
\ding{182} vocabulary discrepancies between the collaborator and the policy model may lead to incompatible token-level probability estimates; and 
\ding{183} evaluating and storing collaborator token distributions introduces additional computational and memory overhead.
To address these limitations, we approximate the collaborator distribution as a one-hot distribution over the sampled external token and place it under the policy distribution, yielding the corrected objective:
\begin{equation} \label{eq:modified_importance_sampling}
    \setlength\abovedisplayskip{0.5pt}%
    \setlength\belowdisplayskip{0.5pt}
    \mathcal{J}^\prime(\theta)
    =
    \mathbb{E}_{\pi_{\theta_{\text{old}}}}
    \left[
    \frac{1}{|\bm{\tau}_\iota|}
    \sum_{\bm{\tau}_{t} \in \bm{\tau}_\iota}
    \frac{\pi_\theta(\bm{\tau}_{t} \mid \bm{\tau}_{<t})}
    {\pi_{\theta_{\text{old}}}(\bm{\tau}_{t} \mid \bm{\tau}_{<t})}
    \tilde{A}_t
    \right]
    +
    \mathbb{E}_{\pi_{\theta_{\text{old}}}}
    \left[
    \frac{1}{|\bm{\tau}_\epsilon|}
    \sum_{\bm{\tau}_{t} \in \bm{\tau}_\epsilon}
    \pi_\theta(\bm{\tau}_{t} \mid \bm{\tau}_{<t})
    \tilde{A}_t
    \right].
\end{equation}
This approximation can be interpreted as assigning all probability mass to the observed collaborator token. 
It obviates explicit access to collaborator logits and their associated computation, resolves token-level vocabulary incompatibility, and still encourages the policy model to raise the likelihood of high-advantage external tokens.
We provide further approximation error analysis in \Cref{app:error_analysis}. 
Importantly, eliminating the explicit collaborator distribution from the coefficient relaxes a \emph{distributional compatibility} requirement rather than a \emph{quality} requirement.
The collaborator need not share the policy's vocabulary, backbone, or token-level distribution, but its feedback must still contain useful reasoning signals for successful trajectories.

\paragraph{Gradient Vanishing for External Tokens.}
The standard surrogate objective uses clipping to prevent excessive policy deviation from the previous policy. 
For external tokens, however, the corrected coefficient $\pi_\theta(\bm{\tau}_t \mid \bm{\tau}_{<t})$ is naturally bounded by the softmax output. 
Thus, applying the original clipping mechanism is unnecessary. 
Without additional treatment, the gradient contribution of an external token is proportional to 
$
    \pi_\theta(\bm{\tau}_t \mid \bm{\tau}_{<t})
    \cdot
    \tilde{A}_t
    \cdot
    \nabla_\theta \log \pi_\theta(\bm{\tau}_t \mid \bm{\tau}_{<t}).
$ 
When an external token has low probability under the policy model, i.e., $\pi_\theta(\bm{\tau}_t \mid \bm{\tau}_{<t}) \to 0$, its gradient contribution also approaches zero. 
This issue is particularly relevant to external tokens that initially have low probability under the policy, including potentially useful reasoning patterns that are uncommon under its current distribution.

\paragraph{Clipping Strategy for External Tokens.}
To alleviate the vanishing gradient problem, we introduce a lower-bound clipping strategy for external tokens:
\begin{equation} \label{eq:clipping_strategy}
    \setlength\abovedisplayskip{0.5pt}%
    \setlength\belowdisplayskip{0.5pt}
    \text{clip}
    \left(
    \pi_\theta,
    \frac{\omega}{\text{sg} \left( \pi_\theta \right)}
    \cdot
    \pi_\theta,
    \infty
    \right),
\end{equation}
where $\omega$ is a clipping hyperparameter and $\text{sg}(\cdot)$ denotes the stop-gradient operation. 
This formulation sets a lower bound on the coefficient for external tokens, while the term $\pi_\theta / \text{sg}(\pi_\theta)$ preserves numerical equivalence in the forward pass. 
The resulting gradients are:
\begin{equation}
    \setlength\abovedisplayskip{0.0pt}%
    \setlength\belowdisplayskip{0.0pt}
    \left\{
    \begin{aligned}
        & \pi_\theta \cdot \tilde A_t \cdot \nabla_\theta \log \pi_\theta, 
        && \text{if } \pi_\theta \ge \omega, \\
        & \omega \cdot \tilde A_t \cdot \nabla_\theta \log \pi_\theta, 
        && \text{if } 0 \le \pi_\theta < \omega.
    \end{aligned}
    \right.
\end{equation}
This strategy ensures that low-probability external tokens with positive learning signals still receive non-vanishing gradients, allowing the policy to learn from external feedback even when such tokens initially have low probability.

\paragraph{Final Objective.}
By integrating the corrected importance sampling coefficient and the proposed clipping strategy, we obtain the final objective of \name:
\begin{equation}
    \setlength\abovedisplayskip{0.0pt}%
    \setlength\belowdisplayskip{0.0pt}
    \begin{aligned}
        \widehat{\gJ}(\theta)
        = &\;
        \mathbb E_{\pi_{\theta_{\text{old}}}}
        \left[
        \frac{1}{|\bm{\tau}_\iota|}
        \sum_{\bm{\tau}_{t} \in \bm{\tau}_\iota}
        \min
        \left\{
        \frac{\pi_\theta^t}{\pi_{\theta_{\text{old}}}^t}
        \tilde A_t,
        \text{clip}
        \left(
        \frac{\pi_\theta^t}{\pi_{\theta_{\text{old}}}^t},
        1 - \epsilon,
        1 + \epsilon
        \right)
        \tilde A_t
        \right\}
        \right] \\
        &+
        \mathbb E_{\pi_{\theta_{\text{old}}}}
        \left[
        \frac{1}{|\bm{\tau}_\epsilon|}
        \sum_{\bm{\tau}_{t} \in \bm{\tau}_\epsilon}
        \text{clip}
        \left(
        \pi_\theta^t,
        \frac{\omega}{\text{sg} \left( \pi_\theta^t \right)}
        \cdot
        \pi_\theta^t,
        \infty
        \right)
        \tilde A_t
        \right],
    \end{aligned}
\end{equation}
where $\pi_\theta^t$ denotes $\pi_\theta(\bm{\tau}_{t} \mid \bm{\tau}_{<t})$ and $\pi_{\theta_{\text{old}}}^t$ denotes $\pi_{\theta_{\text{old}}}(\bm{\tau}_{t} \mid \bm{\tau}_{<t})$. The first term optimizes internal on-policy tokens using the standard clipped surrogate objective, while the second term enables learning from external off-policy tokens.

\section{Experiments}

\subsection{Setup} \label{sec:exp-settings}

\vspace{-0.5em}
\paragraph{Baselines.}
We compare \name against several representative LLM post-training methods: 
\ding{182} Supervised Fine-Tuning (SFT): This method fine-tunes the model using synthetic data generated via rejection sampling. 
\ding{183} On-Policy Distillation: These methods distill knowledge from a stronger teacher model while preserving on-policy trajectory generation. 
Specifically, trajectories are sampled from the policy model, whereas the supervision signal is derived from the logits of the teacher model. 
We adopt OPSD~\citep{abs-2601-18734} as a representative work.
\ding{184} Reinforcement Learning (RL): These methods train the model using algorithms such as GRPO \citep{abs-2402-03300}, PRIME \citep{abs-2502-01456}, and Dr.GRPO~\citep{abs-2503-20783}. 
\ding{185} Off-Policy Enhanced Reinforcement Learning: Represented by LUFFY \citep{abs-2505-15612}, these methods enhance RL exploration by utilizing trajectories generated by strong models as guidance.

\vspace{-0.5em}
\paragraph{Evaluation Benchmarks.}
We evaluate all models across four domain-specific benchmarks: \ding{182} \textit{Mathematical Reasoning}: Includes AIME24, AIME25, MATH500~\citep{HendrycksBKABTS21}, and LiveMathBench~\citep{abs-2412-13147}; \ding{183} \textit{Scientific Reasoning}: Represented by GPQA-Diamond~\citep{abs-2311-12022}; \ding{184} \textit{Code Reasoning}: Comprises MBPP~\citep{abs-2108-07732} and LiveCodeBench~\citep{JainHGLYZWSSS25}; \ding{185} \textit{Puzzle Reasoning}: Includes puzzles from Reasoning-Gym~\citep{abs-2505-24760}. 

\vspace{-0.5em}
\paragraph{Implementation Details.} 
We use Qwen2.5-7B-Instruct~\citep{abs-2412-15115} and Llama3.2-3B-Instruct~\citep{abs-2407-21783} as policy models, and instantiate collaborators from Qwen2.5-7B-Instruct, Llama3.2-3B-Instruct, or the substantially stronger Qwen3-30B-A3B-Instruct-2507~\citep{abs-2505-09388}.
Training uses $34{,}000$ reasoning-intensive examples drawn from DAPO~\citep{abs-2503-14476} and OpenScienceReasoning-2\footnote{\url{https://huggingface.co/datasets/nvidia/OpenScienceReasoning-2}} and runs for 200 optimization steps, with at most three collaborator interactions allowed per generation.
For all RL-based methods, we evaluate a held-out validation set every 10 training steps and select the checkpoint with the best validation performance.
All methods are evaluated under the same prompt templates and sampling configuration.
Full training hyperparameters, baseline-specific configurations, evaluation protocol, checkpoint selection, and efficiency statistics are provided in \Cref{app:training_details,app:baseline_details,app:eval_protocol,app:checkpoint,app:efficiency}.

\subsection{Main Results and Analysis}
\Cref{tab:main_performance} illustrates the performance of \name and baselines on different benchmarks, with different collaborator LLMs. 
We summarize the main findings as follows.

\begin{table}[t]
    \centering
    \caption{
    Experimental results of \name and baselines with Qwen and Llama, where LUFFY and \name are based on GRPO.
    We report Avg@$16$ for AIME24 and AIME25 and Avg@$4$ for the remaining benchmarks, where Avg@$k$ denotes the mean correctness over $k$ independently sampled responses per evaluation problem.
    The subscripts denote the standard deviation over independent evaluations.
    We abbreviate LMB as LiveMathBench v202505, LCB as LiveCodeBench v6, and RG as Reasoning Gym.
    \(\spadesuit\) denotes in-domain evaluation and \(\clubsuit\) denotes out-of-domain evaluation.
    The performance of Llama is provided in \Cref{app:llama3_rl_performance}.
    }
    \label{tab:main_performance}
    \resizebox{0.95\textwidth}{!}{
        \begin{tabular}{lcccccccc}
            \toprule
            \multirow{3}{*}{\bf Methods}
            & \multicolumn{4}{c}{\bf Math \(\spadesuit\)}
            & \multicolumn{1}{c}{\bf Science \(\spadesuit\)}
            & \multicolumn{2}{c}{\bf Code \(\clubsuit\)}
            & \multicolumn{1}{c}{\bf Puzzle \(\clubsuit\)}
            \\
            \cmidrule(r){2-5}
            \cmidrule(r){6-6}
            \cmidrule(r){7-8}
            \cmidrule(r){9-9}
            & \bf AIME24 & \bf AIME25 & \bf MATH500 & \bf LMB
            & \bf GPQA-D & \bf MBPP & \bf LCB & \bf RG \\
            & \bf Avg@\(16\) & \bf Avg@\(16\) & \bf Avg@\(4\) & \bf Avg@\(4\)
            & \bf Avg@\(4\) & \bf Avg@\(4\) & \bf Avg@\(4\) & \bf Avg@\(4\) \\
            \midrule

            \rowcolor{lightgray!30}
            \multicolumn{9}{c}{\textit{Qwen2.5-7B-Instruct}} \\
            \midrule

            Original
            & $9.8_{\pm 0.1}$
            & $7.5_{\pm 0.3}$
            & $73.0_{\pm 0.2}$
            & $10.8_{\pm 0.4}$
            & $33.3_{\pm 0.1}$
            & $58.7_{\pm 0.3}$
            & $15.7_{\pm 0.1}$
            & $9.6_{\pm 0.2}$ \\

            GRPO
            & $23.3_{\pm 0.4}$
            & $18.9_{\pm 0.3}$
            & $78.4_{\pm 0.1}$
            & $13.9_{\pm 0.2}$
            & $38.4_{\pm 0.5}$
            & $61.3_{\pm 0.1}$
            & $18.1_{\pm 0.4}$
            & $14.5_{\pm 0.3}$ \\

            PRIME
            & $22.2_{\pm 0.6}$
            & $18.3_{\pm 0.4}$
            & $76.5_{\pm 0.3}$
            & $11.5_{\pm 0.5}$
            & $35.8_{\pm 0.7}$
            & $57.7_{\pm 0.4}$
            & $16.3_{\pm 0.6}$
            & $12.2_{\pm 0.3}$ \\

            Dr.GRPO
            & $23.7_{\pm 0.3}$
            & $19.2_{\pm 0.2}$
            & $78.8_{\pm 0.1}$
            & $14.0_{\pm 0.1}$
            & $38.6_{\pm 0.4}$
            & $61.5_{\pm 0.1}$
            & $18.4_{\pm 0.3}$
            & $15.0_{\pm 0.1}$ \\

            \midrule
            \rowcolor{lightgray!30}
            \multicolumn{9}{c}{
                \textit{Qwen2.5-7B-Instruct}
                \; $\leftrightarrow$ \;
                \textit{Qwen2.5-7B-Instruct}
            } \\
            \midrule

            SFT
            & $20.0_{\pm 0.5}$
            & $17.1_{\pm 0.4}$
            & $73.4_{\pm 0.3}$
            & $11.1_{\pm 0.5}$
            & $33.6_{\pm 0.7}$
            & $59.0_{\pm 0.4}$
            & $15.9_{\pm 0.6}$
            & $9.8_{\pm 0.3}$ \\

            OPSD
            & $22.8_{\pm 0.4}$
            & $18.7_{\pm 0.3}$
            & $77.0_{\pm 0.1}$
            & $12.8_{\pm 0.3}$
            & $38.5_{\pm 0.5}$
            & $60.1_{\pm 0.2}$
            & $17.0_{\pm 0.4}$
            & $12.4_{\pm 0.3}$ \\

            LUFFY
            & $23.8_{\pm 0.6}$
            & $18.4_{\pm 0.5}$
            & $76.8_{\pm 0.3}$
            & $13.2_{\pm 0.4}$
            & $39.1_{\pm 0.6}$
            & $60.7_{\pm 0.3}$
            & $17.3_{\pm 0.5}$
            & $13.9_{\pm 0.4}$ \\

            \rowcolor{cyan!10}
            \name
            & $\textbf{26.5}_{\pm 0.2}$
            & $\textbf{21.3}_{\pm 0.1}$
            & $\textbf{80.5}_{\pm 0.1}$
            & $\textbf{14.9}_{\pm 0.1}$
            & $\textbf{41.7}_{\pm 0.3}$
            & $\textbf{62.7}_{\pm 0.1}$
            & $\textbf{19.2}_{\pm 0.2}$
            & $\textbf{16.0}_{\pm 0.1}$ \\

            \midrule
            \rowcolor{lightgray!30}
            \multicolumn{9}{c}{
                \textit{Qwen2.5-7B-Instruct}
                \; $\leftrightarrow$ \;
                \textit{Qwen3-30B-A3B-Instruct-2507}
            } \\
            \midrule

            SFT
            & $22.4_{\pm 0.5}$
            & $18.8_{\pm 0.4}$
            & $76.0_{\pm 0.3}$
            & $12.5_{\pm 0.5}$
            & $34.3_{\pm 0.7}$
            & $59.5_{\pm 0.4}$
            & $16.4_{\pm 0.6}$
            & $10.5_{\pm 0.3}$ \\

            OPSD
            & $25.8_{\pm 0.4}$
            & $20.5_{\pm 0.3}$
            & $79.2_{\pm 0.1}$
            & $14.2_{\pm 0.3}$
            & $40.4_{\pm 0.5}$
            & $61.3_{\pm 0.2}$
            & $18.1_{\pm 0.4}$
            & $14.7_{\pm 0.3}$ \\

            LUFFY
            & $26.7_{\pm 0.6}$
            & $21.2_{\pm 0.4}$
            & $80.9_{\pm 0.2}$
            & $15.1_{\pm 0.4}$
            & $41.8_{\pm 0.4}$
            & $62.6_{\pm 0.3}$
            & $19.2_{\pm 0.5}$
            & $15.7_{\pm 0.4}$ \\

            \rowcolor{cyan!10}
            \name
            & $\textbf{28.7}_{\pm 0.6}$
            & $\textbf{22.4}_{\pm 0.4}$
            & $\textbf{82.3}_{\pm 0.2}$
            & $\textbf{17.5}_{\pm 0.4}$
            & $\textbf{42.9}_{\pm 0.4}$
            & $\textbf{63.9}_{\pm 0.3}$
            & $\textbf{21.1}_{\pm 0.5}$
            & $\textbf{17.8}_{\pm 0.4}$ \\

            \midrule
            \rowcolor{lightgray!30}
            \multicolumn{9}{c}{
                \textit{Llama3.2-3B-Instruct}
                \; $\leftrightarrow$ \;
                \textit{Llama3.2-3B-Instruct}
            } \\
            \midrule

            SFT
            & $10.6_{\pm 0.5}$ 
            & $8.6_{\pm 0.4}$ 
            & $61.1_{\pm 0.5}$ 
            & $5.1_{\pm 0.5}$
            & $34.0_{\pm 0.6}$ 
            & $38.6_{\pm 0.4}$ 
            & $6.2_{\pm 0.5}$ 
            & $4.1_{\pm 0.4}$ \\

            OPSD
            & $14.5_{\pm 0.5}$ 
            & $9.8_{\pm 0.5}$ 
            & $63.2_{\pm 0.4}$ 
            & $7.0_{\pm 0.4}$
            & $35.4_{\pm 0.4}$ 
            & $39.8_{\pm 0.5}$ 
            & $7.5_{\pm 0.4}$ 
            & $4.6_{\pm 0.5}$ \\

            LUFFY
            & $13.6_{\pm 0.6}$ 
            & $8.9_{\pm 0.4}$ 
            & $62.6_{\pm 0.5}$ 
            & $4.8_{\pm 0.6}$
            & $34.7_{\pm 0.5}$ 
            & $39.1_{\pm 0.4}$ 
            & $8.7_{\pm 0.4}$ 
            & $4.9_{\pm 0.4}$ \\

            \rowcolor{cyan!10}
            \name
            & $\textbf{17.9}_{\pm 0.5}$
            & $\textbf{11.9}_{\pm 0.4}$
            & $\textbf{67.7}_{\pm 0.5}$
            & $\textbf{10.8}_{\pm 0.5}$
            & $\textbf{36.4}_{\pm 0.4}$
            & $\textbf{43.0}_{\pm 0.4}$
            & $\textbf{10.1}_{\pm 0.4}$
            & $\textbf{11.0}_{\pm 0.5}$ \\

            \midrule
            \rowcolor{lightgray!30}
            \multicolumn{9}{c}{
                \textit{Llama3.2-3B-Instruct}
                \; $\leftrightarrow$ \;
                \textit{Qwen3-30B-A3B-Instruct-2507}
            } \\
            \midrule

            SFT
            & $12.4_{\pm 0.4}$
            & $10.5_{\pm 0.5}$ 
            & $62.9_{\pm 0.6}$ 
            & $7.0_{\pm 0.4}$
            & $35.8_{\pm 0.5}$ 
            & $40.4_{\pm 0.5}$ 
            & $8.0_{\pm 0.4}$ 
            & $6.0_{\pm 0.5}$ \\

            OPSD
            & $17.2_{\pm 0.5}$ 
            & $11.8_{\pm 0.4}$ 
            & $67.0_{\pm 0.5}$ 
            & $9.5_{\pm 0.5}$
            & $35.7_{\pm 0.6}$ 
            & $43.5_{\pm 0.4}$ 
            & $9.6_{\pm 0.5}$ 
            & $10.4_{\pm 0.4}$ \\

            LUFFY
            & $18.7_{\pm 0.4}$ 
            & $12.9_{\pm 0.5}$ 
            & $67.5_{\pm 0.5}$ 
            & $10.8_{\pm 0.6}$
            & $36.7_{\pm 0.5}$ 
            & $44.2_{\pm 0.4}$ 
            & $10.8_{\pm 0.5}$ 
            & $11.9_{\pm 0.5}$ \\

            \rowcolor{cyan!10}
            \name
            & $\textbf{20.1}_{\pm 0.5}$
            & $\textbf{14.4}_{\pm 0.4}$
            & $\textbf{69.9}_{\pm 0.5}$
            & $\textbf{13.3}_{\pm 0.4}$
            & $\textbf{38.9}_{\pm 0.5}$
            & $\textbf{45.4}_{\pm 0.3}$
            & $\textbf{12.4}_{\pm 0.4}$
            & $\textbf{13.3}_{\pm 0.4}$ \\

            \bottomrule
        \end{tabular}
    }
    \vspace{-1.5em}
\end{table}

\vspace{-0.5em}
\paragraph{\name Outperforms Salient Baselines.}
As shown in \Cref{tab:main_performance}, \name achieves superior performance over strong RLVR baselines, demonstrating the effectiveness of our approach in enhancing reasoning capabilities. 
Compared with standard SFT, OPSD, and LUFFY, \name obtains consistent improvements across all benchmarks. 
These results support our claim that active interactions with external collaborators can substantially improve policy exploration and thereby lead to better reasoning performance. 
Additional efficiency analyses in \Cref{app:efficiency} show that these gains persist under a matched computational budget, while quantifying the additional training-time cost introduced by collaborator interaction.

\vspace{-0.5em}
\paragraph{\name Generalizes to Different Policy Models.}
The benefits of \name are not tied to a specific model architecture, but generalize across different foundation models.
When applied to both Qwen2.5-7B-Instruct and Llama3.2-3B-Instruct, \name yields consistent improvements over the LUFFY baseline, while additional RLVR results on Llama3.2-3B-Instruct are reported in \Cref{app:llama3_rl_performance}.
We further evaluate \name on the larger Qwen2.5-32B-Instruct model and the long-CoT DeepSeek-R1-Distill-Qwen-1.5B model in \Cref{app:larger_models,app:long_cot_models}, respectively, demonstrating consistent effectiveness across model scales and reasoning behaviors.
Notably, even when the policy model and collaborators share the same backbone, \name still outperforms standard RL algorithms.
This indicates that the gains of \name do not simply come from using a stronger external model, but are primarily driven by enhanced exploration through active interaction.

\vspace{-0.5em}
\paragraph{\name Generalizes to Different RLVR Algorithms.}
To evaluate the generalization capability of \name across different RLVR algorithms, we further implement \name with both DAPO~\citep{abs-2503-14476} and GSPO~\citep{abs-2507-18071}. 
The results, presented in \Cref{app:rlvr_algorithms}, show that \name delivers consistent performance improvements regardless of the underlying RLVR algorithm, thereby demonstrating its strong generalizability. 

\vspace{-0.5em}
\paragraph{\name Generalizes to Out-of-Domain Benchmarks.}
Beyond standard mathematical and scientific reasoning tasks, \name exhibits strong generalization across distinct domains. The results show consistent improvements over baselines on code generation benchmarks, including MBPP and LCB, as well as complex puzzle tasks in RG. This indicates that the enhanced exploration enabled by \name transfers beyond the training distribution and benefits broader reasoning scenarios. 

\vspace{-0.5em}
\paragraph{\name Benefits from Stronger Collaborators.}
\name can effectively exploit additional information provided by a stronger collaborator.
Replacing same-backbone collaborators with Qwen3-30B-A3B-Instruct-2507, whose standalone performance is reported in \Cref{tab:teacher_performance}, improves the overall performance of both policy-model families.
These results indicate that collaborator capability is one factor affecting the usefulness of interactive exploration.
However, the benefit is not necessarily monotonic: once collaborator feedback exceeds what the policy can effectively internalize, further increases in collaborator capability need not translate proportionally into policy improvement.
We analyze this effect further in \Cref{app:collaborator_capability}.

\subsection{Ablation Study}

\begin{table}[t]
    \centering
    \caption{Ablation study of \name's components w.r.t. the corrected importance sampling coefficient~(abbreviated as IS, \Cref{eq:modified_importance_sampling}) and clipping strategy~(abbreviated as CS, \Cref{eq:clipping_strategy}). We also include the results where we masked the external tokens generated by collaborators.} \label{tab:ablation}
    \resizebox{1.\textwidth}{!}{
        \begin{tabular}{lcccccccc}
            \toprule
            \multirow{3}{*}{\bf Methods} 
            & \multicolumn{4}{c}{\bf Math \(\spadesuit\)} 
            & \multicolumn{1}{c}{\bf Science \(\spadesuit\)} 
            & \multicolumn{2}{c}{\bf Code \(\clubsuit\)} 
            & \multicolumn{1}{c}{\bf Puzzle \(\clubsuit\)} \\
            \cmidrule(r){2-5} \cmidrule(r){6-6} \cmidrule(r){7-8} \cmidrule(r){9-9}
            & \bf AIME24 & \bf AIME25 & \bf MATH500 & \bf LMB & \bf GPQA-D & \bf MBPP & \bf LCB & \bf RG \\
            & \bf Avg@\(16\) & \bf Avg@\(16\) & \bf Avg@\(4\) & \bf Avg@\(4\) & \bf Avg@\(4\) & \bf Avg@\(4\) & \bf Avg@\(4\) & \bf Avg@\(4\) \\
            \midrule
            \rowcolor{cyan!10} \name 
            & $\textbf{26.5}_{\pm 0.2}$
            & $\textbf{21.3}_{\pm 0.1}$
            & $\textbf{80.5}_{\pm 0.1}$
            & $\textbf{14.9}_{\pm 0.1}$
            & $\textbf{41.7}_{\pm 0.3}$
            & $\textbf{62.7}_{\pm 0.1}$
            & $\textbf{19.2}_{\pm 0.2}$
            & $\textbf{16.0}_{\pm 0.1}$ \\
            \; \textit{w/o} IS 
            & $25.6_{\pm 0.4}$\drop{0.9} 
            & $20.5_{\pm 0.3}$\drop{0.8} 
            & $77.5_{\pm 0.5}$\drop{3.0} 
            & $13.7_{\pm 0.5}$\drop{1.2} 
            & $40.0_{\pm 0.5}$\drop{1.7} 
            & $61.6_{\pm 0.4}$\drop{1.1} 
            & $16.5_{\pm 0.6}$\drop{2.7} 
            & $15.1_{\pm 0.4}$\drop{0.9} \\
            \; \textit{w/o} CS 
            & $25.5_{\pm 0.4}$\drop{1.0} 
            & $19.8_{\pm 0.5}$\drop{1.5} 
            & $76.9_{\pm 0.6}$\drop{3.6} 
            & $10.6_{\pm 0.4}$\drop{4.3} 
            & $38.6_{\pm 0.5}$\drop{3.1} 
            & $60.3_{\pm 0.5}$\drop{2.4} 
            & $15.4_{\pm 0.5}$\drop{3.8} 
            & $14.2_{\pm 0.5}$\drop{1.8} \\
            \; Mask   
            & $25.0_{\pm 0.5}$\drop{1.5} 
            & $19.4_{\pm 0.4}$\drop{1.9} 
            & $79.7_{\pm 0.3}$\drop{0.8} 
            & $11.9_{\pm 0.5}$\drop{3.0} 
            & $37.8_{\pm 0.6}$\drop{3.9} 
            & $59.7_{\pm 0.4}$\drop{3.0} 
            & $15.1_{\pm 0.5}$\drop{4.1} 
            & $13.4_{\pm 0.5}$\drop{2.6} \\
            \bottomrule
        \end{tabular}
    }
    \vspace{-1.0em}
\end{table}

\vspace{-0.5em} 
We ablate the three components of \name in \Cref{tab:ablation} and \Cref{fig:ablation}.
Replacing the corrected importance sampling coefficient (\Cref{eq:modified_importance_sampling}) with the standard ratio computed directly under the policy model consistently degrades performance, indicating that it is an effective approximation of the policy--collaborator discrepancy and mitigates off-policy bias.
Removing the clipping strategy (\Cref{eq:clipping_strategy}) causes a substantial drop across all benchmarks.
Complementary diagnostics in \Cref{app:vanishing_stats} show that a substantial fraction of external tokens initially fall into the near-vanishing regime, directly motivating the proposed lower-bound clipping.
Restricting the policy to a single collaborator reveals a clear division of labour: the Knowledge Agent contributes most on knowledge-intensive benchmarks such as science, while the Verify and Reasoning Agents are more effective elsewhere.
Nevertheless, any single collaborator still outperforms the baseline.
Complementary usage statistics in \Cref{app:interaction_stats} show that the learned policy does not collapse onto a single collaborator, but continues to make substantial use of all three agent types during training.

\subsection{Does \name Expand the Effective Exploration Boundary?}
\label{sec:boundary}

\begin{figure}[t]
    \centering
    \begin{minipage}[b]{0.52\textwidth}
        \centering
        \includegraphics[width=\textwidth]{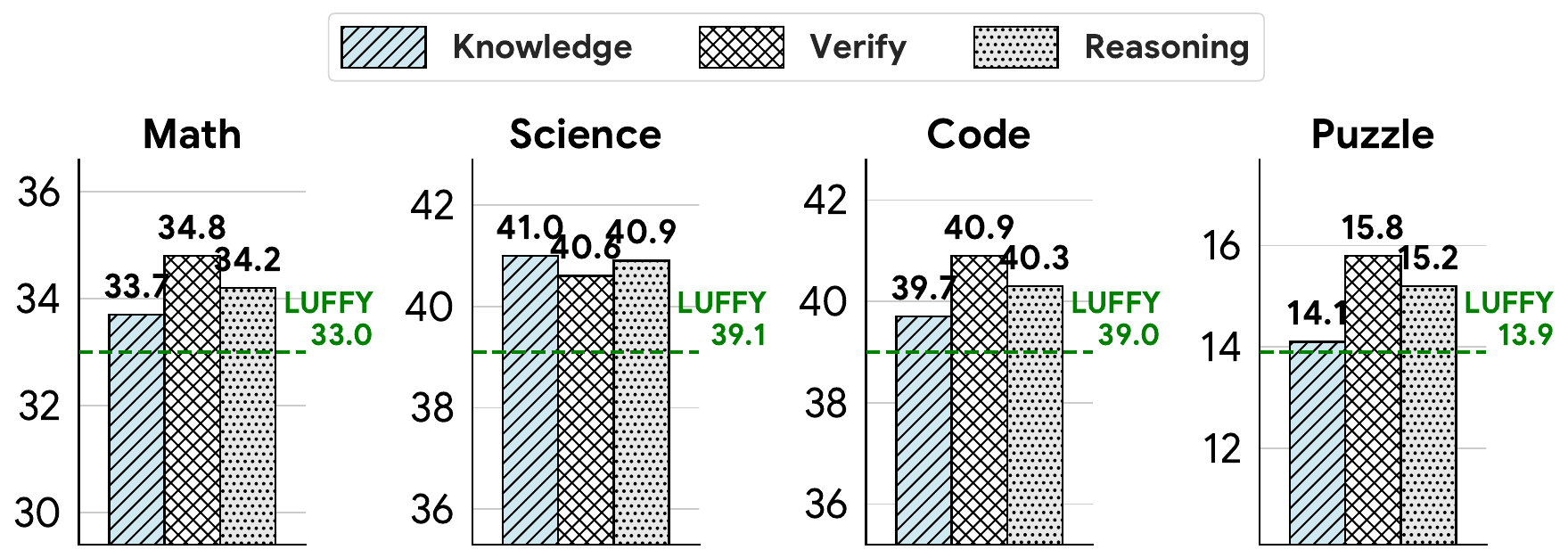}
        \caption{Ablation Study of the collaborators in \name. Each bar indicates the average performance of all benchmarks in this domain.} 
        \label{fig:ablation}
    \end{minipage}
    \hfill
    \begin{minipage}[b]{0.44\textwidth}
        \centering
        \includegraphics[width=\textwidth]{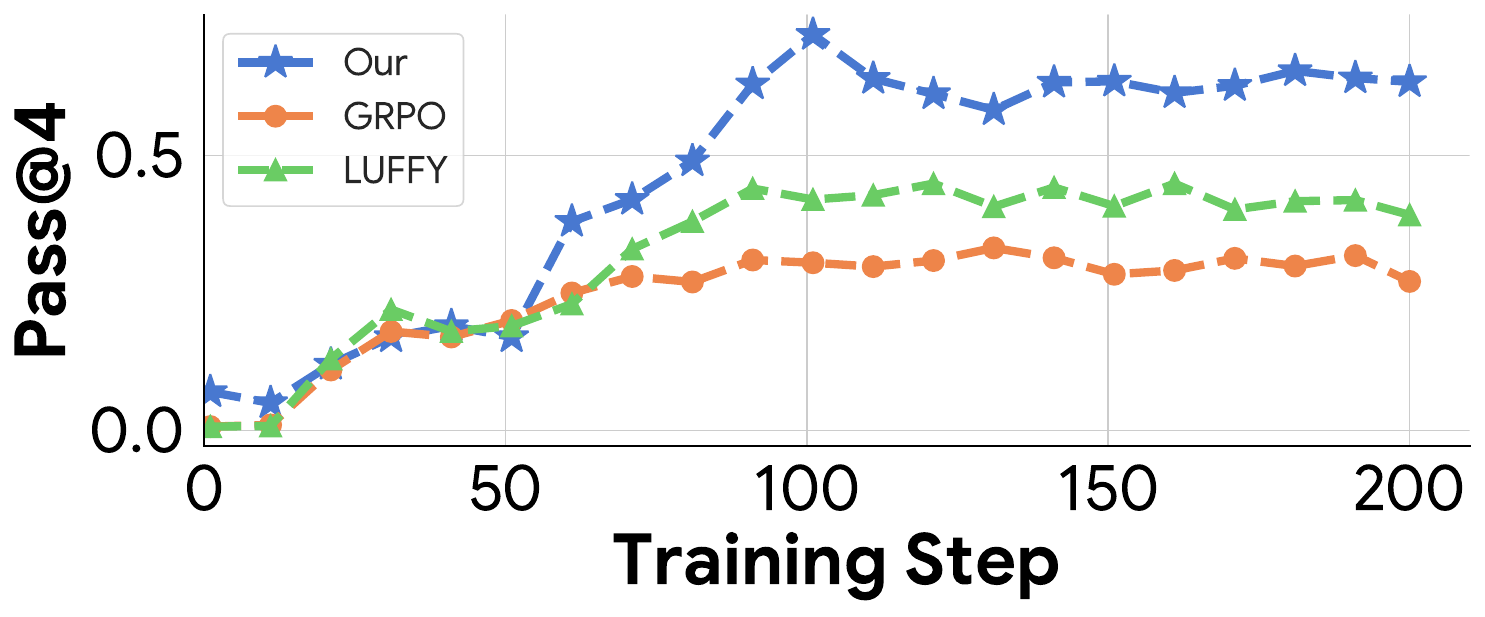}
        \caption{Expansion of the effective exploration boundary on DeepMath, measured by Pass@$4$. } 
        \label{fig:knowledge_boundary}
    \end{minipage}
    \vspace{-2.0em}
\end{figure} 

\vspace{-0.5em} 
We examine what \name actually learns through its Pass@$4$ behaviour on DeepMath~\citep{abs-2504-11456}, which is disjoint from our training corpus, and through its training dynamics.
As shown in \Cref{fig:knowledge_boundary}, all three methods are indistinguishable for the first $\sim$50 steps, after which \name separates sharply, peaks around step 100, and then stabilises at roughly $0.63$, against $0.42$ for LUFFY and $0.30$ for GRPO. 
\Cref{fig:training_dynamics} traces the underlying mechanism.
\name drives the number of tasks failed across \emph{all} rollouts in a batch (\textit{Batch Failed Tasks}) from $\sim\!150$ down to $\sim\!45$, whereas GRPO stalls near $130$, and batch accuracy rises to $\sim\!0.48$ against $\sim\!0.26$ for GRPO.
Interaction frequency (\textit{Batch Interactions}) climbs steeply over the first $\sim$60 steps, then declines and stabilises, consistent with a policy that relies more heavily on collaborators early in training and increasingly solves problems independently as it internalises their guidance.
Detailed invocation statistics in \Cref{app:interaction_stats} further show that approximately $10\%$ of training samples require no interaction and that the policy continues to use all three collaborator types rather than collapsing onto a single agent.
Masking external tokens severs exactly this internalisation: interaction frequency never leaves its initial level, and although batch accuracy still improves to $\sim\!0.31$, \textit{Batch Failed Tasks} remains flat at $\sim\!150$ throughout training.
The masked variant therefore gets better at problems it could already solve while never cracking the ones it could not, which is precisely the difference between consuming collaborator information and integrating it into the policy's own parameters (\Cref{tab:ablation}).

\begin{figure}[t]
    \centering
    \includegraphics[width=.95\textwidth]{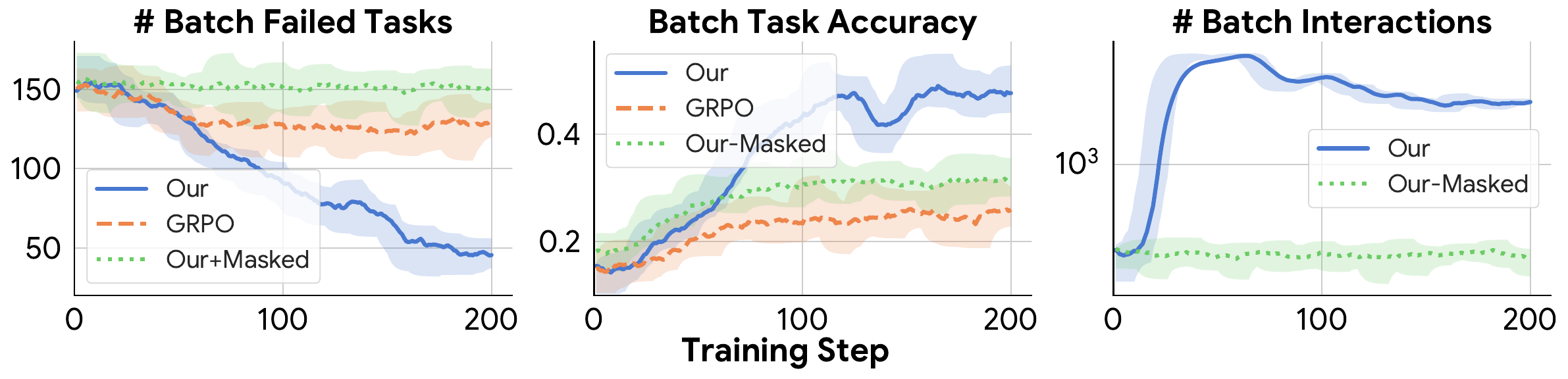}
    \caption{Training Dynamics of \name and baselines on Qwen2.5-7B-Instruct with the same model as collaborators.} \label{fig:training_dynamics} 
    \vspace{-2.0em}
\end{figure} 

\section{Related Work} 

\vspace{-0.5em}
\paragraph{Reinforcement Learning for LLM Reasoning.}
Reinforcement Learning with Verifiable Rewards (RLVR) has become a widely used paradigm for improving LLM reasoning~\citep{openai2024o3,openai2025gpt5,abs-2501-12948,abs-2501-12599,abs-2505-09388,abs-2507-06261}.
Recent methods further improve RLVR through advances in sampling, reward design, and policy optimization, including Dr.GRPO~\citep{abs-2503-20783}, VAPO~\citep{abs-2504-05118}, and DAPO~\citep{abs-2503-14476}.
Despite these advances, purely on-policy RLVR remains limited by the reasoning trajectories that the current policy can discover during rollout~\citep{abs-2504-07912,abs-2504-13837}.

\vspace{-0.5em}
\paragraph{External-Guidance-Enhanced Learning.}
A growing line of work introduces external supervision through expert trajectories~\citep{abs-2504-14945,abs-2509-04419,abs-2508-11408,abs-2509-06948,abs-2506-05316}, teacher distributions~\citep{abs-2506-02208,abs-2601-18734}, critiques~\citep{ShinnCGNY23,abs-2509-26306}, or supervised expert data~\citep{abs-2506-07527,abs-2506-19767}.
In contrast to methods that provide supervision over complete trajectories or policy-generated rollouts, \name introduces targeted collaborator feedback during rollout, producing mixed-policy trajectories conditioned on the policy's current reasoning bottlenecks.
The collaborators are used only during training, while the resulting policy performs inference independently.
Detailed comparisons with on-policy distillation, off-policy expert learning, multi-agent guidance, and self-reflection methods are provided in \Cref{app:related_work}.

\section{Conclusion}

\vspace{-0.5em}
In this paper, we introduce \textbf{\name}, an enhanced RLVR framework that broadens policy exploration through active interaction with external collaborators during training.
Rather than relying on complete expert trajectories, \name allows the policy to selectively request verification, knowledge, or sub-problem reasoning when it encounters intermediate bottlenecks, producing mixed-policy trajectories with fine-grained and state-dependent guidance.
To learn effectively from the externally generated tokens in these trajectories, we introduce an corrected importance sampling coefficient and a lower-bound clipping strategy that mitigate off-policy discrepancy and vanishing gradients.
Across diverse mathematical, scientific, coding, and puzzle benchmarks, \name consistently improves over representative RLVR, distillation, and expert-guided baselines and generalizes across policy models, collaborator backbones, and RLVR algorithms.
Further analyses indicate that training-time interaction broadens the policy's effective exploration boundary, while the trained policy operates independently without collaborators or additional inference-time computation.

\clearpage
\bibliography{refs}

@misc{openai2024o1,
  author       = {{OpenAI}},
  title        = {Learning to reason with llms},
  howpublished = {\url{https://openai.com/index/learning-to-reason-with-llms/}},
  note         = {Accessed: 2024-09},
  year         = {2024}
}

@misc{openai2024o3,
  author       = {{OpenAI}},
  title        = {Introducing OpenAI o3 and o4-mini},
  howpublished = {\url{https://openai.com/index/introducing-o3-and-o4-mini/}},
  note         = {Accessed: 2024-12},
  year         = {2024}
}

@misc{openai2025gpt5,
  author       = {{OpenAI}},
  title        = {GPT-5 and the new era of work},
  howpublished = {\url{https://openai.com/index/gpt-5-new-era-of-work/}},
  note         = {Accessed: 2025-08},
  year         = {2025}
}

@misc{xai2025grok4,
  author       = {{xAI}},
  title        = {Grok 4},
  howpublished = {\url{https://x.ai/news/grok-4/}},
  note         = {Accessed: 2025-07},
  year         = {2025}
}

@article{abs-2501-12948,
  author       = {DeepSeek{-}AI and
                  Daya Guo and
                  Dejian Yang and
                  Haowei Zhang and
                  Junxiao Song and
                  Ruoyu Zhang and
                  Runxin Xu and
                  Qihao Zhu and
                  Shirong Ma and
                  Peiyi Wang and
                  Xiao Bi and
                  Xiaokang Zhang and
                  Xingkai Yu and
                  Yu Wu and
                  Z. F. Wu and
                  Zhibin Gou and
                  Zhihong Shao and
                  Zhuoshu Li and
                  Ziyi Gao and
                  Aixin Liu and
                  Bing Xue and
                  Bingxuan Wang and
                  Bochao Wu and
                  Bei Feng and
                  Chengda Lu and
                  Chenggang Zhao and
                  Chengqi Deng and
                  Chenyu Zhang and
                  Chong Ruan and
                  Damai Dai and
                  Deli Chen and
                  Dongjie Ji and
                  Erhang Li and
                  Fangyun Lin and
                  Fucong Dai and
                  Fuli Luo and
                  Guangbo Hao and
                  Guanting Chen and
                  Guowei Li and
                  H. Zhang and
                  Han Bao and
                  Hanwei Xu and
                  Haocheng Wang and
                  Honghui Ding and
                  Huajian Xin and
                  Huazuo Gao and
                  Hui Qu and
                  Hui Li and
                  Jianzhong Guo and
                  Jiashi Li and
                  Jiawei Wang and
                  Jingchang Chen and
                  Jingyang Yuan and
                  Junjie Qiu and
                  Junlong Li and
                  J. L. Cai and
                  Jiaqi Ni and
                  Jian Liang and
                  Jin Chen and
                  Kai Dong and
                  Kai Hu and
                  Kaige Gao and
                  Kang Guan and
                  Kexin Huang and
                  Kuai Yu and
                  Lean Wang and
                  Lecong Zhang and
                  Liang Zhao and
                  Litong Wang and
                  Liyue Zhang and
                  Lei Xu and
                  Leyi Xia and
                  Mingchuan Zhang and
                  Minghua Zhang and
                  Minghui Tang and
                  Meng Li and
                  Miaojun Wang and
                  Mingming Li and
                  Ning Tian and
                  Panpan Huang and
                  Peng Zhang and
                  Qiancheng Wang and
                  Qinyu Chen and
                  Qiushi Du and
                  Ruiqi Ge and
                  Ruisong Zhang and
                  Ruizhe Pan and
                  Runji Wang and
                  R. J. Chen and
                  R. L. Jin and
                  Ruyi Chen and
                  Shanghao Lu and
                  Shangyan Zhou and
                  Shanhuang Chen and
                  Shengfeng Ye and
                  Shiyu Wang and
                  Shuiping Yu and
                  Shunfeng Zhou and
                  Shuting Pan and
                  S. S. Li},
  title        = {DeepSeek-R1: Incentivizing Reasoning Capability in LLMs via Reinforcement
                  Learning},
  journal      = {CoRR},
  volume       = {abs/2501.12948},
  year         = {2025}
}

@article{abs-2501-12599,
  author       = {Kimi-Team and
                  Angang Du and
                  Bofei Gao and
                  Bowei Xing and
                  Changjiu Jiang and
                  Cheng Chen and
                  Cheng Li and
                  Chenjun Xiao and
                  Chenzhuang Du and
                  Chonghua Liao and
                  Chuning Tang and
                  Congcong Wang and
                  Dehao Zhang and
                  Enming Yuan and
                  Enzhe Lu and
                  Fengxiang Tang and
                  Flood Sung and
                  Guangda Wei and
                  Guokun Lai and
                  Haiqing Guo and
                  Han Zhu and
                  Hao Ding and
                  Hao Hu and
                  Hao Yang and
                  Hao Zhang and
                  Haotian Yao and
                  Haotian Zhao and
                  Haoyu Lu and
                  Haoze Li and
                  Haozhen Yu and
                  Hongcheng Gao and
                  Huabin Zheng and
                  Huan Yuan and
                  Jia Chen and
                  Jianhang Guo and
                  Jianlin Su and
                  Jianzhou Wang and
                  Jie Zhao and
                  Jin Zhang and
                  Jingyuan Liu and
                  Junjie Yan and
                  Junyan Wu and
                  Lidong Shi and
                  Ling Ye and
                  Longhui Yu and
                  Mengnan Dong and
                  Neo Zhang and
                  Ningchen Ma and
                  Qiwei Pan and
                  Qucheng Gong and
                  Shaowei Liu and
                  Shengling Ma and
                  Shupeng Wei and
                  Sihan Cao and
                  Siying Huang and
                  Tao Jiang and
                  Weihao Gao and
                  Weimin Xiong and
                  Weiran He and
                  Weixiao Huang and
                  Wenhao Wu and
                  Wenyang He and
                  Xianghui Wei and
                  Xianqing Jia and
                  Xingzhe Wu and
                  Xinran Xu and
                  Xinxing Zu and
                  Xinyu Zhou and
                  Xuehai Pan and
                  Y. Charles and
                  Yang Li and
                  Yangyang Hu and
                  Yangyang Liu and
                  Yanru Chen and
                  Yejie Wang and
                  Yibo Liu and
                  Yidao Qin and
                  Yifeng Liu and
                  Ying Yang and
                  Yiping Bao and
                  Yulun Du and
                  Yuxin Wu and
                  Yuzhi Wang and
                  Zaida Zhou and
                  Zhaoji Wang and
                  Zhaowei Li and
                  Zhen Zhu and
                  Zheng Zhang and
                  Zhexu Wang and
                  Zhilin Yang and
                  Zhiqi Huang and
                  Zihao Huang and
                  Ziyao Xu and
                  Zonghan Yang},
  title        = {Kimi k1.5: Scaling Reinforcement Learning with LLMs},
  journal      = {CoRR},
  volume       = {abs/2501.12599},
  year         = {2025}
}

@article{abs-2507-06261,
  author       = {Gheorghe Comanici and
                  Eric Bieber and
                  Mike Schaekermann and
                  Ice Pasupat and
                  Noveen Sachdeva and
                  Inderjit S. Dhillon and
                  Marcel Blistein and
                  Ori Ram and
                  Dan Zhang and
                  Evan Rosen and
                  Luke Marris and
                  Sam Petulla and
                  Colin Gaffney and
                  Asaf Aharoni and
                  Nathan Lintz and
                  Tiago Cardal Pais and
                  Henrik Jacobsson and
                  Idan Szpektor and
                  Nan{-}Jiang Jiang and
                  Krishna Haridasan and
                  Ahmed Omran and
                  Nikunj Saunshi and
                  Dara Bahri and
                  Gaurav Mishra and
                  Eric Chu and
                  Toby Boyd and
                  Brad Hekman and
                  Aaron Parisi and
                  Chaoyi Zhang and
                  Kornraphop Kawintiranon and
                  Tania Bedrax{-}Weiss and
                  Oliver Wang and
                  Ya Xu and
                  Ollie Purkiss and
                  Uri Mendlovic and
                  Ila{\"{\i}} Deutel and
                  Nam Nguyen and
                  Adam Langley and
                  Flip Korn and
                  Lucia Rossazza and
                  Alexandre Ram{\'{e}} and
                  Sagar Waghmare and
                  Helen Miller and
                  Nathan Byrd and
                  Ashrith Sheshan and
                  Raia Hadsell Sangnie Bhardwaj and
                  Pawel Janus and
                  Tero Rissa and
                  Dan Horgan and
                  Sharon Silver and
                  Ayzaan Wahid and
                  Sergey Brin and
                  Yves Raimond and
                  Klemen Kloboves and
                  Cindy Wang and
                  Nitesh Bharadwaj Gundavarapu and
                  Ilia Shumailov and
                  Bo Wang and
                  Mantas Pajarskas and
                  Joe Heyward and
                  Martin Nikoltchev and
                  Maciej Kula and
                  Hao Zhou and
                  Zachary Garrett and
                  Sushant Kafle and
                  Sercan Arik and
                  Ankita Goel and
                  Mingyao Yang and
                  Jiho Park and
                  Koji Kojima and
                  Parsa Mahmoudieh and
                  Koray Kavukcuoglu and
                  Grace Chen and
                  Doug Fritz and
                  Anton Bulyenov and
                  Sudeshna Roy and
                  Dimitris Paparas and
                  Hadar Shemtov and
                  Bo{-}Juen Chen and
                  Robin Strudel and
                  David Reitter and
                  Aurko Roy and
                  Andrey Vlasov and
                  Changwan Ryu and
                  Chas Leichner and
                  Haichuan Yang and
                  Zelda Mariet and
                  Denis Vnukov and
                  Tim Sohn and
                  Amy Stuart and
                  Wei Liang and
                  Minmin Chen and
                  Praynaa Rawlani and
                  Christy Koh and
                  JD Co{-}Reyes and
                  Guangda Lai and
                  Praseem Banzal and
                  Dimitrios Vytiniotis and
                  Jieru Mei and
                  Mu Cai},
  title        = {Gemini 2.5: Pushing the Frontier with Advanced Reasoning, Multimodality,
                  Long Context, and Next Generation Agentic Capabilities},
  journal      = {CoRR},
  volume       = {abs/2507.06261},
  year         = {2025}
}

@article{abs-2503-09567,
  author       = {Qiguang Chen and
                  Libo Qin and
                  Jinhao Liu and
                  Dengyun Peng and
                  Jiannan Guan and
                  Peng Wang and
                  Mengkang Hu and
                  Yuhang Zhou and
                  Te Gao and
                  Wanxiang Che},
  title        = {Towards Reasoning Era: {A} Survey of Long Chain-of-Thought for Reasoning
                  Large Language Models},
  journal      = {CoRR},
  volume       = {abs/2503.09567},
  year         = {2025}
}

@article{abs-2402-03300,
  author       = {Zhihong Shao and
                  Peiyi Wang and
                  Qihao Zhu and
                  Runxin Xu and
                  Junxiao Song and
                  Mingchuan Zhang and
                  Y. K. Li and
                  Y. Wu and
                  Daya Guo},
  title        = {DeepSeekMath: Pushing the Limits of Mathematical Reasoning in Open
                  Language Models},
  journal      = {CoRR},
  volume       = {abs/2402.03300},
  year         = {2024}
}

@article{SchulmanWDRK17,
  author       = {John Schulman and
                  Filip Wolski and
                  Prafulla Dhariwal and
                  Alec Radford and
                  Oleg Klimov},
  title        = {Proximal Policy Optimization Algorithms},
  journal      = {CoRR},
  volume       = {abs/1707.06347},
  year         = {2017}
}

@article{abs-2408-03314,
  author       = {Charlie Snell and
                  Jaehoon Lee and
                  Kelvin Xu and
                  Aviral Kumar},
  title        = {Scaling {LLM} Test-Time Compute Optimally can be More Effective than
                  Scaling Model Parameters},
  journal      = {CoRR},
  volume       = {abs/2408.03314},
  year         = {2024}
}

@inproceedings{LightmanKBEBLLS24,
  author       = {Hunter Lightman and
                  Vineet Kosaraju and
                  Yuri Burda and
                  Harrison Edwards and
                  Bowen Baker and
                  Teddy Lee and
                  Jan Leike and
                  John Schulman and
                  Ilya Sutskever and
                  Karl Cobbe},
  title        = {Let's Verify Step by Step},
  booktitle    = {{ICLR}},
  publisher    = {OpenReview.net},
  year         = {2024}
}

@article{abs-2505-09388,
  author       = {An Yang and
                  Anfeng Li and
                  Baosong Yang and
                  Beichen Zhang and
                  Binyuan Hui and
                  Bo Zheng and
                  Bowen Yu and
                  Chang Gao and
                  Chengen Huang and
                  Chenxu Lv and
                  Chujie Zheng and
                  Dayiheng Liu and
                  Fan Zhou and
                  Fei Huang and
                  Feng Hu and
                  Hao Ge and
                  Haoran Wei and
                  Huan Lin and
                  Jialong Tang and
                  Jian Yang and
                  Jianhong Tu and
                  Jianwei Zhang and
                  Jian Yang and
                  Jiaxi Yang and
                  Jingren Zhou and
                  Jingren Zhou and
                  Junyang Lin and
                  Kai Dang and
                  Keqin Bao and
                  Kexin Yang and
                  Le Yu and
                  Lianghao Deng and
                  Mei Li and
                  Mingfeng Xue and
                  Mingze Li and
                  Pei Zhang and
                  Peng Wang and
                  Qin Zhu and
                  Rui Men and
                  Ruize Gao and
                  Shixuan Liu and
                  Shuang Luo and
                  Tianhao Li and
                  Tianyi Tang and
                  Wenbiao Yin and
                  Xingzhang Ren and
                  Xinyu Wang and
                  Xinyu Zhang and
                  Xuancheng Ren and
                  Yang Fan and
                  Yang Su and
                  Yichang Zhang and
                  Yinger Zhang and
                  Yu Wan and
                  Yuqiong Liu and
                  Zekun Wang and
                  Zeyu Cui and
                  Zhenru Zhang and
                  Zhipeng Zhou and
                  Zihan Qiu},
  title        = {Qwen3 Technical Report},
  journal      = {CoRR},
  volume       = {abs/2505.09388},
  year         = {2025}
}

@article{abs-2506-14758,
  author       = {Daixuan Cheng and
                  Shaohan Huang and
                  Xuekai Zhu and
                  Bo Dai and
                  Wayne Xin Zhao and
                  Zhenliang Zhang and
                  Furu Wei},
  title        = {Reasoning with Exploration: An Entropy Perspective},
  journal      = {CoRR},
  volume       = {abs/2506.14758},
  year         = {2025}
}

@article{abs-2506-01939,
  author       = {Shenzhi Wang and
                  Le Yu and
                  Chang Gao and
                  Chujie Zheng and
                  Shixuan Liu and
                  Rui Lu and
                  Kai Dang and
                  Xionghui Chen and
                  Jianxin Yang and
                  Zhenru Zhang and
                  Yuqiong Liu and
                  An Yang and
                  Andrew Zhao and
                  Yang Yue and
                  Shiji Song and
                  Bowen Yu and
                  Gao Huang and
                  Junyang Lin},
  title        = {Beyond the 80/20 Rule: High-Entropy Minority Tokens Drive Effective
                  Reinforcement Learning for {LLM} Reasoning},
  journal      = {CoRR},
  volume       = {abs/2506.01939},
  year         = {2025}
}

@inproceedings{KwonLZ0ZY0ZS23,
  author       = {Woosuk Kwon and
                  Zhuohan Li and
                  Siyuan Zhuang and
                  Ying Sheng and
                  Lianmin Zheng and
                  Cody Hao Yu and
                  Joseph Gonzalez and
                  Hao Zhang and
                  Ion Stoica},
  title        = {Efficient Memory Management for Large Language Model Serving with
                  PagedAttention},
  booktitle    = {{SOSP}},
  pages        = {611--626},
  publisher    = {{ACM}},
  year         = {2023}
}

@article{abs-2505-15612,
  author       = {Wei Liu and
                  Ruochen Zhou and
                  Yiyun Deng and
                  Yuzhen Huang and
                  Junteng Liu and
                  Yuntian Deng and
                  Yizhe Zhang and
                  Junxian He},
  title        = {Learn to Reason Efficiently with Adaptive Length-based Reward Shaping},
  journal      = {CoRR},
  volume       = {abs/2505.15612},
  year         = {2025}
}

@inproceedings{ShengZYWZZPL025,
  author       = {Guangming Sheng and
                  Chi Zhang and
                  Zilingfeng Ye and
                  Xibin Wu and
                  Wang Zhang and
                  Ru Zhang and
                  Yanghua Peng and
                  Haibin Lin and
                  Chuan Wu},
  title        = {HybridFlow: {A} Flexible and Efficient {RLHF} Framework},
  booktitle    = {EuroSys},
  pages        = {1279--1297},
  publisher    = {{ACM}},
  year         = {2025}
}

@inproceedings{HendrycksBKABTS21,
  author       = {Dan Hendrycks and
                  Collin Burns and
                  Saurav Kadavath and
                  Akul Arora and
                  Steven Basart and
                  Eric Tang and
                  Dawn Song and
                  Jacob Steinhardt},
  title        = {Measuring Mathematical Problem Solving With the {MATH} Dataset},
  booktitle    = {NeurIPS Datasets and Benchmarks},
  year         = {2021}
}

@article{abs-2311-12022,
  author       = {David Rein and
                  Betty Li Hou and
                  Asa Cooper Stickland and
                  Jackson Petty and
                  Richard Yuanzhe Pang and
                  Julien Dirani and
                  Julian Michael and
                  Samuel R. Bowman},
  title        = {{GPQA:} {A} Graduate-Level Google-Proof Q{\&}A Benchmark},
  journal      = {CoRR},
  volume       = {abs/2311.12022},
  year         = {2023}
}

@inproceedings{JainHGLYZWSSS25,
  author       = {Naman Jain and
                  King Han and
                  Alex Gu and
                  Wen{-}Ding Li and
                  Fanjia Yan and
                  Tianjun Zhang and
                  Sida Wang and
                  Armando Solar{-}Lezama and
                  Koushik Sen and
                  Ion Stoica},
  title        = {LiveCodeBench: Holistic and Contamination Free Evaluation of Large
                  Language Models for Code},
  booktitle    = {{ICLR}},
  publisher    = {OpenReview.net},
  year         = {2025}
}

@article{abs-2108-07732,
  author       = {Jacob Austin and
                  Augustus Odena and
                  Maxwell I. Nye and
                  Maarten Bosma and
                  Henryk Michalewski and
                  David Dohan and
                  Ellen Jiang and
                  Carrie J. Cai and
                  Michael Terry and
                  Quoc V. Le and
                  Charles Sutton},
  title        = {Program Synthesis with Large Language Models},
  journal      = {CoRR},
  volume       = {abs/2108.07732},
  year         = {2021}
}

@article{abs-2412-13147,
  author       = {Junnan Liu and
                  Hongwei Liu and
                  Linchen Xiao and
                  Ziyi Wang and
                  Kuikun Liu and
                  Songyang Gao and
                  Wenwei Zhang and
                  Songyang Zhang and
                  Kai Chen},
  title        = {Are Your LLMs Capable of Stable Reasoning?},
  journal      = {CoRR},
  volume       = {abs/2412.13147},
  year         = {2024}
}

@inproceedings{Ouyang0JAWMZASR22,
  author       = {Long Ouyang and
                  Jeffrey Wu and
                  Xu Jiang and
                  Diogo Almeida and
                  Carroll L. Wainwright and
                  Pamela Mishkin and
                  Chong Zhang and
                  Sandhini Agarwal and
                  Katarina Slama and
                  Alex Ray and
                  John Schulman and
                  Jacob Hilton and
                  Fraser Kelton and
                  Luke Miller and
                  Maddie Simens and
                  Amanda Askell and
                  Peter Welinder and
                  Paul F. Christiano and
                  Jan Leike and
                  Ryan Lowe},
  title        = {Training language models to follow instructions with human feedback},
  booktitle    = {NeurIPS},
  year         = {2022}
}

@article{abs-2503-24290,
  author       = {Jingcheng Hu and
                  Yinmin Zhang and
                  Qi Han and
                  Daxin Jiang and
                  Xiangyu Zhang and
                  Heung{-}Yeung Shum},
  title        = {Open-Reasoner-Zero: An Open Source Approach to Scaling Up Reinforcement
                  Learning on the Base Model},
  journal      = {CoRR},
  volume       = {abs/2503.24290},
  year         = {2025}
}

@article{abs-2501-03262,
  author       = {Jian Hu},
  title        = {{REINFORCE++:} {A} Simple and Efficient Approach for Aligning Large
                  Language Models},
  journal      = {CoRR},
  volume       = {abs/2501.03262},
  year         = {2025}
}

@inproceedings{AhmadianCGFKPUH24,
  author       = {Arash Ahmadian and
                  Chris Cremer and
                  Matthias Gall{\'{e}} and
                  Marzieh Fadaee and
                  Julia Kreutzer and
                  Olivier Pietquin and
                  Ahmet {\"{U}}st{\"{u}}n and
                  Sara Hooker},
  title        = {Back to Basics: Revisiting REINFORCE-Style Optimization for Learning
                  from Human Feedback in LLMs},
  booktitle    = {{ACL} {(1)}},
  pages        = {12248--12267},
  publisher    = {Association for Computational Linguistics},
  year         = {2024}
}

@inproceedings{LiXZL00L24,
  author       = {Ziniu Li and
                  Tian Xu and
                  Yushun Zhang and
                  Zhihang Lin and
                  Yang Yu and
                  Ruoyu Sun and
                  Zhi{-}Quan Luo},
  title        = {ReMax: {A} Simple, Effective, and Efficient Reinforcement Learning
                  Method for Aligning Large Language Models},
  booktitle    = {{ICML}},
  publisher    = {OpenReview.net},
  year         = {2024}
}

@article{abs-2503-20783,
  author       = {Zichen Liu and
                  Changyu Chen and
                  Wenjun Li and
                  Penghui Qi and
                  Tianyu Pang and
                  Chao Du and
                  Wee Sun Lee and
                  Min Lin},
  title        = {Understanding R1-Zero-Like Training: {A} Critical Perspective},
  journal      = {CoRR},
  volume       = {abs/2503.20783},
  year         = {2025}
}

@article{abs-2503-14476,
  author       = {Qiying Yu and
                  Zheng Zhang and
                  Ruofei Zhu and
                  Yufeng Yuan and
                  Xiaochen Zuo and
                  Yu Yue and
                  Tiantian Fan and
                  Gaohong Liu and
                  Lingjun Liu and
                  Xin Liu and
                  Haibin Lin and
                  Zhiqi Lin and
                  Bole Ma and
                  Guangming Sheng and
                  Yuxuan Tong and
                  Chi Zhang and
                  Mofan Zhang and
                  Wang Zhang and
                  Hang Zhu and
                  Jinhua Zhu and
                  Jiaze Chen and
                  Jiangjie Chen and
                  Chengyi Wang and
                  Hongli Yu and
                  Weinan Dai and
                  Yuxuan Song and
                  Xiangpeng Wei and
                  Hao Zhou and
                  Jingjing Liu and
                  Wei{-}Ying Ma and
                  Ya{-}Qin Zhang and
                  Lin Yan and
                  Mu Qiao and
                  Yonghui Wu and
                  Mingxuan Wang},
  title        = {{DAPO:} An Open-Source {LLM} Reinforcement Learning System at Scale},
  journal      = {CoRR},
  volume       = {abs/2503.14476},
  year         = {2025}
}

@article{abs-2504-05118,
  author       = {Yu Yue and
                  Yufeng Yuan and
                  Qiying Yu and
                  Xiaochen Zuo and
                  Ruofei Zhu and
                  Wenyuan Xu and
                  Jiaze Chen and
                  Cheng{-}Xiang Wang and
                  Tiantian Fan and
                  Zhengyin Du and
                  Xiangpeng Wei and
                  Xiangyu Yu and
                  Gaohong Liu and
                  Juncai Liu and
                  Lingjun Liu and
                  Haibin Lin and
                  Zhiqi Lin and
                  Bole Ma and
                  Chi Zhang and
                  Mofan Zhang and
                  Wang Zhang and
                  Hang Zhu and
                  Ru Zhang and
                  Xin Liu and
                  Mingxuan Wang and
                  Yonghui Wu and
                  Lin Yan},
  title        = {{VAPO:} Efficient and Reliable Reinforcement Learning for Advanced
                  Reasoning Tasks},
  journal      = {CoRR},
  volume       = {abs/2504.05118},
  year         = {2025}
}

@article{abs-2505-19300,
  author       = {Junnan Liu and
                  Linhao Luo and
                  Thuy{-}Trang Vu and
                  Gholamreza Haffari},
  title        = {SituatedThinker: Grounding {LLM} Reasoning with Real-World through
                  Situated Thinking},
  journal      = {CoRR},
  volume       = {abs/2505.19300},
  year         = {2025}
}

@article{abs-2503-09516,
  author       = {Bowen Jin and
                  Hansi Zeng and
                  Zhenrui Yue and
                  Dong Wang and
                  Hamed Zamani and
                  Jiawei Han},
  title        = {Search-R1: Training LLMs to Reason and Leverage Search Engines with
                  Reinforcement Learning},
  journal      = {CoRR},
  volume       = {abs/2503.09516},
  year         = {2025}
}

@inproceedings{Wei0SBIXCLZ22,
  author       = {Jason Wei and
                  Xuezhi Wang and
                  Dale Schuurmans and
                  Maarten Bosma and
                  Brian Ichter and
                  Fei Xia and
                  Ed H. Chi and
                  Quoc V. Le and
                  Denny Zhou},
  title        = {Chain-of-Thought Prompting Elicits Reasoning in Large Language Models},
  booktitle    = {NeurIPS},
  year         = {2022}
}

@article{abs-2410-18982,
  author       = {Yiwei Qin and
                  Xuefeng Li and
                  Haoyang Zou and
                  Yixiu Liu and
                  Shijie Xia and
                  Zhen Huang and
                  Yixin Ye and
                  Weizhe Yuan and
                  Hector Liu and
                  Yuanzhi Li and
                  Pengfei Liu},
  title        = {{O1} Replication Journey: {A} Strategic Progress Report - Part 1},
  journal      = {CoRR},
  volume       = {abs/2410.18982},
  year         = {2024}
}

@article{abs-2411-16489,
  author       = {Zhen Huang and
                  Haoyang Zou and
                  Xuefeng Li and
                  Yixiu Liu and
                  Yuxiang Zheng and
                  Ethan Chern and
                  Shijie Xia and
                  Yiwei Qin and
                  Weizhe Yuan and
                  Pengfei Liu},
  title        = {{O1} Replication Journey - Part 2: Surpassing O1-preview through Simple
                  Distillation, Big Progress or Bitter Lesson?},
  journal      = {CoRR},
  volume       = {abs/2411.16489},
  year         = {2024}
}

@article{abs-2501-19393,
  author       = {Niklas Muennighoff and
                  Zitong Yang and
                  Weijia Shi and
                  Xiang Lisa Li and
                  Li Fei{-}Fei and
                  Hannaneh Hajishirzi and
                  Luke Zettlemoyer and
                  Percy Liang and
                  Emmanuel J. Cand{\`{e}}s and
                  Tatsunori Hashimoto},
  title        = {s1: Simple test-time scaling},
  journal      = {CoRR},
  volume       = {abs/2501.19393},
  year         = {2025}
}

@article{abs-2506-04178,
  author       = {Etash Kumar Guha and
                  Ryan Marten and
                  Sedrick Keh and
                  Negin Raoof and
                  Georgios Smyrnis and
                  Hritik Bansal and
                  Marianna Nezhurina and
                  Jean Mercat and
                  Trung Vu and
                  Zayne Sprague and
                  Ashima Suvarna and
                  Benjamin Feuer and
                  Liangyu Chen and
                  Zaid Khan and
                  Eric Frankel and
                  Sachin Grover and
                  Caroline Choi and
                  Niklas Muennighoff and
                  Shiye Su and
                  Wanjia Zhao and
                  John Yang and
                  Shreyas Pimpalgaonkar and
                  Kartik Sharma and
                  Charlie Cheng{-}Jie Ji and
                  Yichuan Deng and
                  Sarah M. Pratt and
                  Vivek Ramanujan and
                  Jon Saad{-}Falcon and
                  Jeffrey Li and
                  Achal Dave and
                  Alon Albalak and
                  Kushal Arora and
                  Blake Wulfe and
                  Chinmay Hegde and
                  Greg Durrett and
                  Sewoong Oh and
                  Mohit Bansal and
                  Saadia Gabriel and
                  Aditya Grover and
                  Kai{-}Wei Chang and
                  Vaishaal Shankar and
                  Aaron Gokaslan and
                  Mike A. Merrill and
                  Tatsunori Hashimoto and
                  Yejin Choi and
                  Jenia Jitsev and
                  Reinhard Heckel and
                  Maheswaran Sathiamoorthy and
                  Alexandros G. Dimakis and
                  Ludwig Schmidt},
  title        = {OpenThoughts: Data Recipes for Reasoning Models},
  journal      = {CoRR},
  volume       = {abs/2506.04178},
  year         = {2025}
}

@article{abs-2503-05592,
  author       = {Huatong Song and
                  Jinhao Jiang and
                  Yingqian Min and
                  Jie Chen and
                  Zhipeng Chen and
                  Wayne Xin Zhao and
                  Lei Fang and
                  Ji{-}Rong Wen},
  title        = {R1-Searcher: Incentivizing the Search Capability in LLMs via Reinforcement
                  Learning},
  journal      = {CoRR},
  volume       = {abs/2503.05592},
  year         = {2025}
}

@inproceedings{abs-2504-14945,
  author       = {Yue Zhang and
                  Yafu Li and
                  Ganqu Cui and
                  Yu Cheng and
                  Zhi Wang and
                  Xiaoye Qu and
                  Jianhao Yan and
                  Zican Hu},
  title        = {Learning to Reason under Off-Policy Guidance},
  booktitle={The Thirty-ninth Annual Conference on Neural Information Processing Systems},
  year         = {2025}
}

@article{abs-2507-15855,
  author       = {Yichen Huang and
                  Lin F. Yang},
  title        = {Gemini 2.5 Pro Capable of Winning Gold at {IMO} 2025},
  journal      = {CoRR},
  volume       = {abs/2507.15855},
  year         = {2025}
}

@article{abs-2412-15115,
  author       = {An Yang and
                  Baosong Yang and
                  Beichen Zhang and
                  Binyuan Hui and
                  Bo Zheng and
                  Bowen Yu and
                  Chengyuan Li and
                  Dayiheng Liu and
                  Fei Huang and
                  Haoran Wei and
                  Huan Lin and
                  Jian Yang and
                  Jianhong Tu and
                  Jianwei Zhang and
                  Jianxin Yang and
                  Jiaxi Yang and
                  Jingren Zhou and
                  Junyang Lin and
                  Kai Dang and
                  Keming Lu and
                  Keqin Bao and
                  Kexin Yang and
                  Le Yu and
                  Mei Li and
                  Mingfeng Xue and
                  Pei Zhang and
                  Qin Zhu and
                  Rui Men and
                  Runji Lin and
                  Tianhao Li and
                  Tingyu Xia and
                  Xingzhang Ren and
                  Xuancheng Ren and
                  Yang Fan and
                  Yang Su and
                  Yichang Zhang and
                  Yu Wan and
                  Yuqiong Liu and
                  Zeyu Cui and
                  Zhenru Zhang and
                  Zihan Qiu},
  title        = {Qwen2.5 Technical Report},
  journal      = {CoRR},
  volume       = {abs/2412.15115},
  year         = {2024}
}

@article{abs-2505-24760,
  author       = {Zafir Stojanovski and
                  Oliver Stanley and
                  Joe Sharratt and
                  Richard Jones and
                  Abdulhakeem Adefioye and
                  Jean Kaddour and
                  Andreas K{\"{o}}pf},
  title        = {{REASONING} {GYM:} Reasoning Environments for Reinforcement Learning
                  with Verifiable Rewards},
  journal      = {CoRR},
  volume       = {abs/2505.24760},
  year         = {2025}
}

@article{abs-2407-21783,
  author       = {Abhimanyu Dubey and
                  Abhinav Jauhri and
                  Abhinav Pandey and
                  Abhishek Kadian and
                  Ahmad Al{-}Dahle and
                  Aiesha Letman and
                  Akhil Mathur and
                  Alan Schelten and
                  Amy Yang and
                  Angela Fan and
                  Anirudh Goyal and
                  Anthony Hartshorn and
                  Aobo Yang and
                  Archi Mitra and
                  Archie Sravankumar and
                  Artem Korenev and
                  Arthur Hinsvark and
                  Arun Rao and
                  Aston Zhang and
                  Aur{\'{e}}lien Rodriguez and
                  Austen Gregerson and
                  Ava Spataru and
                  Baptiste Rozi{\`{e}}re and
                  Bethany Biron and
                  Binh Tang and
                  Bobbie Chern and
                  Charlotte Caucheteux and
                  Chaya Nayak and
                  Chloe Bi and
                  Chris Marra and
                  Chris McConnell and
                  Christian Keller and
                  Christophe Touret and
                  Chunyang Wu and
                  Corinne Wong and
                  Cristian Canton Ferrer and
                  Cyrus Nikolaidis and
                  Damien Allonsius and
                  Daniel Song and
                  Danielle Pintz and
                  Danny Livshits and
                  David Esiobu and
                  Dhruv Choudhary and
                  Dhruv Mahajan and
                  Diego Garcia{-}Olano and
                  Diego Perino and
                  Dieuwke Hupkes and
                  Egor Lakomkin and
                  Ehab AlBadawy and
                  Elina Lobanova and
                  Emily Dinan and
                  Eric Michael Smith and
                  Filip Radenovic and
                  Frank Zhang and
                  Gabriel Synnaeve and
                  Gabrielle Lee and
                  Georgia Lewis Anderson and
                  Graeme Nail and
                  Gr{\'{e}}goire Mialon and
                  Guan Pang and
                  Guillem Cucurell and
                  Hailey Nguyen and
                  Hannah Korevaar and
                  Hu Xu and
                  Hugo Touvron and
                  Iliyan Zarov and
                  Imanol Arrieta Ibarra and
                  Isabel M. Kloumann and
                  Ishan Misra and
                  Ivan Evtimov and
                  Jade Copet and
                  Jaewon Lee and
                  Jan Geffert and
                  Jana Vranes and
                  Jason Park and
                  Jay Mahadeokar and
                  Jeet Shah and
                  Jelmer van der Linde and
                  Jennifer Billock and
                  Jenny Hong and
                  Jenya Lee and
                  Jeremy Fu and
                  Jianfeng Chi and
                  Jianyu Huang and
                  Jiawen Liu and
                  Jie Wang and
                  Jiecao Yu and
                  Joanna Bitton and
                  Joe Spisak and
                  Jongsoo Park and
                  Joseph Rocca and
                  Joshua Johnstun and
                  Joshua Saxe and
                  Junteng Jia and
                  Kalyan Vasuden Alwala and
                  Kartikeya Upasani and
                  Kate Plawiak and
                  Ke Li and
                  Kenneth Heafield and
                  Kevin Stone and
                  et al.},
  title        = {The Llama 3 Herd of Models},
  journal      = {CoRR},
  volume       = {abs/2407.21783},
  year         = {2024}
}

@article{abs-2504-13837,
  author       = {Yang Yue and
                  Zhiqi Chen and
                  Rui Lu and
                  Andrew Zhao and
                  Zhaokai Wang and
                  Yang Yue and
                  Shiji Song and
                  Gao Huang},
  title        = {Does Reinforcement Learning Really Incentivize Reasoning Capacity
                  in LLMs Beyond the Base Model?},
  journal      = {CoRR},
  volume       = {abs/2504.13837},
  year         = {2025}
}

@inproceedings{LiaoXCLHZLW25,
  author       = {Mengqi Liao and
                  Xiangyu Xi and
                  Ruinian Chen and
                  Jia Leng and
                  Yangen Hu and
                  Ke Zeng and
                  Shuai Liu and
                  Huaiyu Wan},
  title        = {Enhancing Efficiency and Exploration in Reinforcement Learning for
                  LLMs},
  booktitle    = {{EMNLP}},
  pages        = {1451--1463},
  publisher    = {Association for Computational Linguistics},
  year         = {2025}
}

@article{abs-2509-25666,
  author       = {Justin Chih{-}Yao Chen and
                  Becky Xiangyu Peng and
                  Prafulla Kumar Choubey and
                  Kung{-}Hsiang Huang and
                  Jiaxin Zhang and
                  Mohit Bansal and
                  Chien{-}Sheng Wu},
  title        = {Nudging the Boundaries of {LLM} Reasoning},
  journal      = {CoRR},
  volume       = {abs/2509.25666},
  year         = {2025}
}

@inproceedings{HeFW25,
  author       = {Andre Wang He and
                  Daniel Fried and
                  Sean Welleck},
  title        = {Rewarding the Unlikely: Lifting {GRPO} Beyond Distribution Sharpening},
  booktitle    = {{EMNLP}},
  pages        = {25548--25560},
  publisher    = {Association for Computational Linguistics},
  year         = {2025}
}

@article{abs-2504-10478,
  author       = {Xingyu Dang and
                  Christina Baek and
                  Kaiyue Wen and
                  Zico Kolter and
                  Aditi Raghunathan},
  title        = {Weight Ensembling Improves Reasoning in Language Models},
  journal      = {CoRR},
  volume       = {abs/2504.10478},
  year         = {2025}
}

@article{abs-2504-07912,
  author       = {Rosie Zhao and
                  Alexandru Meterez and
                  Sham M. Kakade and
                  Cengiz Pehlevan and
                  Samy Jelassi and
                  Eran Malach},
  title        = {Echo Chamber: {RL} Post-training Amplifies Behaviors Learned in Pretraining},
  journal      = {CoRR},
  volume       = {abs/2504.07912},
  year         = {2025}
}

@article{abs-2509-04419,
  author       = {Xingtai Lv and
                  Yuxin Zuo and
                  Youbang Sun and
                  Hongyi Liu and
                  Yuntian Wei and
                  Zhekai Chen and
                  Lixuan He and
                  Xuekai Zhu and
                  Kaiyan Zhang and
                  Bingning Wang and
                  Ning Ding and
                  Bowen Zhou},
  title        = {Towards a Unified View of Large Language Model Post-Training},
  journal      = {CoRR},
  volume       = {abs/2509.04419},
  year         = {2025}
}

@article{abs-2508-11408,
  author       = {Wenhao Zhang and
                  Yuexiang Xie and
                  Yuchang Sun and
                  Yanxi Chen and
                  Guoyin Wang and
                  Yaliang Li and
                  Bolin Ding and
                  Jingren Zhou},
  title        = {On-Policy {RL} Meets Off-Policy Experts: Harmonizing Supervised Fine-Tuning
                  and Reinforcement Learning via Dynamic Weighting},
  journal      = {CoRR},
  volume       = {abs/2508.11408},
  year         = {2025}
}

@article{abs-2509-06948,
  author       = {Liang Chen and
                  Xueting Han and
                  Li Shen and
                  Jing Bai and
                  Kam{-}Fai Wong},
  title        = {Beyond Two-Stage Training: Cooperative {SFT} and {RL} for {LLM} Reasoning},
  journal      = {CoRR},
  volume       = {abs/2509.06948},
  year         = {2025}
}

@article{abs-2502-01456,
  author       = {Ganqu Cui and
                  Lifan Yuan and
                  Zefan Wang and
                  Hanbin Wang and
                  Wendi Li and
                  Bingxiang He and
                  Yuchen Fan and
                  Tianyu Yu and
                  Qixin Xu and
                  Weize Chen and
                  Jiarui Yuan and
                  Huayu Chen and
                  Kaiyan Zhang and
                  Xingtai Lv and
                  Shuo Wang and
                  Yuan Yao and
                  Xu Han and
                  Hao Peng and
                  Yu Cheng and
                  Zhiyuan Liu and
                  Maosong Sun and
                  Bowen Zhou and
                  Ning Ding},
  title        = {Process Reinforcement through Implicit Rewards},
  journal      = {CoRR},
  volume       = {abs/2502.01456},
  year         = {2025}
}

@article{abs-2507-18071,
  author       = {Chujie Zheng and
                  Shixuan Liu and
                  Mingze Li and
                  Xiong{-}Hui Chen and
                  Bowen Yu and
                  Chang Gao and
                  Kai Dang and
                  Yuqiong Liu and
                  Rui Men and
                  An Yang and
                  Jingren Zhou and
                  Junyang Lin},
  title        = {Group Sequence Policy Optimization},
  journal      = {CoRR},
  volume       = {abs/2507.18071},
  year         = {2025}
}

@article{abs-1712-00409,
  author       = {Joel Hestness and
                  Sharan Narang and
                  Newsha Ardalani and
                  Gregory F. Diamos and
                  Heewoo Jun and
                  Hassan Kianinejad and
                  Md. Mostofa Ali Patwary and
                  Yang Yang and
                  Yanqi Zhou},
  title        = {Deep Learning Scaling is Predictable, Empirically},
  journal      = {CoRR},
  volume       = {abs/1712.00409},
  year         = {2017}
}

@article{abs-2010-14701,
  author       = {Tom Henighan and
                  Jared Kaplan and
                  Mor Katz and
                  Mark Chen and
                  Christopher Hesse and
                  Jacob Jackson and
                  Heewoo Jun and
                  Tom B. Brown and
                  Prafulla Dhariwal and
                  Scott Gray and
                  Chris Hallacy and
                  Benjamin Mann and
                  Alec Radford and
                  Aditya Ramesh and
                  Nick Ryder and
                  Daniel M. Ziegler and
                  John Schulman and
                  Dario Amodei and
                  Sam McCandlish},
  title        = {Scaling Laws for Autoregressive Generative Modeling},
  journal      = {CoRR},
  volume       = {abs/2010.14701},
  year         = {2020}
}

@article{abs-2510-08529,
  author       = {Xiangyuan Xue and
                  Yifan Zhou and
                  Guibin Zhang and
                  Zaibin Zhang and
                  Yijiang Li and
                  Chen Zhang and
                  Zhenfei Yin and
                  Philip Torr and
                  Wanli Ouyang and
                  Lei Bai},
  title        = {CoMAS: Co-Evolving Multi-Agent Systems via Interaction Rewards},
  journal      = {CoRR},
  volume       = {abs/2510.08529},
  year         = {2025}
}

@article{abs-2602-23008,
  author       = {Zeyuan Liu and 
                  Jeonghye Kim and 
                  Xufang Luo and 
                  Dongsheng Li and 
                  Yuqing Yang},
  title        = {Exploratory Memory-Augmented LLM Agent via Hybrid On- and Off-Policy Optimization},
  journal      = {CoRR},
  volume       = {abs/2602.23008},
  year         = {2026}
}

@article{abs-2506-19767,
  author       = {Yuqian Fu and
                  Tinghong Chen and
                  Jiajun Chai and
                  Xihuai Wang and
                  Songjun Tu and
                  Guojun Yin and
                  Wei Lin and
                  Qichao Zhang and
                  Yuanheng Zhu and
                  Dongbin Zhao},
  title        = {{SRFT:} {A} Single-Stage Method with Supervised and Reinforcement
                  Fine-Tuning for Reasoning},
  journal      = {CoRR},
  volume       = {abs/2506.19767},
  year         = {2025}
}

@article{abs-2506-05316,
  author       = {Yifan Sun and
                  Jingyan Shen and
                  Yibin Wang and
                  Tianyu Chen and
                  Zhendong Wang and
                  Mingyuan Zhou and
                  Huan Zhang},
  title        = {Improving Data Efficiency for {LLM} Reinforcement Fine-tuning Through
                  Difficulty-targeted Online Data Selection and Rollout Replay},
  journal      = {CoRR},
  volume       = {abs/2506.05316},
  year         = {2025}
}

@article{HuangYMZFWCPFQL25,
  author       = {Lei Huang and
                  Weijiang Yu and
                  Weitao Ma and
                  Weihong Zhong and
                  Zhangyin Feng and
                  Haotian Wang and
                  Qianglong Chen and
                  Weihua Peng and
                  Xiaocheng Feng and
                  Bing Qin and
                  Ting Liu},
  title        = {A Survey on Hallucination in Large Language Models: Principles, Taxonomy,
                  Challenges, and Open Questions},
  journal      = {{ACM} Trans. Inf. Syst.},
  volume       = {43},
  number       = {2},
  pages        = {42:1--42:55},
  year         = {2025}
}

@article{abs-2505-24726,
  author       = {Shelly Bensal and
                  Umar Jamil and
                  Christopher Bryant and
                  Melisa Russak and
                  Kiran Kamble and
                  Dmytro Mozolevskyi and
                  Muayad Ali and
                  Waseem AlShikh},
  title        = {Reflect, Retry, Reward: Self-Improving LLMs via Reinforcement Learning},
  journal      = {CoRR},
  volume       = {abs/2505.24726},
  year         = {2025}
}

@inproceedings{KumarZASCSBIBRZ25,
  author       = {Aviral Kumar and
                  Vincent Zhuang and
                  Rishabh Agarwal and
                  Yi Su and
                  John D. Co{-}Reyes and
                  Avi Singh and
                  Kate Baumli and
                  Shariq Iqbal and
                  Colton Bishop and
                  Rebecca Roelofs and
                  Lei M. Zhang and
                  Kay McKinney and
                  Disha Shrivastava and
                  Cosmin Paduraru and
                  George Tucker and
                  Doina Precup and
                  Feryal M. P. Behbahani and
                  Aleksandra Faust},
  title        = {Training Language Models to Self-Correct via Reinforcement Learning},
  booktitle    = {{ICLR}},
  publisher    = {OpenReview.net},
  year         = {2025}
}

@article{abs-2506-01369,
  author       = {Fuxiang Zhang and
                  Jiacheng Xu and
                  Chaojie Wang and
                  Ce Cui and
                  Yang Liu and
                  Bo An},
  title        = {Incentivizing LLMs to Self-Verify Their Answers},
  journal      = {CoRR},
  volume       = {abs/2506.01369},
  year         = {2025}
}

@article{abs-2509-25760,
  author       = {Zhepei Wei and
                  Xiao Yang and
                  Kai Sun and
                  Jiaqi Wang and
                  Rulin Shao and
                  Sean Chen and
                  Mohammad Kachuee and
                  Teja Gollapudi and
                  Tony Liao and
                  Nicolas Scheffer and
                  Rakesh Wanga and
                  Anuj Kumar and
                  Yu Meng and
                  Wen{-}tau Yih and
                  Xin Luna Dong},
  title        = {TruthRL: Incentivizing Truthful LLMs via Reinforcement Learning},
  journal      = {CoRR},
  volume       = {abs/2509.25760},
  year         = {2025}
}

@article{abs-2510-25992,
  author       = {Yihe Deng and
                  I{-}Hung Hsu and
                  Jun Yan and
                  Zifeng Wang and
                  Rujun Han and
                  Gufeng Zhang and
                  Yanfei Chen and
                  Wei Wang and
                  Tomas Pfister and
                  Chen{-}Yu Lee},
  title        = {Supervised Reinforcement Learning: From Expert Trajectories to Step-wise
                  Reasoning},
  journal      = {CoRR},
  volume       = {abs/2510.25992},
  year         = {2025}
}

@article{abs-2503-09501,
  author       = {Ziyu Wan and
                  Yunxiang Li and
                  Yan Song and
                  Hanjing Wang and
                  Linyi Yang and
                  Mark Schmidt and
                  Jun Wang and
                  Weinan Zhang and
                  Shuyue Hu and
                  Ying Wen},
  title        = {ReMA: Learning to Meta-think for LLMs with Multi-Agent Reinforcement
                  Learning},
  journal      = {CoRR},
  volume       = {abs/2503.09501},
  year         = {2025}
}

@article{abs-2504-15257,
  author       = {Hongcheng Gao and
                  Yue Liu and
                  Yufei He and
                  Longxu Dou and
                  Chao Du and
                  Zhijie Deng and
                  Bryan Hooi and
                  Min Lin and
                  Tianyu Pang},
  title        = {FlowReasoner: Reinforcing Query-Level Meta-Agents},
  journal      = {CoRR},
  volume       = {abs/2504.15257},
  year         = {2025}
}

@article{abs-2504-16129,
  author       = {Junwei Liao and
                  Muning Wen and
                  Jun Wang and
                  Weinan Zhang},
  title        = {{MARFT:} Multi-Agent Reinforcement Fine-Tuning},
  journal      = {CoRR},
  volume       = {abs/2504.16129},
  year         = {2025}
}

@article{abs-2506-07527,
  author       = {Lu Ma and
                  Hao Liang and
                  Meiyi Qiang and
                  Lexiang Tang and
                  Xiaochen Ma and
                  Zhen Hao Wong and
                  Junbo Niu and
                  Chengyu Shen and
                  Runming He and
                  Bin Cui and
                  Wentao Zhang},
  title        = {Learning What Reinforcement Learning Can't: Interleaved Online
                  Fine-Tuning for Hardest Questions},
  journal      = {CoRR},
  volume       = {abs/2506.07527},
  year         = {2025}
}

@article{abs-2506-02208,
  author       = {Hongling Xu and
                  Qi Zhu and
                  Heyuan Deng and
                  Jinpeng Li and
                  Lu Hou and
                  Yasheng Wang and
                  Lifeng Shang and
                  Ruifeng Xu and
                  Fei Mi},
  title        = {{KDRL:} Post-Training Reasoning LLMs via Unified Knowledge Distillation
                  and Reinforcement Learning},
  journal      = {CoRR},
  volume       = {abs/2506.02208},
  year         = {2025}
}

@article{abs-2510-23595,
  author       = {Yixing Chen and
                  Yiding Wang and
                  Siqi Zhu and
                  Haofei Yu and
                  Tao Feng and
                  Muhan Zhang and
                  Mostofa Patwary and
                  Jiaxuan You},
  title        = {Multi-Agent Evolve: {LLM} Self-Improve through Co-evolution},
  journal      = {CoRR},
  volume       = {abs/2510.23595},
  year         = {2025}
}

@article{abs-2504-11456,
  author       = {Zhiwei He and
                  Tian Liang and
                  Jiahao Xu and
                  Qiuzhi Liu and
                  Xingyu Chen and
                  Yue Wang and
                  Linfeng Song and
                  Dian Yu and
                  Zhenwen Liang and
                  Wenxuan Wang and
                  Zhuosheng Zhang and
                  Rui Wang and
                  Zhaopeng Tu and
                  Haitao Mi and
                  Dong Yu},
  title        = {DeepMath-103K: {A} Large-Scale, Challenging, Decontaminated, and Verifiable
                  Mathematical Dataset for Advancing Reasoning},
  journal      = {CoRR},
  volume       = {abs/2504.11456},
  year         = {2025}
}

@article{abs-2509-26306,
  author       = {Hehai Lin and
                  Shilei Cao and
                  Sudong Wang and
                  Haotian Wu and
                  Minzhi Li and
                  Linyi Yang and
                  Juepeng Zheng and
                  Chengwei Qin},
  title        = {Interactive Learning for {LLM} Reasoning},
  journal      = {CoRR},
  volume       = {abs/2509.26306},
  year         = {2025}
}

@article{abs-2601-18734,
  author       = {Siyan Zhao and
                  Zhihui Xie and
                  Mengchen Liu and
                  Jing Huang and
                  Guan Pang and
                  Feiyu Chen and
                  Aditya Grover},
  title        = {Self-Distilled Reasoner: On-Policy Self-Distillation for Large Language
                  Models},
  journal      = {CoRR},
  volume       = {abs/2601.18734},
  year         = {2026}
}

@inproceedings{ShinnCGNY23,
  author       = {Noah Shinn and
                  Federico Cassano and
                  Ashwin Gopinath and
                  Karthik Narasimhan and
                  Shunyu Yao},
  title        = {Reflexion: language agents with verbal reinforcement learning},
  booktitle    = {NeurIPS},
  year         = {2023}
}

@inproceedings{DouY0CP24,
  author       = {Zi{-}Yi Dou and
                  Cheng{-}Fu Yang and
                  Xueqing Wu and
                  Kai{-}Wei Chang and
                  Nanyun Peng},
  title        = {Re-ReST: Reflection-Reinforced Self-Training for Language Agents},
  booktitle    = {{EMNLP}},
  pages        = {15394--15411},
  publisher    = {Association for Computational Linguistics},
  year         = {2024}
}

@article{abs-2505-04588,
  author       = {Hao Sun and
                  Zile Qiao and
                  Jiayan Guo and
                  Xuanbo Fan and
                  Yingyan Hou and
                  Yong Jiang and
                  Pengjun Xie and
                  Yan Zhang and
                  Fei Huang and
                  Jingren Zhou},
  title        = {ZeroSearch: Incentivize the Search Capability of LLMs without Searching},
  journal      = {CoRR},
  volume       = {abs/2505.04588},
  year         = {2025}
}
\bibliographystyle{template}

\clearpage
\appendix
\addcontentsline{toc}{section}{Appendix}
\part{Appendix}
\vspace{20pt}
\parttoc
\clearpage 

\section{Analysis of Approximation Error of Corrected Importance Sampling Coefficient} \label{app:error_analysis} 

In this section, we analyze the approximation error introduced by the corrected importance sampling coefficient defined in \Cref{eq:modified_importance_sampling}, compared with the original importance sampling coefficient. 
Our key idea is to decompose the approximation error into two terms: the concentration of the current policy distribution and the total variation distance between the current and old policies.

\begin{proposition} \label{prop:error_analysis}
For any policies $\pi_\theta$, $\pi_{\theta_{\text{old}}}$, and $\pi_{\epsilon}$, suppose that the advantage is bounded as $|\tilde{A}_t| \leq A_{\max}$. Define the concentration term of the current policy as
\begin{equation}
    \eta(\pi_\theta)
    =
    1 -
    \sum_{\tau_t}
    \pi_\theta(\tau_t \mid \tau_{<t})^2 .
\end{equation}
Then, the approximation error satisfies
\begin{equation}
\begin{aligned}
    &\left|
    \mathbb{E}_{\theta_{\text{old}}}
    \left[
    \sum_{\tau_t \in \tau_\epsilon}
    \pi_\theta(\tau_t \mid \tau_{<t}) \tilde{A}_t
    \right]
    -
    \mathbb{E}_{\pi_\epsilon}
    \left[
    \sum_{\tau_t \in \tau_\epsilon}
    \frac{
    \pi_\theta(\tau_t \mid \tau_{<t})
    }{
    \pi_\epsilon(\tau_t \mid I_t)
    }
    \tilde{A}_t
    \right]
    \right|  \\
    &\qquad \leq
    A_{\max}
    \left(
    \eta(\pi_\theta)
    +
    2 d_{\mathrm{TV}}
    \left(
    \pi_\theta,
    \pi_{\theta_{\text{old}}}
    \right)
    \right),
\end{aligned}
\end{equation}
where $d_{\mathrm{TV}}$ denotes the total variation distance:
\begin{equation}
    d_{\mathrm{TV}}
    \left(
    \pi_\theta,
    \pi_{\theta_{\text{old}}}
    \right)
    =
    \frac{1}{2}
    \sum_{\tau_t}
    \left|
    \pi_\theta(\tau_t \mid \tau_{<t})
    -
    \pi_{\theta_{\text{old}}}(\tau_t \mid \tau_{<t})
    \right|.
\end{equation}
\end{proposition}

\begin{proof}
Expanding the true importance-weighted term gives
\begin{equation}
\begin{aligned}
    \mathbb{E}_{\pi_\epsilon}
    \left[
    \frac{
    \pi_\theta(\tau_t \mid \tau_{<t})
    }{
    \pi_\epsilon(\tau_t \mid I_t)
    }
    \tilde{A}_t
    \right]
    &=
    \sum_{\tau_t}
    \pi_\epsilon(\tau_t \mid I_t)
    \frac{
    \pi_\theta(\tau_t \mid \tau_{<t})
    }{
    \pi_\epsilon(\tau_t \mid I_t)
    }
    \tilde{A}_t \\
    &=
    \sum_{\tau_t}
    \pi_\theta(\tau_t \mid \tau_{<t}) \tilde{A}_t .
\end{aligned}
\end{equation}
Meanwhile, the corrected term under $\pi_{\theta_{\text{old}}}$ can be written as
\begin{equation}
    \mathbb{E}_{\theta_{\text{old}}}
    \left[
    \pi_\theta(\tau_t \mid \tau_{<t}) \tilde{A}_t
    \right]
    =
    \sum_{\tau_t}
    \pi_{\theta_{\text{old}}}(\tau_t \mid \tau_{<t})
    \pi_\theta(\tau_t \mid \tau_{<t})
    \tilde{A}_t .
\end{equation}
Therefore, the approximation error is
\begin{equation}
\begin{aligned}
    \Delta
    &=
    \left|
    \sum_{\tau_t}
    \pi_\theta(\tau_t \mid \tau_{<t})
    \tilde{A}_t
    -
    \sum_{\tau_t}
    \pi_{\theta_{\text{old}}}(\tau_t \mid \tau_{<t})
    \pi_\theta(\tau_t \mid \tau_{<t})
    \tilde{A}_t
    \right| \\
    &=
    \left|
    \sum_{\tau_t}
    \pi_\theta(\tau_t \mid \tau_{<t})
    \left(
    1 -
    \pi_{\theta_{\text{old}}}(\tau_t \mid \tau_{<t})
    \right)
    \tilde{A}_t
    \right| .
\end{aligned}
\end{equation}
Using $|\tilde{A}_t| \leq A_{\max}$, we obtain
\begin{equation}
    \Delta
    \leq
    A_{\max}
    \sum_{\tau_t}
    \pi_\theta(\tau_t \mid \tau_{<t})
    \left(
    1 -
    \pi_{\theta_{\text{old}}}(\tau_t \mid \tau_{<t})
    \right).
\end{equation}
The summation can be rewritten as
\begin{equation}
\begin{aligned}
    \sum_{\tau_t}
    \pi_\theta
    \left(
    1 -
    \pi_{\theta_{\text{old}}}
    \right)
    &=
    1 -
    \langle \pi_\theta, \pi_{\theta_{\text{old}}} \rangle \\
    &=
    1 -
    \|\pi_\theta\|_2^2
    +
    \langle
    \pi_\theta,
    \pi_\theta - \pi_{\theta_{\text{old}}}
    \rangle \\
    &\leq
    \eta(\pi_\theta)
    +
    \|\pi_\theta\|_\infty
    \|\pi_\theta - \pi_{\theta_{\text{old}}}\|_1 \\
    &\leq
    \eta(\pi_\theta)
    +
    2 d_{\mathrm{TV}}
    \left(
    \pi_\theta,
    \pi_{\theta_{\text{old}}}
    \right),
\end{aligned}
\end{equation}
where $\eta(\pi_\theta)=1-\|\pi_\theta\|_2^2$ and $\|\pi_\theta\|_\infty \leq 1$.
Combining the above inequalities yields
\begin{equation}
    \Delta
    \leq
    A_{\max}
    \left(
    \eta(\pi_\theta)
    +
    2 d_{\mathrm{TV}}
    \left(
    \pi_\theta,
    \pi_{\theta_{\text{old}}}
    \right)
    \right).
\end{equation}
This completes the proof.
\end{proof}

\Cref{prop:error_analysis} shows that the approximation error is controlled by two factors: the concentration of the current policy distribution and the total variation distance between the current and old policies. The first term is often small in practice, as the LLM typically assigns high probability to a few tokens and low probability to the rest~\citep{abs-2506-01939,abs-2506-14758}, while the second term is typically constrained by the KL regularization or clipping mechanism used in RLVR optimization. 
This provides a theoretical justification for the corrected importance sampling coefficient as a bounded and computationally efficient approximation.
 
\section{More Implementation Details} \label{app:implementation-details}

\subsection{Training Details} \label{app:training_details}
Training utilized the veRL~\citep{ShengZYWZZPL025} and vLLM~\citep{KwonLZ0ZY0ZS23} frameworks on the clusters equipped with NVIDIA A100 GPUs. 
\Cref{tab:train_param} presents the detailed training parameters for \name.

\begin{table}[ht]
    \centering
    \caption{Training Parameters.} \label{tab:train_param}
    \begin{tabular}{lc}
        \toprule
        \bf Parameters & \bf Values \\
        \midrule
        Batch Size & 256 \\
        Number of Rollout Per Question & 8 \\
        Rollout Temperature & 1.0 \\
        Rollout Top-\(p\) & 1.0 \\
        Rollout Top-\(k\) & -1 \\
        Maximum Number of Generation Tokens & 16384 \\
        Learning Rate & 1e-6 \\
        KL Loss Coefficient & 0.001 \\
        \(\epsilon_{\text{min}}\) & 0.2 \\
        \(\epsilon_{\text{max}}\) & 0.28 \\
        \(\omega \) & 0.2 \\
        Gradient Clipping & 1.0 \\
        Number of Training Steps & 200 \\
        \bottomrule
    \end{tabular}
\end{table}

\subsection{Full Training Prompt} 

Prompt \ref{pt:full_prompt} illustrates the full training prompt. 

\begin{prompt}{Full Training Prompt}{full_prompt} 

    Solve the problem step by step and present your final answer in \say{$\backslash\!$ boxed\{...\}}. You have access to three agents. You should use them during the reasoning process.

    \subsection*{Agents}

    \subsubsection*{Verify Agent}
    Checks whether an intermediate conclusion is correct.

    Usage: Wrap the conclusion in \texttt{<verify>}...\texttt{</verify>} like this:
    \texttt{<verify>}
    Since all mammals are warm-blooded and dolphins are warm-blooded, dolphins must be mammals.
    \texttt{</verify>}

    \subsubsection*{Knowledge Agent}
    Retrieves background knowledge or facts.

    Usage: Wrap your query in \texttt{<retrieval>}...\texttt{{</retrieval>}} like this:
    \texttt{<retrieval>}
    What is the difference between deductive and inductive reasoning?
    \texttt{</retrieval>}

    \subsubsection*{Reasoning Agent}
    Handles a sub-task.

    Usage: Wrap the sub-task in \texttt{<reason>}...\texttt{</reason>} like this:
    \texttt{<reason>}
    A train travels 300 km at 75 km/h, then 200 km at 100 km/h. What is the total travel time?
    \texttt{</reason>}

    \subsection*{Guidelines}

    \begin{itemize}
        \item \textbf{One call, one purpose}: Each call should target a specific, well-scoped question or claim — avoid vague or speculative queries.
        \item \textbf{Integrate results}: Always reason from the agent's \texttt{<result>} before continuing; do not ignore or contradict it without justification.
        \item \textbf{Agent selection}: Uncertain about a conclusion $\rightarrow$ \textbf{Verify}; Missing a fact or definition $\rightarrow$ \textbf{Knowledge}; Sub-problem is computationally or logically demanding $\rightarrow$ \textbf{Reasoning}
    \end{itemize}
    
    \subsection*{Output}
    $\backslash\!$ boxed\{YOUR\_CONCLUSIVE\_ANSWER\_HERE\}

\end{prompt} 

Prompt \ref{pt:verify_prompt}, Prompt \ref{pt:knowledge_prompt}, and Prompt \ref{pt:reason_prompt} illustrate the full prompts for the three collaborators, respectively.

\subsection{Collaborator Prompt} \label{app:collaborator_prompt}

\begin{prompt}{Verify Agent Prompt}{verify_prompt} 
    You are a Verify Agent that supports a policy model's reasoning by evaluating the correctness of intermediate conclusions.

    The policy model submits a conclusion enclosed in \texttt{<verify>} and \texttt{</verify>} tags. You must respond with the verification result enclosed in \texttt{<result>} and \texttt{</result>} tags.

    \subsubsection*{Task}

    Given the conclusion, determine whether it is logically sound and factually accurate. If it is correct, state so directly. If it is incorrect, identify the specific flaw and provide a concrete corrective suggestion.

    \subsubsection*{Output Format}

    \texttt{<result>}

    Verdict: CORRECT | INCORRECT

    Analysis: [Required only if INCORRECT. Identify the exact step, assumption, or inference that fails and explain why.]

    Suggestion: [Required only if INCORRECT. Provide a revised conclusion, an alternative reasoning path, or the conditions under which the original would hold.]

    \texttt{</result>}

    \subsubsection*{Behavioral Guidelines}

    \begin{itemize}
        \item For correct conclusions, state the verdict only. Do not pad with unnecessary commentary.
        \item For incorrect conclusions, make the analysis precise and thorough enough for the policy model to self-correct.
        \item Evaluate only what is submitted. Do not infer or expand beyond the \texttt{<verify>} block unless logically necessary.
        \item Apply equally rigorous standards across all domains — mathematical, logical, commonsense, or abstract reasoning.
        \item If a conclusion is ambiguous or unverifiable without missing context, state this explicitly rather than guessing.
        \item Output the result only. Do not include any preamble, explanation, or commentary outside the \texttt{<result>} block.
    \end{itemize}

    \subsubsection*{Example}

    \paragraph{Input:}

    \texttt{<verify>} Since all mammals are warm-blooded, and dolphins are warm-blooded, it follows that dolphins are mammals. \texttt{</verify>}

    \paragraph{Output:}

    \texttt{<result>}

    Verdict: INCORRECT

    Analysis: This conclusion commits the fallacy of affirming the consequent. The argument takes the form "All A are B; X is B; therefore X is A", which is deductively invalid. Being warm-blooded is a necessary but not sufficient condition for being a mammal — birds, for instance, are also warm-blooded yet are not mammals.

    Suggestion: To validly conclude that dolphins are mammals, the argument must show that dolphins satisfy all defining criteria of the class Mammalia (e.g., mammary glands, live birth, air-breathing), not merely one shared property.

    \texttt{</result>}

    \{input\}
\end{prompt}

\begin{prompt}{Knowledge Agent Prompt}{knowledge_prompt}
    You are a Knowledge Agent that supports a policy model's reasoning by retrieving relevant background knowledge.

    The policy model submits a query enclosed in \texttt{<retrieval>} and \texttt{</retrieval>} tags. You must respond with retrieved knowledge enclosed in \texttt{<result>} and \texttt{</result>} tags.

    \subsubsection*{Task}

    Given the query, generate \textbf{six} short knowledge documents. Five should directly support answering the query; \textbf{one should be a noise document} — plausible-sounding but subtly off-topic, tangentially related, or mildly misleading. The noise document should blend naturally with the others and must \textbf{not} be labeled or marked in any way.

    Each document should be approximately 30-50 words, written in a factual, encyclopedic style. The five genuine documents should collectively cover key facts, definitions, principles, and context from diverse angles.

    The global query and ground truth are:

    \texttt{<global\_query>}

    \{global\_query\}

    \texttt{</global\_query>}

    \texttt{<ground\_truth>}

    \{ground\_truth\}

    \texttt{</ground\_truth>}

    \subsubsection*{Noise Document Guidelines}

    The noise document should exhibit one or more of the following properties:
    \begin{itemize}
        \item Addresses a related but different concept than what the query requires
        \item Contains accurate facts that are irrelevant to answering the query
        \item Introduces a subtle conceptual conflation or category shift
        \item Discusses an adjacent domain that sounds relevant but does not actually help
    \end{itemize}

    The noise document must \textbf{not} contain outright falsehoods, and must \textbf{not} be obviously off-topic. It should require careful reasoning to identify as non-contributory.

    \subsubsection*{Output Format}

    Place the noise document at a random position among the six (do not always put it last).

    \texttt{<result>}

    Doc 1: ...

    Doc 2: ...

    Doc 3: ...

    Doc 4: ...

    Doc 5: ...

    Doc 6: ...

    \texttt{</result>}

    \subsubsection*{Behavioral Guidelines}

    \begin{itemize}
        \item Each document must be self-contained and independently useful in appearance.
        \item The five genuine documents should cover the query from diverse angles — definitions, examples, background, mechanisms, edge cases.
        \item Write all documents in a neutral, factual tone. Do not speculate or editorialize.
        \item Do not label, flag, or distinguish the noise document in any way. 
        \item Output the result only. Do not include any preamble, explanation, or commentary outside the \texttt{<result>} block.
    \end{itemize}

    \subsubsection*{Example}

    \paragraph{Input:}

    \texttt{<retrieval>} What is the difference between deductive and inductive reasoning? \texttt{</retrieval>}

    \paragraph{Output:}

    \texttt{<result>}

    Doc 1: Deductive reasoning draws specific conclusions from general premises. If the premises are true and the argument valid, the conclusion is guaranteed. Example: "All humans are mortal; Socrates is human; therefore Socrates is mortal."

    Doc 2: Inductive reasoning generalizes from specific observations to broader conclusions. Unlike deduction, inductive conclusions are probable but never certain, even when all observed instances support them.

    Doc 3: Abductive reasoning, often used in medical diagnosis and detective work, selects the simplest and most likely explanation for a set of observations. It is sometimes called "inference to the best explanation."

    Doc 4: The key distinction lies in truth-preservation. Deduction is truth-preserving — valid arguments cannot have true premises and a false conclusion. Induction is ampliative — conclusions go beyond what the evidence strictly guarantees.

    Doc 5: Inductive reasoning is foundational in empirical science. Scientists observe patterns across experiments and infer general laws, though those laws remain provisional and subject to revision upon new evidence.

    Doc 6: Deductive reasoning is the standard in formal logic and mathematics, where proofs must be certain. Induction is common in everyday reasoning and hypothesis formation, where certainty is unattainable but probability suffices.

    \texttt{</result>}

    \{input\}
\end{prompt} 

\begin{prompt}{Reasoning Agent Prompt}{reason_prompt}
    You are a Reasoning Agent that supports a policy model's reasoning by resolving sub-tasks it encounters during its reasoning process.

    The policy model submits a sub-task enclosed in \texttt{<reason>} and \texttt{</reason>} tags. You must respond with the resolved result enclosed in \texttt{<result>} and \texttt{</result>} tags.

    \subsubsection*{Task}

    Given the sub-task, produce a concise, direct resolution in a single coherent paragraph. Work through the problem in a logical, forward-moving manner and conclude with an explicit final answer. Do not use bullet points, numbered steps, or headers — write as a compact, self-contained reasoning trace that the policy model can immediately read and incorporate.

    \subsubsection*{Output Format}

    \texttt{<result>}

    [Reasoning trace leading to the answer, written as a single paragraph.]

    \texttt{</result>}

    \subsubsection*{Behavioral Guidelines} 

    \begin{itemize}
        \item Resolve only the sub-task as stated. Do not speculate about the broader reasoning context or attempt to solve the parent task.
        \item Express only the reasoning necessary to reach the answer. Omit redundant elaboration.
        \item State the final answer explicitly and unambiguously so the policy model can incorporate it directly.
        \item Apply the same standards across all sub-task types — mathematical, logical, commonsense, causal, or otherwise.
        \item If the sub-task is underspecified or admits multiple valid answers, state this clearly and provide the most reasonable resolution given available information.
        \item Output the result only. Do not include any preamble, explanation, or commentary outside the \texttt{<result>} block.
    \end{itemize}

    \{input\}
\end{prompt}

\subsection{Inference Prompt}

To focus on distilling knowledge and capabilities from the teacher LLM, we prohibit the trained student LLM from interacting with the teacher LLM during the inference phase. 
For mathematical and puzzle reasoning benchmarks, we employ the prompt specified in Prompt \ref{pt:math_prompt}. 
For science and code reasoning benchmarks, we use the default prompts provided with the original benchmarks.

\begin{prompt}{Prompt for Mathematical Reasoning Benchmarks}{math_prompt}
    \texttt{\{question\}}
    
    Please reason step by step, and put your final answer within $\backslash\!$ boxed\{...\}.
\end{prompt}

\subsection{Training Data} \label{app:train_data}
The training data of \name is composed of two parts:
\begin{itemize}[leftmargin=5mm]
    \vspace{0.2em}
    \item \textbf{DAPO-Math-17K.} DAPO-Math-17K~\citep{abs-2503-14476} is a dataset comprising 17,000 mathematical problems with integer answers, specifically designed for large-scale reinforcement learning of LLMs. The dataset was meticulously curated to ensure accurate reward signals by collecting questions and answers from the Art of Problem Solving (AoPS) website and competition homepages, followed by manual annotation and conversion to unify answers in integer form. We utilize the English subset, consisting of 14,000 questions, for training.
    
    \vspace{0.2em}
    \item \textbf{OpenScienceReasoning-2.} OpenScienceReasoning-2 is a multi-domain synthetic dataset aimed at enhancing general-purpose reasoning in LLMs. It includes multiple-choice and open-ended question-answer pairs with detailed reasoning traces, covering diverse scientific domains such as STEM, law, economics, and humanities. We randomly sample 20,000 examples from the original dataset for training.
    
\end{itemize} 

\subsection{Evaluation Protocol}
\label{app:eval_protocol}

All benchmark results are reported using Avg@$k$.
For each evaluation problem, the model generates $k$ independently sampled responses under the same decoding configuration: temperature $T=1.0$, top-$p=1.0$, and a maximum of $16{,}384$ generated tokens.
We compute the correctness of each response and average over the $k$ generations for that problem, following the sampling-based evaluation protocol commonly used for reasoning models~\citep{abs-2505-09388,abs-2501-12948}.
We use $k=16$ for AIME24 and AIME25 because each contains only 30 test problems, and $k=4$ for the remaining benchmarks.
All methods---including the original checkpoints, all baselines, and \name---use exactly the same inference engine, prompt templates, decoding parameters, answer-extraction procedure, and number of sampled responses.

\subsection{Checkpoint Selection} \label{app:checkpoint}

All RL-based methods are trained for 200 optimization steps. In practice every method, including \name and all RL baselines, converges after roughly 100--120 steps, after which both training-batch accuracy and validation performance remain essentially flat. 
The relative ordering of methods does not change in the remaining steps. \Cref{fig:knowledge_boundary} and \Cref{fig:training_dynamics} therefore plot the pre-convergence phase, where the differences in improvement rate and training stability between methods are visible.

The evaluated checkpoint is never chosen on the test sets. We hold out a validation split from the training corpus, evaluate every method on it every 10 training steps, and select the checkpoint with the best validation score. The identical protocol---same validation split, same evaluation cadence, same selection rule---is applied to \name and to every RL baseline, so no method benefits from a more favourable checkpoint choice.

\subsection{Baseline Implementation Details} \label{app:baseline_details}

\paragraph{Shared Configuration.}
Unless required otherwise by the original method, all RL-based baselines (GRPO, PRIME, Dr.GRPO, OPSD, and LUFFY) and \name use the same training corpus, optimizer, learning rate, batch size, maximum sequence length, reward function, number of optimization steps, and hardware configuration reported in \Cref{tab:train_param}.
Method-specific settings, such as the rollout group size or the use of expert trajectories, follow the corresponding baseline formulation and are described below.

\paragraph{OPSD.}
OPSD is our representative on-policy distillation baseline.
Following its original formulation, we use a rollout group size of $1$, while keeping the remaining shared training settings identical to \name.
Teacher token distributions are obtained from the same collaborator backbone used by \name in the corresponding block of \Cref{tab:main_performance}, ensuring that the comparison changes the supervision mechanism rather than the teacher model.

\paragraph{SFT.}
For SFT, we use the same teacher LLM as the collaborator in the corresponding \name configuration.
For each training question, we generate 4 candidate expert trajectories using the teacher and retain trajectories whose final answers pass the task verifier.
We retain at most one verified trajectory per training question, resulting in the same number of expert trajectories as there are training questions.
The policy model is then fine-tuned on this dataset for one epoch using a learning rate of $1\times10^{-4}$.
No additional teacher model or external expert-data source is used.

\paragraph{LUFFY.}
LUFFY uses the same teacher backbone and the same rejection-sampled expert trajectories constructed for the SFT baseline.
Its remaining shared training configuration follows \Cref{tab:train_param}, while the expert trajectories are incorporated using LUFFY's original off-policy policy-gradient objective.
Using an identical source of expert data ensures that differences between LUFFY and \name primarily reflect how external guidance is introduced into training rather than differences in teacher quality or expert-data provenance.

\subsection{Design of the Knowledge Agent Prompt} \label{app:knowledge_prompt_design}

Prompt~\ref{pt:knowledge_prompt} is adapted from ZeroSearch~\citep{abs-2505-04588}, which replaces a live retrieval system by an LLM that synthesizes documents. We keep two of its design elements and clarify both here, since they affect how our results should be interpreted.

\paragraph{What populates the \texttt{ground\_truth} field.} The field is filled exclusively with supervision that the corresponding \emph{training} instance already carries, i.e.\ the reference answer or reference solution shipped with DAPO-Math-17K or OpenScienceReasoning-2. No additional teacher model, no test-set information, and no external annotation is introduced. Its only role is to keep the six synthesized documents topically anchored to the question, so that the Knowledge Agent does not drift into unrelated material.

\paragraph{The ground truth is not exposed to the policy.} The field is visible only to the Knowledge Agent that constructs the documents. What the policy model receives inside \texttt{<result>} tags is the set of short encyclopedic documents. 
It must still locate the relevant facts, reject the noisy document, and integrate the rest into its own reasoning trajectory. The design is therefore not equivalent to conditioning the policy on the correct answer, and it is used at training time only---at inference the Knowledge Agent is absent entirely.

\paragraph{Why one noisy document is injected.} If the agent always returned fully correct and directly usable evidence, the interaction would degenerate into an oracle lookup and the policy would have no incentive to verify or filter. Including one plausible-sounding but subtly off-topic or mildly misleading document approximates the imperfect evidence returned by real retrieval systems and fallible collaborators, and trains the policy to remain robust to it.

\paragraph{The gains are not an artifact of ground-truth conditioning.} \Cref{fig:ablation} reports the setting in which the Knowledge Agent---the only collaborator whose prompt sees the reference answer---is removed entirely. \name still outperforms the RL baselines when restricted to the Verify Agent or the Reasoning Agent alone, neither of which is conditioned on the ground truth. The ground-truth-conditioned prompt is thus a controlled training-time mechanism for producing informative but imperfect interaction signals, and not the source of the reported improvements.

\subsection{Evaluation Benchmarks} \label{app:eval_data}

The following details describe our evaluation benchmarks:
\begin{itemize}[leftmargin=5mm]
    \vspace{0.2em}
    \item \textbf{AIME24.} AIME24 comprises 30 challenging questions from the 2024 American Invitational Mathematics Examination (AIME), designed to test advanced mathematical reasoning skills.
    
    \vspace{0.2em}
    \item \textbf{AIME25.} AIME25 includes 30 challenging questions from the 2025 American Invitational Mathematics Examination (AIME), focusing on complex mathematical problem-solving.
    
    \vspace{0.2em}
    \item \textbf{MATH500.} The original MATH dataset~\citep{HendrycksBKABTS21} contains 12,500 problems from American high school mathematics competitions. For this study, we use MATH500~\citep{LightmanKBEBLLS24}, a subset of the test split consisting exclusively of Level 5 questions.
    
    \vspace{0.2em}
    \item \textbf{LiveMathBench.} LiveMathBench~\citep{abs-2412-13147} is a continuously updated dataset of challenging mathematical problems. We utilize the May 2025 hard split, which includes 100 questions in English.
    
    \vspace{0.2em}
    \item \textbf{GPQA.} The Graduate-Level Google-Proof Q\&A Benchmark (GPQA)~\citep{abs-2311-12022} is a challenging dataset of professional-level, multiple-choice science questions. We evaluate on its diamond subset, comprising 198 questions.
    
    \vspace{0.2em}
    \item \textbf{MBPP.} The Mostly Basic Programming Problems (MBPP) dataset~\citep{abs-2108-07732} evaluates programming models on basic Python tasks. Constructed via crowdsourcing, the problems and solutions undergo revision and manual inspection to ensure clarity and accurate test cases.
    
    \vspace{0.2em}
    \item \textbf{LiveCodeBench.} LiveCodeBench~\citep{JainHGLYZWSSS25} is a benchmark for comprehensive and uncontaminated evaluation of LLM code-related capabilities, incorporating questions from LeetCode, AtCoder, and Codeforces.
    
    \vspace{0.2em}
    \item \textbf{Reasoning-Gym.} Reasoning-Gym~\citep{abs-2505-24760} is a community-developed Python library featuring procedural dataset generators and algorithmically verifiable reasoning environments for training reasoning models with RL. It encompasses over 100 tasks across domains including algebra, arithmetic, computation, cognition, geometry, graph theory, logic, and various games. We generate 270 samples for evaluation, with each of 27 configurations producing 10 sample using \Cref{lst:reasoning_gym_tasks}. 
    
    \begin{lstlisting}[
        language=Python,
        basicstyle=\footnotesize\ttfamily,
        columns=fullflexible,
        keepspaces=true,
        showstringspaces=false,
        breaklines=true,
        breakatwhitespace=false,
        breakindent=1.5em,
        postbreak=\mbox{\textcolor{gray}{$\hookrightarrow$\space}},
        commentstyle=\itshape\color{gray!70!black},
        frame=single,
        framerule=0.4pt,
        rulecolor=\color{gray!40},
        xleftmargin=0.5em,
        xrightmargin=0.5em,
        framexleftmargin=0.3em,
        aboveskip=0.5em,
        belowskip=0.5em,
        caption={Reasoning-Gym task configurations used in our experiments.},
        label={lst:reasoning_gym_tasks},
    ]
    tasks = [
        # String matching
        ("ab", 1.0, {
            "seed": 42, "length": 10, "size": size,
        }),
        ("ab", 1.0, {
            "seed": 42, "length": 15, "size": size,
        }),

        # Geometry and spatial reasoning
        ("acre", 1.0, {
            "seed": 42, "size": size,
        }),
        ("advanced_geometry", 1.0, {
            "seed": 42,
            "min_coord": -100,
            "max_coord": 100,
            "size": size,
        }),

        # Language and logic
        ("aiw", 1.0, {
            "seed": 42,
            "max_entities": 10,
            "size": size,
        }),
        ("cryptarithm", 1.0, {
            "seed": 42,
            "min_words": 5,
            "max_words": 20,
            "size": size,
        }),

        # Puzzles and games
        ("dice", 1.0, {
            "seed": 42,
            "num_dice": 5,
            "max_dice_size": 30,
            "size": size,
        }),
        ("futoshiki", 1.0, {
            "seed": 42, "size": size,
        }),

        # Game of Life
        ("game_of_life", 1.0, {
            "seed": 42,
            "grid_size_x": 30,
            "grid_size_y": 30,
            "simulation_steps": 3,
            "size": size,
        }),
        ("game_of_life", 1.0, {
            "seed": 42,
            "grid_size_x": 30,
            "grid_size_y": 30,
            "simulation_steps": 4,
            "size": size,
        }),
        ("game_of_life", 1.0, {
            "seed": 42,
            "grid_size_x": 30,
            "grid_size_y": 30,
            "simulation_steps": 5,
            "size": size,
        }),
        ("game_of_life_halting", 1.0, {
            "seed": 42,
            "grid_size_x": 30,
            "grid_size_y": 30,
            "difficulty": 3,
            "num_oscillators": 8,
            "max_simulation_steps": 40,
            "size": size,
        }),

        # Planning and search
        ("jugs", 1.0, {
            "seed": 42,
            "difficulty": 20,
            "size": size,
        }),
        ("knight_swap", 1.0, {
            "seed": 42, "size": size,
        }),
        ("rush_hour", 1.0, {
            "seed": 42,
            "min_moves": 10,
            "size": size,
        }),

        # Knights and Knaves
        ("knights_knaves", 1.0, {
            "seed": 42,
            "n_people": 3,
            "depth_constraint": 3,
            "width_constraint": 3,
            "size": size,
        }),
        ("knights_knaves", 1.0, {
            "seed": 42,
            "n_people": 5,
            "depth_constraint": 5,
            "width_constraint": 5,
            "size": size,
        }),

        # Memory and retrieval
        ("mahjong_puzzle", 1.0, {
            "seed": 42,
            "min_num_rounds": 30,
            "size": size,
        }),
        ("needle_haystack", 1.0, {
            "seed": 42,
            "min_num_statements": 50,
            "size": size,
        }),

        # Quantum and constraint reasoning
        ("quantum_lock", 1.0, {
            "seed": 42,
            "difficulty": 10,
            "size": size,
        }),
        ("quantum_lock", 1.0, {
            "seed": 42,
            "difficulty": 20,
            "size": size,
        }),

        # Self-referential reasoning
        ("self_reference", 1.0, {
            "seed": 42,
            "difficulty": 10,
            "size": size,
        }),

        # Classic logic puzzles
        ("sudoku", 1.0, {
            "seed": 42, "size": size,
        }),

        # Zebra puzzles
        ("zebra_puzzles", 1.0, {
            "seed": 42,
            "num_people": 4,
            "num_characteristics": 4,
            "size": size,
        }),
        ("zebra_puzzles", 1.0, {
            "seed": 42,
            "num_people": 5,
            "num_characteristics": 5,
            "size": size,
        }),
        ("zebra_puzzles", 1.0, {
            "seed": 42,
            "num_people": 6,
            "num_characteristics": 6,
            "size": size,
        }),
        ("zebra_puzzles", 1.0, {
            "seed": 42,
            "num_people": 7,
            "num_characteristics": 7,
            "size": size,
        }),
    ]
    \end{lstlisting}
\end{itemize}

\section{Additional Experimental Results and Analysis}

\subsection{Performance of External LLMs} 
\Cref{tab:teacher_performance} shows the performance of Qwen3-30B-A3B-Instruct-2507, which is utilized to implement the external collaborator in this work. 
\begin{table}[ht]
    \centering
    \caption{Performance of Qwen3-30B-A3B-Instruct-2507.} \label{tab:teacher_performance}
    \vspace{0.5em}
    \resizebox{1.0\textwidth}{!}{
        \begin{tabular}{lcccccccc}
            \toprule
            \multirow{3}{*}{\bf Methods} & \multicolumn{4}{c}{\bf Math} & \multicolumn{1}{c}{\bf Science} & \multicolumn{2}{c}{\bf Code} & \multicolumn{1}{c}{\bf Puzzle} \\
            \cmidrule(r){2-5} \cmidrule(r){6-6} \cmidrule(r){7-8} \cmidrule(r){9-9}
            & \bf AIME24 & \bf AIME25 & \bf MATH500 & \bf LMB & \bf GPQA-D & \bf MBPP & \bf LCB & \bf RG \\
            & \bf Avg@\(16\) & \bf Avg@\(16\) & \bf Avg@\(4\) & \bf Avg@\(4\) & \bf Avg@\(4\) & \bf Avg@\(4\) & \bf Avg@\(4\) & \bf Avg@\(4\) \\
            \midrule
            \multicolumn{9}{c}{\textit{Qwen3-30B-A3B-Instruct-2507}} \\
            \midrule
            Original 
            & $76.9_{\pm 0.4}$ 
            & $64.0_{\pm 0.5}$ 
            & $96.8_{\pm 0.2}$ 
            & $44.5_{\pm 0.4}$ 
            & $55.2_{\pm 0.4}$ 
            & $84.1_{\pm 0.4}$ 
            & $44.7_{\pm 0.4}$ 
            & $19.5_{\pm 0.3}$ \\
            \bottomrule
        \end{tabular}
    }
\end{table}

\subsection{Performance of Reinforcement Learning on Llama3.2-3B-Instruct} \label{app:llama3_rl_performance} 

\Cref{tab:llama3_rl_performance} presents the performance of different RLVR algorithms on Llama3.2-3B-Instruct.

\begin{table}[ht]
    \centering
    \caption{Experimental results of reinforcement learning on Llama3.2-3B-Instruct. We report Avg@16 for AIME24 and AIME25 and Avg@4 for the remaining benchmarks. We abbreviate LMB as LiveMathBench v202505, LCB as LiveCodeBench v6, and RG as Reasoning Gym. \(\spadesuit\) denotes the in-domain evaluation benchmark and \(\clubsuit\) denotes the out-of-domain benchmark.} \label{tab:llama3_rl_performance}
    \resizebox{1.0\textwidth}{!}{
        \begin{tabular}{lcccccccccc}
            \toprule
            \multirow{3}{*}{\bf Methods} & \multicolumn{4}{c}{\bf Math \(\spadesuit\)}  & \multicolumn{1}{c}{\bf Science \(\spadesuit\)}  & \multicolumn{2}{c}{\bf Code \(\clubsuit\)}  & \multicolumn{1}{c}{\bf Puzzle \(\clubsuit\)}  \\
            \cmidrule(r){2-5} \cmidrule(r){6-6} \cmidrule(r){7-8} \cmidrule(r){9-9}
            & \bf AIME24 & \bf AIME25 & \bf MATH500 & \bf LMB & \bf GPQA-D & \bf MBPP & \bf LCB & \bf RG \\
            & \bf Avg@\(16\) & \bf Avg@\(16\) & \bf Avg@\(4\) & \bf Avg@\(4\) & \bf Avg@\(4\) & \bf Avg@\(4\) & \bf Avg@\(4\) & \bf Avg@\(4\) \\
            \midrule
            \rowcolor{lightgray!30} \multicolumn{10}{c}{\textit{Llama3.2-3B-Instruct}} \\
            \midrule
            Original 
            & $3.8_{\pm 0.3}$ 
            & $0.3_{\pm 0.1}$ 
            & $40.6_{\pm 0.5}$ 
            & $2.3_{\pm 0.2}$ 
            & $28.9_{\pm 0.4}$ 
            & $35.9_{\pm 0.4}$ 
            & $2.0_{\pm 0.2}$ 
            & $0.2_{\pm 0.1}$ \\
            GRPO 
            & $13.8_{\pm 0.5}$ 
            & $12.5_{\pm 0.4}$ 
            & $62.8_{\pm 0.5}$ 
            & $6.8_{\pm 0.3}$ 
            & $34.7_{\pm 0.4}$ 
            & $38.3_{\pm 0.5}$ 
            & $7.1_{\pm 0.4}$ 
            & $3.6_{\pm 0.3}$ \\
            PRIME 
            & $12.1_{\pm 0.4}$ 
            & $11.2_{\pm 0.5}$ 
            & $60.0_{\pm 0.6}$ 
            & $5.4_{\pm 0.4}$ 
            & $33.1_{\pm 0.5}$ 
            & $37.0_{\pm 0.4}$ 
            & $6.5_{\pm 0.5}$ 
            & $3.3_{\pm 0.4}$ \\
            Dr.GRPO 
            & $14.4_{\pm 0.4}$ 
            & $13.1_{\pm 0.3}$ 
            & $63.3_{\pm 0.4}$ 
            & $6.5_{\pm 0.4}$ 
            & $34.3_{\pm 0.5}$ 
            & $38.9_{\pm 0.3}$ 
            & $7.8_{\pm 0.3}$ 
            & $4.2_{\pm 0.3}$ \\
            \bottomrule
        \end{tabular}
    }
\end{table}

\subsection{\name on Larger Policy Models} \label{app:larger_models}

In this section, we evaluate the effectiveness of \name on LLMs with larger parameter sizes, specifically training Qwen2.5-32B-Instruct~\citep{abs-2412-15115} with \name. 
As shown in \Cref{tab:larger_performance}, the evaluation results demonstrate that \name remains effective for larger-scale models, with \name-trained models outperforming baseline models across all benchmarks. 
Notably, the performance improvements for Qwen2.5-32B-Instruct are more pronounced compared to those for Qwen2.5-7B-Instruct. This enhanced improvement may stem from the 32B model's stronger baseline capabilities, enabling it to formulate higher-quality questions and acquire knowledge more efficiently during training with \name.

\begin{table}[ht]
    \centering
    \caption{Experimental results of \name with Qwen2.5-32B-Instruct. We report the average performance for 16 runs on AIME24 and AIME25, and 4 runs on others. We abbreviate LMB as LiveMathBench v202505, LCB as LiveCodeBench v6, and RG as Reasoning Gym.} \label{tab:larger_performance}
    \vspace{0.5em}
    \resizebox{1.0\textwidth}{!}{
        \begin{tabular}{lcccccccc}
            \toprule
            \multirow{3}{*}{\bf Methods} & \multicolumn{4}{c}{\bf Math} & \multicolumn{1}{c}{\bf Science} & \multicolumn{2}{c}{\bf Code} & \multicolumn{1}{c}{\bf Puzzle} \\
            \cmidrule(r){2-5} \cmidrule(r){6-6} \cmidrule(r){7-8} \cmidrule(r){9-9}
            & \bf AIME24 & \bf AIME25 & \bf MATH500 & \bf LMB & \bf GPQA-D & \bf MBPP & \bf LCB & \bf RG \\
            & \bf Avg@\(16\) & \bf Avg@\(16\) & \bf Avg@\(4\) & \bf Avg@\(4\) & \bf Avg@\(4\) & \bf Avg@\(4\) & \bf Avg@\(4\) & \bf Avg@\(4\) \\
            \midrule
            \multicolumn{9}{c}{\textit{Qwen2.5-32B-Instruct} \; $\leftrightarrow$ \; \textit{Qwen2.5-32B-Instruct}} \\
            \midrule
            Original 
            & $29.9_{\pm 0.4}$ 
            & $28.6_{\pm 0.4}$ 
            & $96.4_{\pm 0.3}$ 
            & $27.8_{\pm 0.5}$ 
            & $62.6_{\pm 0.5}$ 
            & $80.6_{\pm 0.5}$ 
            & $40.2_{\pm 0.4}$ 
            & $28.6_{\pm 0.4}$ \\
            SFT 
            & $32.6_{\pm 0.4}$ 
            & $31.3_{\pm 0.4}$ 
            & $98.3_{\pm 0.2}$ 
            & $30.3_{\pm 0.4}$ 
            & $65.3_{\pm 0.5}$ 
            & $83.6_{\pm 0.4}$ 
            & $42.9_{\pm 0.4}$ 
            & $31.1_{\pm 0.4}$ \\
            GRPO 
            & $35.7_{\pm 0.5}$ 
            & $34.5_{\pm 0.5}$ 
            & $98.6_{\pm 0.2}$ 
            & $33.5_{\pm 0.5}$ 
            & $68.4_{\pm 0.4}$ 
            & $86.7_{\pm 0.5}$ 
            & $46.0_{\pm 0.4}$ 
            & $34.2_{\pm 0.5}$ \\
            \rowcolor{cyan!10} \name 
            & $\textbf{39.1}_{\pm 0.4}$ 
            & $\textbf{39.7}_{\pm 0.4}$ 
            & $\textbf{99.3}_{\pm 0.2}$ 
            & $\textbf{38.9}_{\pm 0.4}$ 
            & $\textbf{72.6}_{\pm 0.4}$ 
            & $\textbf{90.2}_{\pm 0.3}$ 
            & $\textbf{48.4}_{\pm 0.4}$ 
            & $\textbf{44.0}_{\pm 0.4}$ \\
            \bottomrule
        \end{tabular}
    }
\end{table}

\subsection{\name on Different RLVR Algorithms} \label{app:rlvr_algorithms}

To evaluate the generalization capability of \name across different RLVR algorithms, we implemented \name with both DAPO~\citep{abs-2503-14476} and GSPO~\citep{abs-2507-18071}. The results, presented in \Cref{tab:rlvr_performance}, show that \name delivers consistent performance improvements regardless of the underlying RLVR algorithm, thereby demonstrating its strong generalizability.

\begin{table}
    \centering
    \caption{Experimental results of \name on different RLVR algorithms. We report the average performance for 16 runs on AIME24 and AIME25, and 4 runs on others. We abbreviate LMB as LiveMathBench v202505, LCB as LiveCodeBench v6, and RG as Reasoning Gym.} \label{tab:rlvr_performance}
    \vspace{0.5em}
    \resizebox{1.0\textwidth}{!}{
        \begin{tabular}{lcccccccc}
           \toprule
            \multirow{3}{*}{\bf Methods} & \multicolumn{4}{c}{\bf Math \(\spadesuit\)}  & \multicolumn{1}{c}{\bf Science \(\spadesuit\)}  & \multicolumn{2}{c}{\bf Code \(\clubsuit\)}  & \multicolumn{1}{c}{\bf Puzzle \(\clubsuit\)}  \\
            \cmidrule(r){2-5} \cmidrule(r){6-6} \cmidrule(r){7-8} \cmidrule(r){9-9}
            & \bf AIME24 & \bf AIME25 & \bf MATH500 & \bf LMB & \bf GPQA-D & \bf MBPP & \bf LCB & \bf RG \\
            & \bf Avg@\(16\) & \bf Avg@\(16\) & \bf Avg@\(4\) & \bf Avg@\(4\) & \bf Avg@\(4\) & \bf Avg@\(4\) & \bf Avg@\(4\) & \bf Avg@\(4\) \\
            \midrule 
            DAPO 
            & $26.1_{\pm 0.5}$ 
            & $21.0_{\pm 0.4}$ 
            & $80.8_{\pm 0.5}$ 
            & $14.2_{\pm 0.4}$ 
            & $41.1_{\pm 0.5}$ 
            & $62.8_{\pm 0.4}$ 
            & $18.7_{\pm 0.4}$ 
            & $15.7_{\pm 0.4}$ \\
            \rowcolor{cyan!10} \;w/ \name 
            & $\textbf{29.5}_{\pm 0.4}$ 
            & $\textbf{25.3}_{\pm 0.4}$ 
            & $\textbf{85.6}_{\pm 0.3}$ 
            & $\textbf{17.8}_{\pm 0.5}$ 
            & $\textbf{44.2}_{\pm 0.4}$ 
            & $\textbf{66.4}_{\pm 0.3}$ 
            & $\textbf{20.5}_{\pm 0.4}$ 
            & $\textbf{18.9}_{\pm 0.4}$ \\
            \hdashline
            GSPO 
            & $26.8_{\pm 0.4}$ 
            & $20.4_{\pm 0.5}$ 
            & $80.3_{\pm 0.4}$ 
            & $15.2_{\pm 0.4}$ 
            & $41.8_{\pm 0.5}$ 
            & $62.0_{\pm 0.4}$ 
            & $19.2_{\pm 0.4}$ 
            & $15.3_{\pm 0.5}$ \\
            \rowcolor{cyan!10} \;w/ \name 
            & $\textbf{30.2}_{\pm 0.4}$ 
            & $\textbf{24.7}_{\pm 0.4}$ 
            & $\textbf{84.1}_{\pm 0.4}$ 
            & $\textbf{18.5}_{\pm 0.4}$ 
            & $\textbf{45.6}_{\pm 0.4}$ 
            & $\textbf{65.8}_{\pm 0.3}$ 
            & $\textbf{21.3}_{\pm 0.4}$ 
            & $\textbf{19.8}_{\pm 0.4}$ \\
            \bottomrule
        \end{tabular}
    }
\end{table}

\subsection{\name on Long-CoT Policy Models} \label{app:long_cot_models}

In this section, we assess the performance of \name on reasoning LLMs utilizing long CoT prompting. 
Given the substantial inference overhead of long CoT LLMs, we conduct experiments using DeepSeek-R1-Distill-Qwen-1.5B~\citep{abs-2501-12948}, with results presented in \Cref{tab:long_cot_performance}. 
The findings demonstrate that \name achieves consistent performance improvements for LLMs with extended reasoning chains, underscoring the generalization capability of \name across such models.

Additionally, we observe a performance decline in models trained with SFT. 
This may be attributed to the external policy LLM, Qwen3-30B-A3B-Instruct-2507, not being optimized for long CoT reasoning. 
Consequently, fine-tuning based on its responses may disrupt the original reasoning patterns of the original LLM, leading to degraded performance. 
In contrast, \name selectively injects knowledge via activate interactions, preserving its inherent reasoning patterns. 
This preservation represents a key advantage of \name, enhancing its effectiveness without compromising the original LLM's original reasoning capabilities.

\begin{table}[ht]
    \centering
    \caption{Experimental results of \name and baselines with DeepSeek-R1-Distill-Qwen-1.5B. We report the average performance for 16 runs on AIME24 and AIME25, and 4 runs on others. We abbreviate LMB as LiveMathBench v202505, LCB as LiveCodeBench v6, and RG as Reasoning Gym.} \label{tab:long_cot_performance}
    \vspace{0.5em}
    \resizebox{1.0\textwidth}{!}{
        \begin{tabular}{lccccccccc}
            \toprule
            \multirow{3}{*}{\bf Methods} & \multicolumn{4}{c}{\bf Math} & \multicolumn{1}{c}{\bf Science} & \multicolumn{2}{c}{\bf Code} & \multicolumn{1}{c}{\bf Puzzle} \\
            \cmidrule(r){2-5} \cmidrule(r){6-6} \cmidrule(r){7-8} \cmidrule(r){9-9}
            & \bf AIME24 & \bf AIME25 & \bf MATH500 & \bf LMB & \bf GPQA-D & \bf MBPP & \bf LCB & \bf RG \\
            & \bf Avg@\(16\) & \bf Avg@\(16\) & \bf Avg@\(4\) & \bf Avg@\(4\) & \bf Avg@\(4\) & \bf Avg@\(4\) & \bf Avg@\(4\) & \bf Avg@\(4\) \\
            \midrule
            \multicolumn{10}{c}{\textit{DeepSeek-R1-Distill-Qwen-1.5B} \; $\leftrightarrow$ \; \textit{Qwen3-30B-A3B-Instruct-2507}} \\
            \midrule
            Original 
            & $21.9_{\pm 0.4}$ 
            & $21.5_{\pm 0.4}$ 
            & $84.0_{\pm 0.5}$ 
            & $13.0_{\pm 0.4}$ 
            & $29.8_{\pm 0.5}$ 
            & $60.1_{\pm 0.4}$ 
            & $14.7_{\pm 0.4}$ 
            & $3.3_{\pm 0.2}$ \\
            SFT 
            & $18.4_{\pm 0.4}$ 
            & $19.9_{\pm 0.4}$ 
            & $77.2_{\pm 0.5}$ 
            & $14.0_{\pm 0.4}$ 
            & $26.6_{\pm 0.5}$ 
            & $55.5_{\pm 0.4}$ 
            & $15.3_{\pm 0.4}$ 
            & $11.0_{\pm 0.4}$ \\
            GRPO 
            & $28.4_{\pm 0.4}$ 
            & $25.7_{\pm 0.4}$ 
            & $86.8_{\pm 0.3}$ 
            & $17.4_{\pm 0.3}$ 
            & $34.7_{\pm 0.4}$ 
            & $65.1_{\pm 0.4}$ 
            & $14.7_{\pm 0.4}$ 
            & $13.5_{\pm 0.4}$ \\
            \rowcolor{cyan!10} \name 
            & $\textbf{30.6}_{\pm 0.4}$ 
            & $\textbf{29.2}_{\pm 0.4}$ 
            & $\textbf{88.5}_{\pm 0.3}$ 
            & $\textbf{18.9}_{\pm 0.3}$ 
            & $\textbf{36.5}_{\pm 0.3}$ 
            & $\textbf{67.2}_{\pm 0.4}$ 
            & $\textbf{17.1}_{\pm 0.4}$ 
            & $\textbf{16.4}_{\pm 0.4}$ \\
            \bottomrule
        \end{tabular}
    }
\end{table} 

\subsection{Efficiency Analysis}
\label{app:efficiency}

Active interaction introduces additional training-time computation because collaborator models must be queried during rollout.
To quantify this trade-off, \Cref{tab:efficiency} reports wall-clock time per training step, peak GPU memory, rollout FLOPs per step, total GPU hours, and the resulting average performance over the eight benchmarks in \Cref{tab:main_performance}.
All methods are measured using the same hardware setting and inference-serving framework.
For methods that rely on offline expert trajectories, total GPU hours additionally include the cost of generating those trajectories.

\begin{table}[ht]
    \centering
    \caption{Efficiency comparison on Qwen2.5-7B-Instruct. \textit{Average Performance} is the unweighted mean over the eight benchmarks of \Cref{tab:main_performance}. Total GPU hours include offline expert-trajectory generation where applicable.} \label{tab:efficiency}
    \resizebox{1.0\textwidth}{!}{
        \begin{tabular}{llccccc}
            \toprule
            \bf Method & \bf Collaborator & \bf Time / Step (min) & \bf Peak GPU Mem. (GB) & \bf Rollout FLOPs ($\times10^{18}$) / Step & \bf Total GPU Hours & \bf Avg.\ Perf. \\
            \midrule
            GRPO  & None            & 11.2 & 286 & 5.8 & 149 & 33.3 \\
            LUFFY & Qwen3-30B-A3B   & 12.4 & 292 & 6.1 & 310 & 35.4 \\
            OPSD  & Qwen3-30B-A3B   & 16.3 & 441 & 8.6 & 326 & 34.3 \\
            \rowcolor{cyan!10} \name & Qwen2.5-7B-Instruct    & 13.7 & 307 & 6.5 & 166 & 35.4 \\
            \rowcolor{cyan!10} \name & Qwen3-30B-A3B & 16.4 & 438 & 8.8 & 329 & 37.1 \\
            \bottomrule
        \end{tabular}
    }
\end{table}

With a same-scale Qwen2.5-7B collaborator, \name raises the time per step from 11.2 to 13.7 minutes ($+22\%$) and the total cost from 149 to 166 GPU hours ($+11\%$) over vanilla GRPO, while improving the eight-benchmark average from 33.3 to 35.4. With Qwen3-30B-A3B the cost is comparable to OPSD (16.4 vs.\ 16.3 minutes per step, 329 vs.\ 326 GPU hours) but the average is 2.8 points higher, and comparable to LUFFY in total cost while being 1.7 points higher. The additional memory relative to GRPO comes from serving the auxiliary model, which is conceptually the same requirement as deploying an LLM-based reward model or an expert teacher.

\paragraph{Matched-Compute Comparison.} Cost and accuracy can also be separated by fixing the budget. \Cref{tab:equal_budget} terminates GRPO and \name at the same 10 GPU-hour budget.

\begin{figure}[ht]
    \centering

    \begin{minipage}[c]{0.45\textwidth}
        \centering

        \captionof{table}{
        GRPO and \name under an identical 10 GPU-hour training budget.
        Each column reports the average over the benchmarks of that domain in
        \Cref{tab:main_performance}.
        }
        \label{tab:equal_budget}

        \resizebox{\linewidth}{!}{
        \begin{tabular}{lcccc}
            \toprule
            \bf Method & \bf Math & \bf Science & \bf Code & \bf Puzzle \\
            \midrule
            GRPO
            & $33.5_{\pm 0.4}$
            & $37.6_{\pm 0.5}$
            & $38.4_{\pm 0.4}$
            & $14.1_{\pm 0.3}$ \\
            \rowcolor{cyan!10}
            \name
            & $\textbf{35.0}_{\pm 0.3}$
            & $\textbf{40.3}_{\pm 0.4}$
            & $\textbf{40.2}_{\pm 0.4}$
            & $\textbf{15.3}_{\pm 0.3}$ \\
            \bottomrule
        \end{tabular}
        }

        \vspace{1.0em}

        \captionof{table}{
        Total number of calls to each collaborator over the full training run.
        }
        \label{tab:agent_calls}

        \resizebox{\linewidth}{!}{
        \begin{tabular}{lccc}
            \toprule
            \bf Statistic
            & \bf Verify
            & \bf Knowledge
            & \bf Reasoning \\
            \midrule
            Total calls
            & 583{,}025
            & 302{,}776
            & 410{,}910 \\
            \bottomrule
        \end{tabular}
        }

    \end{minipage}
    \hfill
    \begin{minipage}[c]{0.52\textwidth}
        \centering

        \includegraphics[
            width=\linewidth
        ]{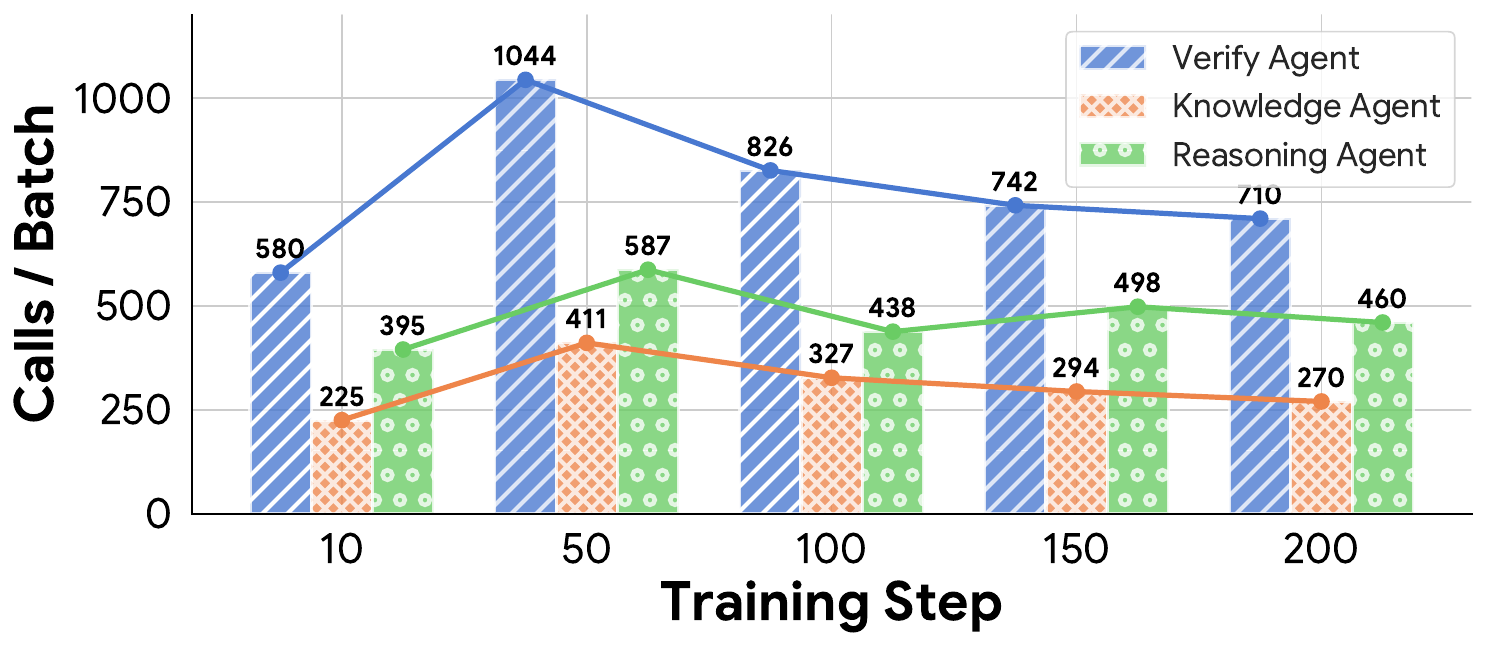}

        \captionof{figure}{
        Per-batch calls to each collaborator at different training steps.
        }
        \label{fig:agent_calls_dynamics}

    \end{minipage}

\end{figure}

Under the matched computational budget, \name remains ahead across all four task families.
This comparison indicates that the observed performance improvements are not explained solely by allowing \name to consume more total training computation.

\paragraph{The overhead is adaptive, and absent at inference.} \name does not query a collaborator at every reasoning step. 
The policy decides when assistance is needed, and as training proceeds it consults collaborators less often (\Cref{app:interaction_stats}). 
The average cost per rollout therefore decreases over training, in contrast to frameworks that supervise every trajectory with a teacher. After training, the policy reasons alone: there is no additional deployment-time memory, latency, or API cost, which preserves the central practical advantage of \name over multi-agent inference systems. 

\subsection{Statistics of Collaborator Invocation} \label{app:interaction_stats}

To characterize how the policy actually uses its collaborators, we log every interaction over the full training run of Qwen2.5-7B-Instruct with same-backbone collaborators.

\paragraph{How often is no collaborator used?} Approximately $10\%$ of training samples trigger no interaction at all. 
Inspection shows these are predominantly short, routine problems that the policy already solves reliably. This is the intended behaviour: collaborators are consulted selectively rather than unconditionally.

\paragraph{How are calls distributed across collaborators?}
The Verify Agent is used most frequently, which is consistent with its breadth: both mathematical and scientific reasoning require checking intermediate conclusions at many points in a trajectory. The Knowledge Agent is used least, since a large share of our training questions concern theorems and concepts already covered during pretraining. The Reasoning Agent sits in between and is invoked when a localized sub-problem blocks progress. Importantly, the policy does not collapse onto a single collaborator: the two less-used agents still account for a substantial share of all interactions.

\paragraph{How does usage evolve during training?}
As shown in \Cref{fig:agent_calls_dynamics}, invocation frequency rises in the early phase, peaks, and then declines for all three collaborators, mirroring the aggregate curve in \Cref{fig:training_dynamics}. 
We read this as a two-stage process: the policy first has to learn \emph{when} an interaction is worthwhile, which increases usage. 
Once it has internalized the knowledge and reasoning strategies acquired through earlier interactions, it solves comparable problems independently and its reliance on collaborators falls. The relative proportions among the three collaborators remain stable throughout, indicating that the decline is a reduction in need rather than a shift toward one preferred agent. 

\subsection{Proportion of Near-Vanishing Gradients}
\label{app:vanishing_stats}

The clipping strategy in \Cref{eq:clipping_strategy} is designed for external tokens whose \emph{unclipped} coefficient
$\pi_\theta(\bm{\tau}_t\mid\bm{\tau}_{<t})$
is smaller than the lower-bound threshold $\omega$ and would therefore produce a very small gradient contribution.
Because exhaustively measuring every external token throughout training is prohibitively expensive, we perform checkpoint-based diagnostics.
At each selected checkpoint, we randomly sample $1{,}000$ training instances, perform model rollouts, and measure the fraction of external tokens satisfying
$\pi_\theta(\bm{\tau}_t\mid\bm{\tau}_{<t}) < \omega$
\emph{before} applying the proposed lower-bound clipping operation.

\begin{figure}[ht]
    \centering

    \begin{minipage}[t]{0.48\textwidth}
        \centering
        \includegraphics[width=0.82\linewidth]{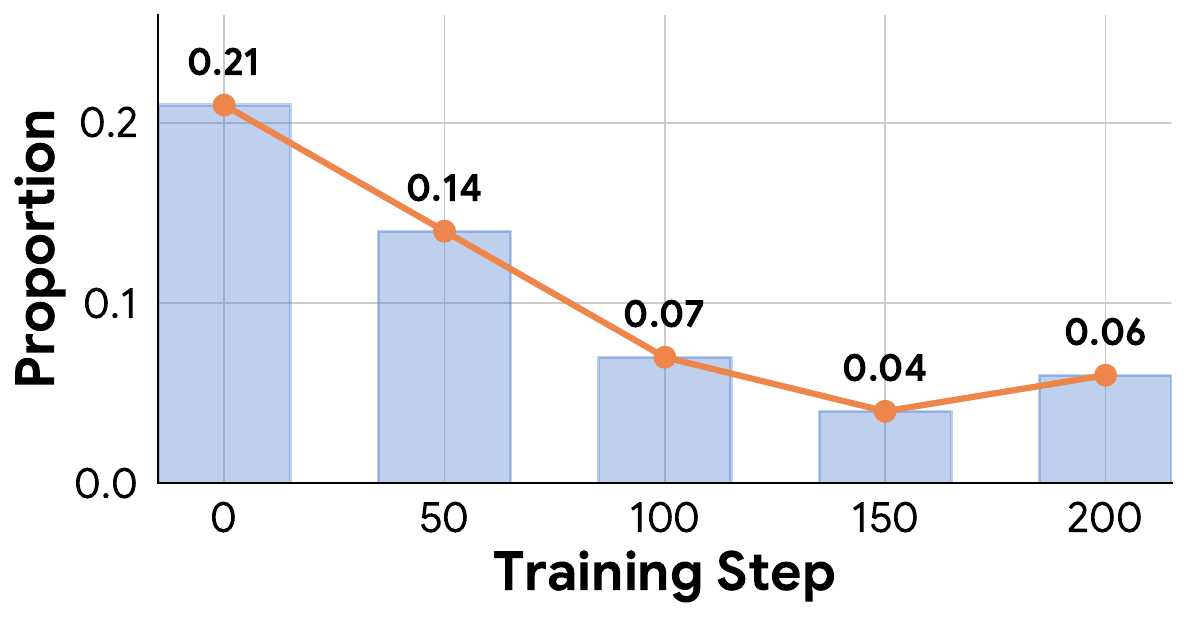}
        \captionof{figure}{
        Proportion of external tokens whose \emph{unclipped} corrected coefficient falls below the lower-bound threshold $\omega$ at different training checkpoints.
        }
        \label{fig:vanishing_proportion}
    \end{minipage}
    \hfill
    \begin{minipage}[t]{0.48\textwidth}
        \centering
        \includegraphics[width=0.82\linewidth]{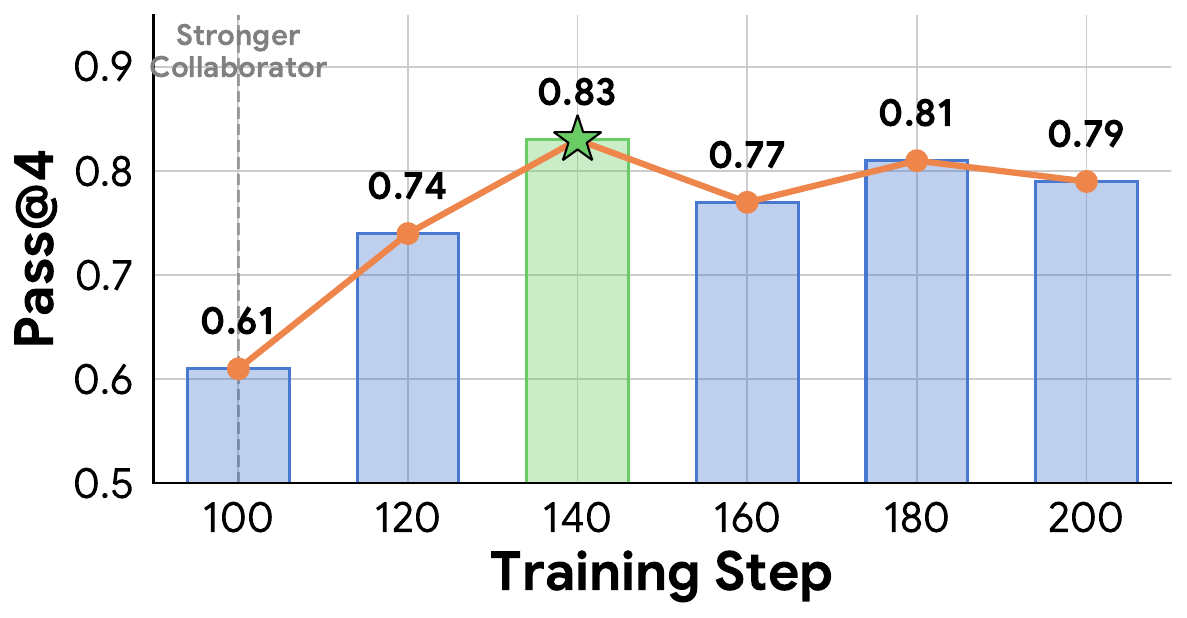}
        \captionof{figure}{
        Pass@4 on DeepMath after introducing Qwen3-30B-A3B-Instruct-2507 as the stronger collaborator after step 100.
        }
        \label{fig:collaborator_capability}
    \end{minipage}

\end{figure}

As shown in \Cref{fig:vanishing_proportion}, the proportion is highest at the beginning of training, when many collaborator-provided tokens have low probability under the policy, and decreases from $0.21$ at initialization to $0.07$ by step 100.
This pattern is consistent with the intended role of the clipping mechanism: early in training, the lower bound preserves learning signals for external tokens that would otherwise receive near-vanishing updates. 
As the policy increasingly internalizes useful reasoning patterns from successful mixed-policy trajectories, these tokens become less unlikely under the policy and the protected regime becomes less frequent.

\subsection{Effect of Collaborator Capability}
\label{app:collaborator_capability}

A natural question is whether a stronger collaborator necessarily induces a monotonically larger effective exploration boundary.
Our hypothesis is that stronger collaborators can expose the policy to a broader set of useful reasoning trajectories, but the marginal benefit need not be monotonic: once collaborator feedback becomes sufficiently informative, or exceeds what the current policy can effectively internalize, further increases in collaborator capability may yield diminishing returns.

To examine this effect, we continue the DeepMath experiment of \Cref{sec:boundary} and replace the original same-scale collaborator with the stronger Qwen3-30B-A3B-Instruct-2507 after step 100.

As shown in \Cref{fig:collaborator_capability}, Pass@4 increases from $0.61$ at step 100 to $0.83$ at step 140 after introducing the stronger collaborator, but does not continue to increase monotonically thereafter.
This result suggests that collaborator capability is an important source of useful exploration signals, while policy-side internalization eventually becomes an additional bottleneck.
We therefore do not claim a monotonic relationship between collaborator scale and the final capability of the trained policy.

\subsection{Case Study}

\paragraph{Training Cases. }
We present examples of \name's training trajectories in Case~\ref{em:example_aime2024} and Case~\ref{em:example_bamboogle}. These cases demonstrate the ability to formulate reasonable and meaningful interactions, leverage information from the external model, and successfully complete the reasoning process. Through this interaction, the model effectively acquires new information and enhances its inherent capabilities.

\begin{example}{Case on AIME2024}{example_aime2024}

{ \small

    \section*{Question}
    Jen enters a lottery by picking $4$ distinct numbers from $S=\{1,2,3,\cdots,9,10\}.$ $4$ numbers are randomly chosen from $S.$ She wins a prize if at least two of her numbers were $2$ of the randomly chosen numbers, and wins the grand prize if all four of her numbers were the randomly chosen numbers. The probability of her winning the grand prize given that she won a prize is $\tfrac{m}{n}$ where $m$ and $n$ are relatively prime positive integers. Find $m+n$.
    
    \section*{Answer}

    $116$

    \section*{Trajectory}

    \noindent \textbf{Reasoning Process}
        
    Find the sum $m+n$, where $\frac{m}{n}$ is the conditional probability that Jen wins the grand prize given that she wins \textit{any} prize.

    \begin{itemize}
        \item Total pool of numbers: $S = \{1, 2, \dots, 10\}$ (Size 10).
        \item Jen's pick ($J$): 4 specific numbers.
        \item Winning numbers ($W$): 4 randomly chosen numbers.
        \item Event $E$ (Prize): The intersection $|J \cap W|$ is at least 2 (i.e., 2, 3, or 4 matches).
        \item Event $G$ (Grand Prize): The intersection $|J \cap W|$ is exactly 4.
    \end{itemize}
        
    The conditional probability is $P(G|E) = \frac{N(G \cap E)}{N(E)}$. Since $G$ is a subset of $E$ (4 matches implies at least 2), this simplifies to $\frac{N(G)}{N(E)}$.
        
    To find the number of ways to match exactly $k$ numbers:
    \begin{itemize}
        \item We need to choose $k$ numbers from the 4 winning numbers: $\binom{4}{k}$.
        \item We need to choose the remaining $4-k$ numbers from the 6 non-winning numbers: $\binom{6}{4-k}$.
    \end{itemize}
        
    I need to calculate the specific combinatorial values for $k=2, 3, 4$ and their sum to determine $N(E)$ and $N(G)$. I will ask an external environment to perform these calculations. 
    
    \begin{quote}
    \texttt{<reasoning>} \\
    Calculate the number of ways to match exactly k numbers in a lottery where 4 numbers are drawn from 10. Specifically, compute the values for $C(4,k) \cdot C(6, 4-k)$ for $k=2$, $k=3$, and $k=4$, and find the sum of these three results. \\
    \texttt{</reasoning>} \\
    \texttt{<result>} \\
    For exactly (k) matches, the count is \(\binom{4}{k}\binom{6}{4-k}\). \\
    For each value: \\
    $$
    k=2:\ \binom{4}{2}\binom{6}{2}=6\cdot 15=90
    $$
    $$
    k=3:\ \binom{4}{3}\binom{6}{1}=4\cdot 6=24
    $$
    $$
    k=4:\ \binom{4}{4}\binom{6}{0}=1\cdot 1=1
    $$
    So the total is: \\
    $$
    90+24+1=115
    $$
    Final answer: \boxed{115} \\
    \texttt{</result>} \\
    \end{quote}

    The external result confirms the counts for each case.
    $N(E)$ (Total ways to win a prize): This is the sum of outcomes for $k=2, 3, 4$. According to the result, $90 + 24 + 1 = 115$.
    $N(G)$ (Ways to win grand prize): This corresponds to the $k=4$ case, which is $1$.
    The conditional probability is $\frac{1}{115}$.
    Identify $m$ and $n$: $m=1$, $n=115$.
    1 and 115 are relatively prime positive integers.
    Calculate $m+n = 1 + 115 = 116$.

    \vspace{1em}

    \noindent \textbf{Final Answer}

    \vspace{0.5em}

    The probability that Jen wins the grand prize given she wins a prize is $\frac{1}{115}$. Therefore, $m=1$ and $n=115$, and the sum $m+n$ is \boxed{116}.

}
\end{example}

\begin{example}{Case on Bamboogle}{example_bamboogle}
    \section*{Question}
    What is the capital of the second largest state in the US by area?
  
    \section*{Answer}
    austin

    \section*{Trajectory}

    \noindent \textbf{Reasoning Process}

    To determine the capital of the second largest state in the US by area, I need to follow a logical sequence.
    \begin{itemize} 
        \item First, I must definitively identify which US state ranks second in terms of total area. 
    
        \item Second, once that state is identified, I need to determine its capital city.
    \end{itemize}

    I will start by asking the external environment to identify the second largest state in the US.

    \begin{quote}
    \texttt{<retrieval>}\\
    What is the second largest state of US?\\
    \texttt{</retrieval>}
    \end{quote}

    \begin{quote}
    \texttt{<result>}\\
    Doc 1: U.S. state size is commonly measured by total area, including both land and inland water. Under this measure, Alaska is the largest state by a wide margin, with more than 660,000 square miles of total area.

    Doc 2: Texas is the second largest U.S. state by total area. It covers approximately 268,596 square miles, making it substantially larger than California and every other state except Alaska.

    Doc 3: California is often associated with large population and economic scale, but it is not the second largest U.S. state by area. Its total area is about 163,695 square miles, placing it behind Alaska and Texas.

    Doc 4: The ranking of U.S. states by area depends on whether total area or land area is used, but Texas remains second under both common measures. Alaska ranks first in both land area and total area.

    Doc 5: State area rankings differ from population rankings. Texas is among the most populous U.S. states, but its status as the second largest state refers to geographic area, not number of residents.

    Doc 6: The continental United States excludes Alaska and Hawaii when describing the contiguous forty-eight states. Within the contiguous states, Texas is the largest by area, while California is second among those states.
    \texttt{</result>}
    \end{quote}

    The external result confirms that \textbf{Texas} is the second largest state in the US by area. Now I need to find the capital of Texas.

    \begin{quote}
    \texttt{<retrieval>}\\
    What is the capital of Texas?\\
    \texttt{</retrieval>}
    \end{quote}

    \begin{quote}
    \texttt{<result>}\\
    <result>
    Doc 1: Austin is the capital city of Texas. It is located in Central Texas along the Colorado River and serves as the seat of Texas state government, including the offices of the governor and the state legislature.

    Doc 2: Texas state government is centered in Austin, where the Texas State Capitol houses the offices and chambers of the Texas Legislature. The city has served as the state capital since the nineteenth century.

    Doc 3: Houston is the largest city in Texas by population and an important economic center, but it is not the state capital. The capital designation belongs to Austin, which hosts the main institutions of state government.

    Doc 4: The Texas State Capitol is located in Austin and is one of the most recognizable government buildings in the state. Its presence reflects Austin's role as the administrative and political capital of Texas.

    Doc 5: Austin is also the county seat of Travis County and a major center for education, technology, and culture. Its civic role includes both local government functions and statewide governmental responsibilities.

    Doc 6: Capital cities are generally selected as seats of government rather than by population size or economic output. In Texas, Austin fulfills this role, while larger cities such as Houston, San Antonio, and Dallas serve other regional functions. </result>

    \texttt{</result>}
    \end{quote}

    From the retrieved information, I have verified that the capital of Texas is Austin.

    \vspace{1em}

    \noindent \textbf{Final Answer}

    The second largest state in the US by area is Texas, and its capital is \boxed{Austin}.
  
\end{example} 

\paragraph{Inference Cases. }
It is worth noting that the functional agents are primarily used during training to provide structured and fine-grained interactive guidance. 
At inference time, the policy model does not interact with the same external collaborative agents. 
Instead, the interaction patterns learned during training encourage the model to internalize agent-like behaviors, such as verifying intermediate conclusions, identifying missing knowledge, and decomposing difficult sub-problems. 
The following examples illustrate the distinction between training-time interactive exploration and inference-time reasoning. 
During training, the policy model explicitly invokes functional agents through predefined tags and receives external feedback. 
At inference time, the trained model can internalize these interaction patterns and perform verification, knowledge retrieval, and sub-problem solving as implicit reasoning behaviors, even without explicit agent calls. 

\begin{example}{Inference-Time Case on AIME2024}{example_inference_aime2024}

{ \small

    \section*{Question}
    Jen enters a lottery by picking $4$ distinct numbers from $S=\{1,2,3,\cdots,9,10\}.$ 
    $4$ numbers are randomly chosen from $S.$ 
    She wins a prize if at least two of her numbers were $2$ of the randomly chosen numbers, and wins the grand prize if all four of her numbers were the randomly chosen numbers. 
    The probability of her winning the grand prize given that she won a prize is $\tfrac{m}{n}$ where $m$ and $n$ are relatively prime positive integers. 
    Find $m+n$.
    
    \section*{Answer}

    $116$

    \section*{Trajectory}

    \noindent \textbf{Reasoning Process}

    Jen fixes a set $J$ of $4$ numbers from $10$, and the lottery randomly selects another set $W$ of $4$ numbers. 
    The prize event occurs when $|J \cap W| \geq 2$, while the grand prize event occurs when $|J \cap W|=4$.

    Since Jen's chosen set has size $4$, the number of lottery outcomes with exactly $k$ matches is
    \[
        \binom{4}{k}\binom{6}{4-k},
    \]
    where $\binom{4}{k}$ chooses the matched numbers from Jen's set and $\binom{6}{4-k}$ chooses the remaining non-matched numbers from the other $6$ numbers.

    Thus, the number of prize-winning outcomes is
    \[
    \begin{aligned}
        N(E)
        &= \binom{4}{2}\binom{6}{2}
            + \binom{4}{3}\binom{6}{1}
            + \binom{4}{4}\binom{6}{0} \\
        &= 6 \cdot 15 + 4 \cdot 6 + 1 \cdot 1 \\
        &= 90 + 24 + 1 \\
        &= 115.
    \end{aligned}
    \]
    
    The grand-prize event corresponds to exactly $4$ matches, so
    \[
        N(G)=\binom{4}{4}\binom{6}{0}=1.
    \]
    Since every grand-prize outcome is also a prize-winning outcome, we have $G \subseteq E$, and therefore
    \[
        P(G \mid E)=\frac{N(G)}{N(E)}=\frac{1}{115}.
    \]
    
    Therefore, $m=1$ and $n=115$, which gives
    \[
        m+n=116.
    \]

    \vspace{1em}

    \noindent \textbf{Final Answer}

    \vspace{0.5em}

    The conditional probability is $\frac{1}{115}$. Hence, $m=1$, $n=115$, and the final answer is $\boxed{116}$.

}
\end{example} 

Although no external agent is explicitly invoked at inference time, the model reveals the corresponding abilities learned during training:
\begin{itemize}[leftmargin=*]
    \item \textbf{Verification-like behavior:} Check whether the grand-prize event is a subset of the prize event.
    \item \textbf{Knowledge-like behavior:} Recall the combinatorial counting rule for exactly $k$ matches.
    \item \textbf{Reasoning-like behavior:} Compute the number of outcomes corresponding to $k=2,3,4$ matches.
\end{itemize}

\section{Discussions} 

\subsection{More Discussions on Related Work}
\label{app:related_work}

The work most closely related to \name is recent research on external-guidance-enhanced reinforcement learning and distillation~\citep{abs-2504-14945,abs-2509-04419,abs-2508-11408,abs-2509-06948,abs-2506-05316,abs-2509-26306,abs-2506-07527,abs-2506-19767,abs-2601-18734}.
Existing methods mainly incorporate external information through supervised learning on expert trajectories, token-level supervision from teacher models, or RL objectives that learn from off-policy expert trajectories.
These approaches differ from \name in when external guidance enters the learning process.
\name introduces collaborator feedback directly during rollout, conditioned on the policy's current partial trajectory and reasoning bottleneck, thereby constructing mixed-policy trajectories that may be unlikely to arise from unaided on-policy exploration.
Learning from these trajectories is supported by the corrected importance sampling coefficient and clipping strategy introduced in \Cref{sec:learn_interaction}.

\paragraph{Comparison with On-Policy Distillation.}
On-policy distillation~\citep{abs-2601-18734} preserves student-generated rollouts while using teacher logits as token-level supervision.
It therefore improves how the student learns from trajectories that have already been generated, but does not directly modify the states visited during rollout.
In contrast, \name introduces external feedback inside the rollout itself.
The policy identifies an intermediate bottleneck and queries a collaborator conditioned on its current reasoning state, so the external guidance directly affects the resulting trajectory distribution.
A more detailed discussion is provided in \Cref{app:opd_comparison}.
We further discuss the distinction from SFT followed by RL in \Cref{app:sft_then_rl}.

\paragraph{Comparison with Multi-Agent Guidance.}
Among existing methods, \citet{abs-2509-26306} is particularly related because it also introduces multiple agents during rollout for communication and feedback.
However, \name differs in two aspects.
First, interaction is initiated adaptively by the policy at intermediate reasoning bottlenecks, allowing external assistance to be conditioned on the current partial trajectory.
Second, \name explicitly optimizes mixed-policy trajectories containing both policy-generated and collaborator-generated tokens, whereas \citet{abs-2509-26306} optimizes the policy primarily through on-policy trajectories.
This distinction allows \name to use external collaborators as part of the exploration process rather than only as auxiliary sources of feedback.

\paragraph{Comparison with Self-Reflection and Self-Evolution.}
Self-reflection and self-evolution methods~\citep{ShinnCGNY23,DouY0CP24} typically rely on feedback, critiques, or revisions generated by the policy model itself.
Their exploration is therefore still primarily determined by information accessible to the model.
In contrast, \name obtains additional reasoning signals from external collaborators during training, which can introduce knowledge or reasoning patterns that are unlikely under the current policy distribution.
Importantly, policy optimization in \name remains grounded in verifiable task rewards rather than relying solely on self-generated supervision.

\subsection{Mitigating Errors in Active Interactions} \label{app:error_mitigation}
While \name achieves significant performance improvements, active interaction with collaborators may introduce erroneous information, due to hallucinations~\citep{HuangYMZFWCPFQL25}. 
Such errors can propagate through the reasoning process and potentially mislead the policy model. 
However, \name optimizes the policy model toward maximizing expected rewards defined by verifiable outcome correctness, providing a stable learning signal that consistently favors trajectories leading to correct solutions. 
Moreover, prior studies show that, when guided by reasonable and verifiable rewards, policy models can learn to filter, retrieve, and reflect on information, thereby extracting useful signals from noisy or imperfect inputs~\citep{abs-2505-24726,KumarZASCSBIBRZ25,abs-2506-01369,abs-2509-25760,abs-2501-12948}. 
Therefore, although active interaction introduces the risk of error propagation, it also expands the exploration space and provides richer learning opportunities. 
The RLVR objective enables the policy model to exploit these benefits while progressively mitigating the negative effects of unreliable collaborator feedback.

\subsection{Why \name Does Not Require Distributionally Matched Collaborators}
\label{app:collaborator_agnostic}

\Cref{eq:modified_importance_sampling} contains no term depending on $\pi_{\epsilon}$, which may appear to imply that the objective is indifferent to which collaborator is used. We clarify what this property does and does not mean.

Importance sampling imposes no formal constraint on the relation between behaviour and target policy, but it is only well behaved in practice when the two are not too dissimilar: a large discrepancy produces extreme importance weights and destabilizes optimization. This is precisely the failure mode our approximation is designed to avoid. Rather than estimating the collaborator's token-level distribution---which is often inaccessible, tokenizer-incompatible, or expensive to evaluate---we approximate it by a one-hot distribution on the sampled external token, so the coefficient reduces to the probability the \emph{policy} assigns to that token. This quantity is bounded in $[0,1]$ by construction, and the pathological weights of conventional importance sampling under strong distribution mismatch cannot arise.

Two consequences follow. First, \name needs neither collaborator log-probabilities nor a shared vocabulary, so it can learn from a broad range of external agents---larger models, differently trained models, and agents specialized in verification, knowledge retrieval, or sub-problem solving. Second, the lower-bound clip of \Cref{eq:clipping_strategy} ensures that informative external tokens are not discarded merely because the policy currently assigns them low probability, which is exactly the situation created by a large policy--collaborator gap.

This robustness should not be read as a guarantee that any collaborator helps. What \name relaxes is a \emph{compatibility} requirement: the collaborator need not resemble the policy distribution. It does not relax the \emph{quality} requirement: feedback that is systematically irrelevant, or adversarially misleading, still injects unhelpful tokens into the trajectory. Two mechanisms limit the damage---the verifiable outcome reward down-weights trajectories that fail regardless of their external content (\Cref{app:error_mitigation}), and the advantage term $\tilde{A}_t$ can be negative for such tokens---but we do not claim robustness to an adversarial collaborator, and we state this explicitly in \Cref{app:limitation}.

\subsection{Comparison with On-Policy Distillation} \label{app:opd_comparison}

On-policy distillation (OPD) keeps trajectory generation on-policy while supervising it with teacher token-level distributions, and can be combined with RL so the student also optimizes task-level rewards. Its limitation is structural rather than a matter of supervision strength: guidance and exploration remain decoupled. The teacher scores trajectories that the \emph{unaided} policy has already produced, so the trajectory distribution is still determined by the policy's unaided rollout behaviour, and the supervision signal is static with respect to where the policy actually got stuck. When teacher and student behaviours diverge substantially, token-level matching can also produce ineffective or negative transfer.

\name intervenes inside the rollout instead. The policy detects an intermediate bottleneck, chooses a collaborator, and receives feedback conditioned on its current partial trajectory. 
The resulting mixed-policy trajectory is one the policy could not have generated alone. Three consequences distinguish it from OPD. First, \name introduces reasoning directions that are unlikely to be discovered by independent exploration, so it changes which trajectories exist rather than only how they are scored. Second, guidance is targeted and state-dependent, and is requested only where it is needed, which makes it more sample-efficient than supervising every token of every trajectory. Third, the policy must \emph{integrate} the feedback into its own reasoning to obtain reward, rather than reproduce a teacher sequence, and it operates without collaborators at inference.

Empirically, \Cref{tab:main_performance} compares \name against OPSD, our OPD-style baseline, under identical data, optimizer, and budget: \name is ahead on all eight benchmarks for both policy models and both collaborator choices.

\subsection{Can the Policy Model Surpass its Collaborators?} \label{app:surpass}

Whether the trained policy can exceed the collaborator depends on the source of the capability gap.

When the collaborator is a substantially larger and more extensively trained model---for example Qwen3-30B-A3B-Instruct-2507 assisting Qwen2.5-7B-Instruct (\Cref{tab:teacher_performance})---the policy generally does not overtake it within our training budget. The gap there stems from model scale, pretraining data, and post-training budget, none of which interaction can supply.

When the collaborator shares the policy's backbone, the trained policy \emph{can} surpass the standalone collaborator. The comparison is not between two equal models: the collaborators are used in three narrow functional roles---verification, knowledge synthesis, sub-problem solving---whose outputs the policy must combine, while the policy additionally receives task-level verifiable rewards. It can therefore internalize a composition of capabilities that no single role provides, and the same-backbone blocks of \Cref{tab:main_performance} show exactly this: \name with Qwen2.5-7B collaborators reaches $26.5$ on AIME24 against $9.8$ for the untrained shared backbone, and also exceeds every RL baseline trained without interaction.

\subsection{Comparison with SFT Initialization followed by RL} \label{app:sft_then_rl}

A natural alternative is to generate expert data by multi-agent collaboration once, perform SFT, and then run RL. This is a reasonable initialization strategy, but it does not address the bottleneck \name targets.

The purpose of \name is not to introduce expert demonstrations but to change the exploration process itself through proactive, state-dependent interaction. SFT supervision is static and fixed before RL begins, so it cannot respond to where the policy fails \emph{during} training. 
It encourages imitation of complete expert behaviour without identifying which intermediate decision resolves the student's current failure, and large teacher--student discrepancies make such imitation less effective. Once RL starts, exploration is again bounded by the policy's own distribution.

Our SFT and OPSD baselines instantiate the two closest variants of this idea---SFT on rejection-sampled expert trajectories, and teacher supervision injected during on-policy training---and both are outperformed by \name under matched budgets (\Cref{tab:main_performance}). We read this as evidence that the gain comes from active, adaptive, bottleneck-conditioned interaction rather than from access to expert content per se. The two are nevertheless complementary: SFT initialization could be combined with \name, and we leave this to future work.

\section{Limitations and Future Work} \label{app:limitation}

While \name achieves significant performance improvements, several areas warrant further exploration. 
First, our method is currently limited to verifiable questions with definitive answers. 
Training on open-ended questions remains an open problem in the community and is beyond the scope of this paper. 
Second, due to computational resource constraints, we conducted experiments only on LLMs up to 32B parameters. 
Given that the scaling law~\citep{abs-1712-00409,abs-2010-14701} is an important principle in the field of LLMs, investigating the performance of \name on larger-scale LLMs will be another promising direction. 
Thirdly, the current scope of our experimentation and analysis is limited exclusively to the English language. 
Therefore, the applicability and performance of \name with interfaces and information presented in languages other than English remain an open question. 
Addressing this linguistic limitation is crucial for establishing the generalizability of the proposed framework across diverse linguistic contexts. 
Fourth, \name relaxes the requirement that the collaborator be distributionally similar to the policy, but not the requirement that its feedback be useful. A collaborator whose responses are systematically irrelevant or adversarial can still inject unhelpful tokens into the trajectory. 
The verifiable reward limits, but does not eliminate, the resulting damage, and we have not characterized \name's behaviour under an adversarial collaborator.
Fifth, our evidence for boundary expansion concerns the \emph{effective exploration boundary}: the set of problems that the trained policy can solve independently under a fixed inference-time sampling budget.
Our DeepMath experiments, including out-of-distribution evaluation, provide evidence that this effective boundary can be broadened through training-time interaction.
We do not interpret these experiments as establishing an expansion of the theoretical capability ceiling of the pretrained model.
A rigorous test of such a stronger claim would require substantially larger evaluations and carefully constructed out-of-distribution problems~\citep{abs-2504-13837}, which we leave to future work.

\end{document}